%% file: bare_jrnl.tex
\documentclass[journal]{IEEEtran}
\ifCLASSINFOpdf
\else
\fi
\usepackage{algorithm}
\usepackage{algorithmic}
\usepackage{titlesec}
\usepackage{etoc}
\usepackage{microtype}
\usepackage{graphicx}
\usepackage{textcomp}
\usepackage{subfigure}
\usepackage{booktabs} 
\usepackage{natbib}
\usepackage{upgreek}
\usepackage{wrapfig}
\usepackage{multirow}
\usepackage{makecell}
\usepackage{rotating}
\usepackage{colortbl}
\usepackage{marvosym}
\usepackage{pifont}
\usepackage{xcolor}
\usepackage{threeparttable}
\usepackage{multicol}
\usepackage{diagbox}
\usepackage[justification=Justified]{caption}
\usepackage{hyperref}
\usepackage[most]{tcolorbox}

\definecolor{kjred}{RGB}{228, 26, 28}
\definecolor{kjblue}{RGB}{55,126,184}
\definecolor{kjgreen}{RGB}{77,175,74}
\definecolor{kjpurple}{RGB}{152,78,163}
\definecolor{kjorange}{RGB}{255,127,0}

\definecolor{Red}{RGB}{244, 124, 124}
\definecolor{Green}{RGB}{162, 222, 147}
\definecolor{Blue}{RGB}{112, 161, 211}

\definecolor{peace}{RGB}{228, 26, 28}
\definecolor{love}{RGB}{55, 126, 184}
\definecolor{joy}{RGB}{77, 175, 74}
\definecolor{kindness}{RGB}{152, 78, 163}

\usepackage{amsmath}
\usepackage{amssymb}
\usepackage{bigstrut}
\usepackage{mathtools}
\usepackage{amsthm}
\usepackage{etoolbox}  

\makeatletter
\patchcmd{\algorithmic}{\addtolength{\ALC@tlm}{\leftmargin} }{\addtolength{\ALC@tlm}{\leftmargin}}{}{}
\makeatother
\usepackage[capitalize,noabbrev]{cleveref}

\def \R {\mathbb{R}}

\def \one {\mathbf{1}}

\def \v {\mathbf{v}}

\def \x {\mathbf{x}}

\def \E {\mathbb{E}}

\def \p {\mathbf{p}}

\def \y {\mathbf{y}}

\def \N {\mathcal{N}}
\def \V {\mathcal{V}}

\theoremstyle{plain}
\newtheorem{theorem}{Theorem}

\newtheorem{lemma}{Lemma}

\theoremstyle{definition}
\newtheorem{definition}{Definition}
\newtheorem{assumption}{Assumption}
\newtheorem{remark}{Remark}

\begin{document}
%
\title{Multinoulli Extension: A Lossless Continuous Relaxation for Partition-Constrained Subset Selection}
%
%
%

\author{Qixin Zhang, Wei Huang, Yan Sun, Yao Shu, Yi Yu and Dacheng Tao,~\IEEEmembership{Fellow,~IEEE}
\thanks{This project is supported by the National Research Foundation, Singapore, under its NRF Professorship Award No.
NRF-P2024-001.}
\thanks{Qixin Zhang and Dacheng Tao are with the college of computing and data science, Nanyang Technological University, Singapore (email: qixin.zhang@ntu.edu.sg;dacheng.tao@ntu.edu.sg).}
\thanks{Wei Huang is with the RIKEN Center for Advanced Intelligence Project(AIP), Japan (email: wei.huang.vr@riken.jp).}
\thanks{Yan Sun is with the School of Computer Science, University of Sydney, Australia (email: woodenchild95@outlook.com).}
\thanks{Yao Shu is with the Hong Kong University of Science and Technology(email: yaoshu@hkust-gz.edu.cn).}
\thanks{Yi Yu is with the school of electrical and electronic engineering, Nanyang Technological University, Singapore (email: yu.yi@ntu.edu.sg).}
}

%
%

\markboth{IEEE TRANSACTIONS, VOL. X, NO. X}%
{Shell \MakeLowercase{\textit{et al.}}: Bare Demo of IEEEtran.cls for IEEE Journals}
%


\maketitle
\begin{abstract}
Identifying the most representative subset for a \emph{close-to-submodular} objective while  satisfying the predefined partition constraint is a fundamental task with numerous applications in machine learning.  However, the existing  distorted local-search methods are often hindered by their prohibitive query complexities and the rigid requirement for prior knowledge of difficult-to-obtain structural parameters. To overcome these limitations, we introduce a novel algorithm titled \textbf{Multinoulli-SCG}, which not only is parameter-free, but also can achieve the same approximation guarantees as the distorted local-search methods 
with significantly fewer function evaluations. More specifically, when the objective function is monotone $\alpha$-weakly DR-submodular or $(\gamma,\beta)$-weakly submodular, our \textbf{Multinoulli-SCG} algorithm can attain a value of   $(1-e^{-\alpha})\text{OPT}-\epsilon$ or $(\frac{\gamma^{2}(1-e^{-(\beta(1-\gamma)+\gamma^2)})}{\beta(1-\gamma)+\gamma^2})\text{OPT}-\epsilon$ with only  $O(1/\epsilon^{2})$ function evaluations, where OPT denotes  the optimal value.  The cornerstone of our \textbf{Multinoulli-SCG} algorithm is an innovative continuous-relaxation framework named Multinoulli Extension(ME), which can effectively convert the discrete subset selection problem subject to partition constraints into a solvable continuous maximization focused on learning the optimal multinoulli priors across the concerned partition. In sharp contrast with the well-established multi-linear extension for submodular subset selection, a notable advantage of our proposed ME is its intrinsic capacity to provide a lossless rounding scheme for any set function. Furthermore, based on our proposed ME, we also present two novel online algorithms, namely, \textbf{Multinoulli-OSCG} and \textbf{Multinoulli-OSGA}, for the unexplored \emph{online} subset selection problems over partition constraints. Especially when each incoming set objective function is monotone $\alpha$-weakly DR-submodular or $(\gamma,\beta)$-weakly submodular,  we verify that both \textbf{Multinoulli-OSCG} and \textbf{Multinoulli-OSGA} can achieve a regret of $\mathcal{O}(\sqrt{T})$ against a $(1-e^{-\alpha})$-approximation or $\big(\frac{\gamma^{2}(1-e^{-(\beta(1-\gamma)+\gamma^2)})}{\beta(1-\gamma)+\gamma^2}\big)$-approximation to the best feasible solution in hindsight, respectively. Here, the symbol $T$ represents the time horizon. To the best of our knowledge, this is the \emph{first} result to attain such approximation ratios with $\mathcal{O}(\sqrt{T})$ regret to the partition-constrained \emph{online}  $\alpha$-weakly DR-submodular or $(\gamma,\beta)$-weakly submodular maximization problems. Finally, extensive simulations and experiments are conducted to demonstrate the effectiveness of our proposed algorithms.
\end{abstract}
%
\newpage
\section{Introduction}
\input{introduction.tex}

\section{Related Work}
\input{Related_Work}

\section{Notations and Preliminaries}
\input{Preliminaries.tex}
\section{Multinoulli Extension}\label{sec:ME}
\input{Multinoulli_Extension.tex}

\subsection{Stationary Point and Non-oblivious Auxiliary Functions}\label{Sec:ME:stationary_point}
\input{Multinoulli_Stationary_Point}

\subsection{Rounding without Replacement for Monotone Set Functions}\label{Sec:ME:rounding}
\input{Rounding_without_Replacement}
\section{Continuous Greedy Algorithms for Subset Selection over Partition Constraints}\label{sec:algorithm}
In this section, we  delve into the development of effective \emph{offline} and \emph{online} approximation algorithms for the partition-constrained subset selection problems based on our previously introduced \underline{M}ultinoulli \underline{E}xtension(\texttt{ME}). Specifically, in subsection~\ref{sec_scg}, we propose a \emph{parameter-free} and \emph{query-efficient} algorithm titled \textbf{Multinoulli-SCG} for our concerned subset selection problem~\eqref{equ_problem}. After that,  in subsection~\ref{sec_online_scg}, we utilize the idea of meta-action~\citep{streeter2009online,chen2018online,harvey2020improved} to extend this \textbf{Multinoulli-SCG} algorithm into more complicated \emph{online} settings.
\subsection{Stochastic Variant of Continuous Greedy Algorithm}\label{sec_scg}
\input{Multinoulli-SCG}

\section{Stochastic Gradient Ascent for Subset Selection over Partition Constraints}\label{sec:algorithm1}
\input{Gradient_Ascent}

\section{Numerical Experiments}\label{sec:expriments}
\input{Experiments}
\section{Conclusion}
This paper introduces a novel continuous-relaxation framework named Multinoulli Extension(\texttt{ME}) for the partition-constrained subset selection problem. In sharp contrast to the well-known multi-linear extension for submodular subset selection, a notable advantage of our \texttt{ME} is that it can provide a lossless round scheme for any set objective function. Subsequently, base on \texttt{ME}, we develop an efficient algorithm titled \textbf{Multinoulli-SCG}, which can achieve a value of   $(1-e^{-\alpha})\text{OPT}-\epsilon$ or $(\frac{\gamma^{2}(1-e^{-(\beta(1-\gamma)+\gamma^2)})}{\beta(1-\gamma)+\gamma^2})\text{OPT}-\epsilon$  with only $O(1/\epsilon^{2})$ function evaluations for monotone $\alpha$-weakly DR-submodular or $(\gamma,\beta)$-weakly submodular objective functions. This result significantly improves the previous $\tilde{\mathcal{O}}(1/\epsilon^{6})$ and $\tilde{\mathcal{O}}(1/\epsilon^{3})$ number of function evaluations associated with the distorted local-search methods.

Furthermore, based on our proposed \texttt{ME}, we also present two novel online algorithms, namely, \textbf{Multinoulli-OSCG} and \textbf{Multinoulli-OSGA}, for the unexplored \emph{online} subset selection problems over partition constraints. Especially when each incoming set objective function is monotone $\alpha$-weakly DR-submodular or $(\gamma,\beta)$-weakly submodular,  we verify that both \textbf{Multinoulli-OSCG} and \textbf{Multinoulli-OSGA} can achieve a regret of $\mathcal{O}(\sqrt{T})$ against a $(1-e^{-\alpha})$-approximation or $\big(\frac{\gamma^{2}(1-e^{-(\beta(1-\gamma)+\gamma^2)})}{\beta(1-\gamma)+\gamma^2}\big)$-approximation to the best feasible solution in hindsight, respectively. To the best of our knowledge, this is the \emph{first} result to attain such approximation ratios with $\mathcal{O}(\sqrt{T})$ regret to the partition-constrained \emph{online}  $\alpha$-weakly DR-submodular or $(\gamma,\beta)$-weakly submodular maximization problems. Finally,
extensive evaluations have been conducted to validate the effectiveness of our proposed algorithms.

{\small
\bibliographystyle{unsrt}
\bibliography{icml2025}}

\onecolumn
\appendix

\etocsettocstyle{\section*{}}{} 

\etocsetnexttocdepth{subsection} 

\localtableofcontents


\section{Additional Experimental Discussions}\label{append:Experiments}
\subsection{More Discussions on \emph{Offline} Experiment Setups}
In this subsection, we highlight some additional details about the experiments in \cref{sec:offline_subset_selection}. At first, we describe the parameter setups regarding our proposed `\textbf{Multinoulli-SGA}', `\textbf{Multinoulli-ASCG}', `\textbf{Multinoulli-SCG}' and `Distorted-LS-G'. More specifically, we consider the following parameter configurations:
\begin{itemize}
\item `\textbf{Distorted-LS-G}': Algorithm B.1 in \citep{thiery2022two} where we set  the number of guesses $L=1+\lceil\log_{(1-\epsilon)}(\frac{3}{16})\rceil$, the total number of improvements $M=\frac{r}{\epsilon}$, the number of samples $N=\frac{r}{\epsilon^{2}}$ and $\epsilon=0.01$ where $r$ is the rank of the partition constraint, namely, $r=\sum_{k=1}^{K}B_{k}$ in problem~\eqref{equ_problem};
\item `\textbf{Multinoulli-SGA}': \emph{Offline} version of \cref{alg:OGA_framework} is implemented with $\text{Auxiliary}\triangleq\textbf{False}$, batch size $L=20$, the step size $\eta=\frac{1}{\sqrt{T}}$ and the number of iterations $T=167$. Moreover, we round each continuous solution via \cref{alg_without_replacement}. As for the Euclidean projection of Step 14 in \cref{alg:OGA_framework}, we utilize the CVX optimization tool~\citep{grant2014cvx}. It is worth noting that there exists efficient algorithms for handling the projection over partitions if we view it as multiple independent singly constrained quadratic programmings problems(See Appendix B in \citep{zhou2022probabilistic});
\item `\textbf{Multinoulli-ASGA}': \emph{Offline} version of \cref{alg:OGA_framework} is implemented with $\text{Auxiliary}\triangleq\textbf{True}$, $w(z)\triangleq e^{z-1}$, batch size $L=20$, the step size $\eta=\frac{1}{\sqrt{T}}$ and the number of iterations $T=167$. Moreover, we round each continuous solution via \cref{alg_without_replacement}. As for the Euclidean projection of Step 14 in \cref{alg:OGA_framework}, we utilize the CVX optimization tool~\citep{grant2014cvx};
\item`\textbf{Multinoulli-SCG}': \cref{alg:scg_mutli} is implemented with batch size $L=\lceil \frac{T}{2}\rceil$ and the number of iterations $T=167$. Furthermore, in order to obtain a high-quality subset, we apply \cref{alg_without_replacement} to round the $\x(T+1)$ a total of $T^{2}$ times and then select  the best one among these resulting subsets.
\end{itemize}
Note that, when $\epsilon=0.01$, the number of guesses $L=1+\lceil\log_{(1-\epsilon)}(\frac{3}{16})\rceil\approx167$ in `Distorted-LS-G' such that we set the total of iterations $T=167$ in `\textbf{Multinoulli-SGA}' and `\textbf{Multinoulli-ASGA}' as well as `\textbf{Multinoulli-SCG}'. Given the long runtimes of `\textbf{Multinoulli-SCG}' and `Distorted-LS-G',  we report their average results in \cref{tab:results_video_summarization} and \cref{tab:results_coverage} based on 5 repeated experiments. In contrast, `\textbf{Multinoulli-SGA}' and `Residual-Greedy' algorithms are repeated 20 times. 
\subsection{More Discussions on \emph{Online} Experiment Setups}\label{append:online_Experiments}
In this subsection, we highlight some additional details about the experiments in \cref{sec:online_subset_selection}. At first, we describe the parameter setups regarding our proposed `\textbf{Multinoulli-OSGA}', `\textbf{Multinoulli-OSCG}', `MA-OSMA' and  `MA-OSEA'. More specifically, we consider the following parameter configurations:
\begin{itemize}
	\item In  `\textbf{Multinoulli-OSGA}', namely, the \cref{alg:OGA_framework}, we set $\text{Auxiliary}\triangleq\textbf{True}$, $w(z)\triangleq e^{\alpha(z-1)}$ where $\alpha\in\{0.1,0.2,\dots,1\}$, the step size $\eta_{t}\triangleq1/\sqrt{T}$ and employ a $10$ batch of stochastic estimation to approximate the surrogate gradient of our proposed multinoulli extension. As for the projection in Line 19 of \cref{alg:OGA_framework}, we utilize the CVX optimization solver~\citep{grant2014cvx};
	\item In `\textbf{Multinoulli-OSCG}', namely, the \cref{alg:online_scg_mutli}, we set the batch size $L=10$ and the number of oracles $K=15$. As for the online linear maximization oracle, we utilize the online gradient ascent algorithm~\citep{zinkevich2003online} with step size $\eta=\mathcal{O}(1/\sqrt{T})$;
	\item In `MA-OSMA', namely Algorithm 1 of \citep{zhang2025nearoptimal}, we consider the Euclidean distance with $\phi(\x)\triangleq\frac{\|\x\|_{2}^{2}}{2}$, set the step size $\eta_{t}\triangleq1/\sqrt{T}$ and implement a $10$ batch of stochastic estimation to approximate the surrogate gradient of \emph{multi-linear extension}. Similarly, we also use the CVX solver for projection operations;
	\item In `MA-OSEA', namely Algorithm 2 of \citep{zhang2025nearoptimal}, we also set the step size $\eta_{t}\triangleq1/\sqrt{T}$ and consider the mixing parameter $\gamma\triangleq1/T^{1.5}$.
    \item In `OSG', namely Algorithm 2 in \citep{xu2023online}
\end{itemize}
Furthermore, for all curves related to `OSG', `\textbf{Multinoulli-OSGA}', `\textbf{Multinoulli-OSCG}', `MA-OSMA' and  `MA-OSEA', we repeat these algorithms \textbf{five runs} and then report the average result. Note that, in \cref{graph_show}, the running average utility at any time $t$ is defined as $\big(\sum_{t_{1}\in[t]}\frac{f_{t}(\cup_{i}\{a_{i}(t_{1})\})}{t}\big)$ where $a_{i}(t_{1})$ is the action chosen by UAV $i$ at time $t_{1}$.
\section{Additional Discussions}\label{append:disscussion}
\subsection{Multi-linear Extension}\label{appendix:Multi_linear_Extension}
In this subsection, we introduce the multi-linear extension of \citep{calinescu2011maximizing} for submodular maximization and simultaneously compare it with our proposed \underline{M}ultinoulli \underline{E}xtension(\texttt{ME}) in \cref{sec:ME}. In order to better illustrate the multi-linear extension,  this subsection supposes $|\V|=n$ and decodes the ground set $\V$ as $[n]$, i.e., $\V\triangleq\{1,\dots,n\}$. Then, we can show that 
\begin{definition}\label{def1:multi-linear}
	For a set function $f:2^{\V}\rightarrow\R_{+}$, we define its multi-linear extension  as 
	\begin{equation}
		\label{equ:multi-linea}
G(\x)=\sum_{\mathcal{A}\subseteq\V}\Big(f(\mathcal{A})\prod_{a\in\mathcal{A}}x_{a}\prod_{a\notin\mathcal{A}}(1-x_{a})\Big)=\E_{\mathcal{R}\sim\x}\Big(f(\mathcal{R})\Big),
	\end{equation} where $\x=(x_{1},\dots,x_{n})\in [0,1]^{n}$ and $\mathcal{R}\subseteq\V$ is a random set that contains each element $a\in\V$ independently with probability $x_{a}$ and excludes it with probability $1-x_{a}$. We write $\mathcal{R}\sim\x$ to denote that $\mathcal{R}\subseteq\V$ is a random set sampled according to $\x$. 
\end{definition}
From the Eq.\eqref{equ:multi-linea}, we can view multi-linear extension $G$ at any point $\x\in[0,1]^{n}$ as the expected utility of independently selecting each action $a\in\V$ with probability $x_{a}$.  With this tool, we can cast the previous discrete subset selection problem~\eqref{equ_problem} into a continuous maximization which learns the independent probability for each element $a\in\V$, that is, we consider the following continuous optimization:
\begin{equation}\label{equ:multilinear_continuous_max}
	\max_{\x\in[0,1]^{n}} G(\x),\ \ \text{ s.t.}\ \  \sum_{a\in\V_{k}}x_{a}\le B_{k},\forall k\in[K]
\end{equation}where $G(\x)$ is the multi-linear extension of $f$.

It is important to note that, if we round any point $\x$ included into the constraint of problem~\eqref{equ:multilinear_continuous_max} throughout the definition of multi-linear extension, i.e., Eq.\eqref{equ:multi-linea}, there is a certain probability that the resulting subset  will violate the partition constraint of the subset selection problem~\eqref{equ_problem}.  Therefore, for multi-linear extension, we need to specifically design the rounding methods based on the properties of the set objective functions. However, current known lossless rounding schemes for multi-linear extension, such as pipage rounding~\citep{ageev2004pipage}, swap rounding~\citep{chekuri2010submodular} and contention resolution~\citep{chekuri2014submodular}, are heavily dependent on the \emph{submodular} assumption. So far, how to losslessly round the multi-linear extension of \emph{non-submodular} set functions, e.g. $(\gamma,\beta)$-weakly submodular and $\alpha$-weakly DR-submodular functions, still remains an open question~\citep{thiery2022two}.  

In contrast, our \texttt{ME}  does not assign probabilities to any subsets that are out of the partition constraint of problem~\eqref{equ_problem}, which means that, for any set function $f:2^{\V}\rightarrow\R_{+}$ and any given $(\p_{1},\dots,\p_{K})\in\prod_{k=1}^{K}\Delta_{n_{k}}$, we can, through the definition of \texttt{ME}, easily produce a subset that conforms to the partition constraint of problem~\eqref{equ_problem} without any loss in terms of the expected function value $F(\p_{1},\dots,\p_{K})$.
\subsection{Non-closedness of the Domain \texorpdfstring{$\prod_{k=1}^{K}\Delta_{n_{k}}$}{} about Coordinate-wise Maximum Operation}\label{appendix:non_close}
We need to emphasize that, when the set function $f$ is monotone submodular, \cref{thm2} implies that our proposed \underline{M}ultinoulli \underline{E}xtension(\texttt{ME}) is also monotone continuous DR-submodular over the domain $\prod_{k=1}^{K}\Delta_{n_{k}}$. However, it is worth noting that the former algorithms about continuous DR-submodular maximization~\citep{hassani2020stochastic,hassani2017gradient,zhang2022boosting,wan2023bandit} cannot be directly applied to our \texttt{ME} and the relaxed problem~\eqref{equ_prob_relaxed}. This is because all of them highly rely on the following inequality: if $G$ is a monotone continuous DR-submodular function,
\begin{equation}\label{equ:DR-submodular}
\langle\y-\x,\nabla G(\x)\rangle\ge\langle\y\vee\x-\x, \nabla G(\x)\rangle\ge G(\y\vee\x)-G(\x)\ge G(\y)-G(\x),
\end{equation} where $\vee$ is the coordinate-wise maximum operation, i.e., $\x\vee\y\triangleq\max(\x,\y)$.

Note that the second inequality in Eq.\eqref{equ:DR-submodular} requires that the vector $\y\vee\x$ is included into the domain of objective function $G$. In other words, the domain of objective function $G$ should be closed under the coordinate-wise maximum operation $\vee$. However, the domain $\prod_{k=1}^{K}\Delta_{n_{k}}$ of our \texttt{ME} does not meet this requirement. Similarly, to the best of our knowledge, the latest research on both upper-linearizable and weakly continuous DR-submodular maximization also requires the domain of the objective function to be closed under the coordinate-wise maximum operation $\vee$. (See \cite{hassani2017gradient,zhang2022boosting,pedramfar2024from})
 \section{The Fundamenta Properties of Multinoulli Extension}\label{append:ME}
 \subsection{Proof of \texorpdfstring{\cref{thm1}}{}}\label{append:proof1}
 \input{Proofs_for_thm_1}
 \subsection{Proof of \texorpdfstring{\cref{thm2}}{}}\label{append:proof2}
 \input{Proofs_for_Thm2}
 \subsection{Proof of \texorpdfstring{\cref{thm4}}{}}\label{append:proof4}
 \input{Proofs_for_Thm4}
 \section{Stationary Point and Non-oblivious Auxiliary Functions}
 \subsection{Proof of \texorpdfstring{\cref{thm3}}{}}\label{append:proof3}
 \input{Proofs_for_Thm3}
 \section{Rounding Without Replacement}\label{append:rounding_without_replacement}
 \input{Proof_for_Rounding}
 \section{Continuous Greedy Method for Multinoulli Extension}\label{append:SCG}
 \subsection{Proof of \texorpdfstring{\cref{thm5}}{}}\label{append:proof5}
 \input{Proofs_for_Thm5}
\section{Stochastic Gradient Ascent for Mutinoulli Extension}\label{append:SGA_Stationary_point}\subsection{Proof of  \texorpdfstring{\cref{thm:online1}}{}}\label{sec:multinoulli-sga}
\input{Multinoulli-SGA}

\end{document}

%% file: introduction.tex
\IEEEPARstart{S}{ubset} selection aims to identify a small group of representative items from a vast ground set, which finds numerous real-world applications in the fields of artificial intelligence, machine learning and data mining, including feature selection~\cite{qian2015subset,das2018approximate,dong2025unsupervised,dong2025delbo}, data summarization~\cite{wei2015submodularity,killamsetty2022automata,jain2023efficient,zhuang2025efficient}, product marketing~\cite{kempe2003maximizing,borgs2014maximizing} and in-context learning~\cite{kumari2024end,fan2025combatting,li2025otter}. Beyond the aforementioned representational capability, ensuring the diversity and fairness of the chosen subset is of significant importance. For instance, in various marketing scenarios, it is essential to equitably allocate free products across different communities~\citep{tsang2019group}. To this end, partition constraints are often imposed in the process of subset selection, which involves dividing the entire set into non-overlapping sub-classes and then fairly distributing the total budget among them. Motivated by these findings, this paper explores the subset selection problem over partition constraints. 

Broadly speaking, the subset selection problem  is \textbf{NP}-hard~\citep{natarajan1995sparse,feige1998threshold}, implying that no polynomial-time algorithms can solve it optimally. In light of this hurdle, many studies have focused on designing efficient approximation algorithms to address the subset selection problem.  Especially when the utility function associated with the subset selection problem is \emph{submodular}, a plethora of effective and practical algorithms have been proposed for maximizing this type of functions subject to partition constraints~\citep{calinescu2011maximizing,filmus2014monotone,liu2013entropy}. Additionally, it has been frequently observed that there are also many scenarios inducing utility functions that are ``\emph{close-to-submodular}", but not strictly submodular.  Examples include variable selection for regression~\citep{elenberg2018restricted}, video summarization~\citep{chen2018weakly,li2025language}, neural network pruning~\citep{el2022data,he2023structured,cheng2024influence} and sparse optimal transport~\citep{montesuma2024recent,manupriya2024submodular}. 

Compared to the extensive literature on submodular maximization, there is a limited amount of research exploring the maximization of ``\emph{close-to-submodular}" objectives under partition constraints. 
Notably, \citep{thiery2022two} recently proposed a distorted local-search algorithm to maximize an important class of ``\emph{close-to-submodular}" functions named $(\gamma,\beta)$-weakly submodular functions and demonstrated that this  approach can secure a $\frac{\gamma^{2}(1-e^{-(\beta(1-\gamma)+\gamma^2)})}{\beta(1-\gamma)+\gamma^2}$-approximation under  partition constraints, where $\gamma$ and  $\beta$ represent the lower and upper submodularity ratio respectively. Subsequently, \citep{JOGO-Lu} extended this local search to another class of ``\emph{close-to-submodular}" functions known as $\alpha$-weakly DR-submodular functions and also confirmed a tight $(1-e^{-\alpha})$-approximation guarantee under partition constraints, where $\alpha$ is the diminishing-return(DR) ratio.

Despite the superior theoretical guarantees of distorted local-search methods, their practical implementation often faces two significant challenges: \textbf{i) Reliance on Unknown Parameters:} Distorted local search generally requires prior knowledge of specific structural parameters regarding the  objective functions, such as the submodularity ratio and diminishing-return(DR) ratio. However, in practice, accurately estimating these parameters can incur exponential computations. \textbf{ii) Prohibitive Query Complexity:} Due to the absence of necessary structural  parameters,  \citep{thiery2022two,JOGO-Lu} have to adopt a brute-force $\mathcal{O}(1/\epsilon)$-round guesses of these unknown parameters to approximate the distorted local-search methods, which will result in an extremely high $\tilde{\mathcal{O}}(1/\epsilon^{6})$ and $\tilde{\mathcal{O}}(1/\epsilon^{3})$ number of value queries to the objective function, respectively. In view of all this, a natural question arises:
\vspace{-0.05em}\begin{tcolorbox}\textbf{Q1:} Is it possible to develop a \emph{parameter-free} and \emph{query-efficient} algorithm for the ``\emph{close-to-submodular}" subset selection problems under partition constraints while keeping strong approximation guarantees?\end{tcolorbox}\vspace{-0.05em}
In addition, recent years have witnessed a surge in applications related to the \emph{online} subset selection problems over partition constraints. A compelling example is the dynamic deployment of mobile sensors, especially unmanned aerial vehicles(UAVs), to monitor an area of interest. In such scenarios, each UAV needs to constantly determine its trajectory and velocity to effectively track all moving points. The primary challenges of this tracking challenge lie in the unpredictability of the targets’ movements. Notably, under the extended Kalman filter framework, the works~\citep{hashemi2019submodular,hashemi2020randomized,rezazadeh2023distributed} have pointed out that the aforementioned multi-target tracking problem of UAVs can be formulated as a partition-constrained \emph{online} weakly DR-submodular maximization problem.  However, so far, how to extend the state-of-the-art distorted local-search methods~\cite{thiery2022two,JOGO-Lu} for \emph{offline} weakly DR-submodular maximization into \emph{online} settings still remains an unknown territory. Given this gap, a follow-up question comes to our mind, i.e.,
\vspace{-0.05em}\begin{tcolorbox}
\textbf{Q2:} Can we design an effective approximation algorithm for the \emph{online} partition-constrained subset selection problems with ``\emph{close-to-submodular}" objectives?
\end{tcolorbox}\vspace{-0.05em}

In this paper, we will first provide an affirmative answer to \textbf{Q1} by presenting an effective algorithm titled \textbf{Multinoulli-SCG}, which not only eliminates the strict requirement for the exact knowledge of both submodularity ratio and DR ratio, but also can attain the same approximation guarantees as the aforementioned distorted local-search methods~\citep{thiery2022two,JOGO-Lu} with only $O(1/\epsilon^{2})$ function evaluations. 
After that, to handle \textbf{Q2}, we introduce two novel online algorithms, namely, \textbf{Multinoulli-OSCG} and \textbf{Multinoulli-OSGA}, for the \emph{online} subset selection problems over partition constraints. Especially when every incoming set objective function is monotone $\alpha$-weakly DR-submodular or $(\gamma,\beta)$-weakly submodular,  we verify that both \textbf{Multinoulli-OSCG} and \textbf{Multinoulli-OSGA} can achieve a regret of $\mathcal{O}(\sqrt{T})$ against a $(1-e^{-\alpha})$-approximation or $\big(\frac{\gamma^{2}(1-e^{-(\beta(1-\gamma)+\gamma^2)})}{\beta(1-\gamma)+\gamma^2}\big)$-approximation to the best solution in hindsight, respectively. Here, $T$ is the time horizon. As far as we know, this is the \emph{first} result to achieve such approximation ratios with $\mathcal{O}(\sqrt{T})$ regret to the partition-constrained \emph{online}  $\alpha$-weakly DR-submodular or $(\gamma,\beta)$-weakly submodular maximization problems. The core of our proposed \textbf{Multinoulli-SCG}, \textbf{Multinoulli-OSCG} and \textbf{Multinoulli-OSGA} algorithms  is an innovative continuous-relaxation framework termed as the Multinoulli Extension(\texttt{ME}), which aims to learn a prior multinoulli distribution for each community within the concerned partition constraints and subsequently leverages these distributions to make selection. In sharp contrast with the well-established multi-linear extension~\citep{calinescu2011maximizing},  a notable advantage of our proposed \texttt{ME} is its inherent capability to provide a lossless rounding scheme for \emph{any} set function. Instead, all known lossless rounding schemes for multi-linear extension~\citep{calinescu2011maximizing} require that the objective set function is \emph{submodular}~\citep{ageev2004pipage,chekuri2010submodular,chekuri2014submodular}.

\begin{table*}[t]
	\centering \small
    \begin{threeparttable}
	\caption{Comparison of theoretical guarantees for \emph{partition-constrained} subset selection problems}\label{tab:Comparison1}
		\setlength{\tabcolsep}{3.4mm}{
			\begin{tabular}{ccccc}
				\toprule[1.0pt]
				\textbf{Method} & \textbf{Objective}&
                \textbf{Para-free?} &\textbf{\#Queries}&\textbf{Utility} \\
                \midrule[1.0pt]
				Residual Random Greedy \citep{chen2018weakly,thiery2022two} & \makecell[c]{$\gamma$-weakly Submodular\\ $(\gamma,\beta)$-weakly Submodular}&  \ding{52} & $\mathcal{O}(nr)$ & \makecell[c]{$\big(\frac{\gamma^{2}}{(1+\gamma)^{2}}\big)$OPT\\ $\big(\frac{\gamma}{\gamma+\beta}\big)$OPT}\\\midrule[1.0pt]
                Standard Greedy \citep{gatmiry2018non}& \makecell[c]{$\alpha$-weakly DR-Submodular \\ $(\gamma,\beta)$-weakly Submodular}& \ding{52}   & $\mathcal{O}(nr)$& \makecell[c]{$\big(\frac{\alpha}{1+\alpha}\big)$OPT\\ $\big(\frac{0.4\gamma^{2}}{\sqrt{\gamma r}+1}\big)$OPT}\\ \midrule[1.0pt]
                Distorted-LS \citep{thiery2022two,JOGO-Lu} & \makecell[c]{$\alpha$-weakly DR-Submodular \\ $(\gamma,\beta)$-weakly Submodular}&\ding{56}  & $\Omega(nr2^{r})$~\tnote{$\ddagger$}& \makecell[c]{$\left(1-e^{-\alpha}\right)$OPT\\ $\big(\frac{\gamma^{2}(1-e^{-\phi(\gamma,\beta)})}{\phi(\gamma,\beta)}\big)$OPT}\\
               \midrule[1.0pt]
                Distorted-LS-Guessing \citep{thiery2022two,JOGO-Lu}&\makecell[c]{$\alpha$-weakly DR-Submodular \\ $(\gamma,\beta)$-weakly Submodular}& \ding{52}& \makecell[c]{$\tilde{\mathcal{O}}\big(nr^{4}/\epsilon^{6}\big)$\\$\tilde{\mathcal{O}}\big(nr^{4}/\epsilon^{3}\big)$}&\makecell[c]{$\left(1-e^{-\alpha}-\epsilon\right)$OPT\\ $\left(\frac{\gamma^{2}(1-e^{-\phi(\gamma,\beta)})}{\phi(\gamma,\beta)}-\epsilon\right)$OPT}\\
              \midrule[1.0pt]
				\rowcolor{cyan!18}
				\Gape{\makecell[c]{\textbf{Multinoulli-SCG}\\(\cref{thm5})}}& \Gape{\makecell[c]{$\alpha$-weakly DR-Submodular \\ $(\gamma,\beta)$-weakly Submodular}} & \ding{52}& $\mathcal{O}\big(r^{3}n^{2}/\epsilon^{2}\big)$& \Gape{\makecell[c]{$\left(1-e^{-\alpha}\right)$OPT$-\epsilon$\\ $\big(\frac{\gamma^{2}(1-e^{-\phi(\gamma,\beta)})}{\phi(\gamma,\beta)}\big)$OPT$-\epsilon$}}\\
				\bottomrule
	\end{tabular}}
    \begin{tablenotes}
        \footnotesize
        \item[$\dagger$]`\textbf{Para-free}' indicates whether the method does not require prior knowledge of DR ratio $\alpha$ and submodular ratio $(\gamma,\beta)$,  `\textbf{\#Queries}' denotes the number
of queries to the set objectives, `OPT' represents  the optimal value of the subset selection problem~\eqref{equ_problem},  $r$ is the rank of partition constraint, i.e., $r\triangleq\sum_{k=1}^{K}B_{k}$, $n$ is the size of the ground set, namely, $n=|\V|$ and $\phi(\gamma,\beta)\triangleq\beta(1-\gamma)+\gamma^2$.
\item[$\ddagger$] Distorted-LS~\citep{thiery2022two,JOGO-Lu} requires exact computation of the auxiliary potential function, which leads to exponential query complexity.
    \end{tablenotes}
    \end{threeparttable}
\end{table*}

\textbf{Our Contributions.}
\begin{enumerate}
    \item[\textbf{1)}] This paper introduces a novel continuous relaxation for the partition-constrained subset selection problems, which we term as the Multinoulli Extension(\texttt{ME}). Furthermore, we conduct an in-depth exploration of the differentiability, monotonicity and weak submodularity about the \texttt{ME}; 
    \item[\textbf{2)}] When the set objective function is monotone $\alpha$-weakly DR-submodular or $(\gamma,\beta)$-weakly submodular, we demonstrate that the stationary points of our proposed  \texttt{ME} can ensure an approximation ratio of $(\frac{\alpha^{2}}{1+\alpha^{2}})$ or $(\frac{\gamma^{2}}{\beta+\beta(1-\gamma)+\gamma^{2}})$ to the concerned subset selection problem. To further improve the approximation guarantees of these stationary points, we design two non-oblivious auxiliary functions for the \texttt{ME} and prove that the stationary points of these auxiliary functions can  provide a better $(1-e^{-\alpha})$-approximation or $\frac{\gamma^{2}(1-e^{-(\beta(1-\gamma)+\gamma^2)})}{\beta(1-\gamma)+\gamma^2}$-approximation than those of the original \texttt{ME};
    \item[\textbf{3)}] To effectively convert the obtained multinoulli distributions into a satisfactory subset within the concern partition constraint, we develop a more practical rounding-without-replacement method for \emph{monotone} set functions;
    \item[\textbf{4)}] We propose a novel algorithm named  \textbf{Multinoulli-SCG}, which effectively integrates our proposed \texttt{ME}, the path-integrated differential estimator and the concept of continuous greedy. Moreover, we prove that, when the objective function is monotone $\alpha$-weakly DR-submodular or $(\gamma,\beta)$-weakly submodular, our \textbf{Multinoulli-SCG} can attain a value of   $(1-e^{-\alpha})\text{OPT}-\epsilon$ or $(\frac{\gamma^{2}(1-e^{-(\beta(1-\gamma)+\gamma^2)})}{\beta(1-\gamma)+\gamma^2})\text{OPT}-\epsilon$  with only $O(1/\epsilon^{2})$ function evaluations, where OPT denotes the optimal function value;
   \item[\textbf{5)}] By imitating the linear-optimization steps of the \textbf{Multinoulli-SCG} algorithm via a sequence of online linear maximization oracles, we present an effective online algorithm named \textbf{Multinoulli-OSCG} for the \emph{online} partition-constrained subset selection problems. Furthermore, we establish that, when every incoming set objective function is monotone $\alpha$-weakly DR-submodular or $(\gamma,\beta)$-weakly submodular, this  \textbf{Multinoulli-OSCG} can achieve a regret  bound of $\mathcal{O}(\sqrt{T})$ against a $(1-e^{-\alpha})$-approximation or $\big(\frac{\gamma^{2}(1-e^{-(\beta(1-\gamma)+\gamma^2)})}{\beta(1-\gamma)+\gamma^2}\big)$-approximation to the best solution in hindsight, respectively;

\item[\textbf{6)}] Inspired by the non-oblivious auxiliary functions of our \texttt{ME}, we propose a novel variant of the classic online gradient ascent algorithm, named \textbf{Multinoulli-OSGA}, for the \emph{online} partition-constrained subset selection problems. Compared to the previous \textbf{Multinoulli-OSCG}, this new \textbf{Multinoulli-OSGA} not only eliminates the need for maintaining multiple online linear oracles but also can attain the same approximation ratios with $\mathcal{O}(\sqrt{T})$ regret to the \emph{online} subset selection problems;
\item[\textbf{7)}] We demonstrate the practical efficacy of our proposed algorithms by applying them to video summarization, target tracking and bayesian A-optimal design, etc.
\end{enumerate}

The preliminary results of this paper were presented in part at the 42nd International Conference on Machine Learning in 2025. In this paper, we have significantly expanded the conference version~\citep{zhang2025multinoulli} by \emph{newly} adding the content of the above contributions \textbf{2)}, \textbf{3)}, \textbf{5)} and \textbf{6)}.

%% file: Related_Work.tex
\textbf{Weakly Submodular Maximization:} The $\gamma$-weakly submodular functions were originally introduced by the work~\citep{das2018approximate}, which also demonstrated that the standard greedy algorithm can achieve an approximation ratio of $(1-e^{-\gamma})$ for the problem of maximizing such functions subject to a cardinality constraint. Subsequently,  \citep{chen2018weakly} investigate weakly submodular maximization beyond simple cardinality constraint and pointed out that the Residual Random Greedy method of \citep{buchbinder2014submodular} can achieve an approximation ratio of  $\frac{\gamma^{2}}{(1+\gamma)^{2}}$ for  the problem of maximizing a monotone $\gamma$-weakly submodular functions subject to a matroid constraint. Note that matroid constraint is a natural generalization of  the partition constraints considered in this paper. After that, \citep{gatmiry2018non} examined the approximation performance of standard greedy algorithm on $\gamma$-weakly submodular maximization over matroid constraints, which showed that the standard greedy algorithm can offer an approximation factor of $\frac{0.4\gamma^{2}}{\sqrt{r\gamma}+1}$ where $r$ is the rank of the matroid. To improve the approximation performance of these greedy algorithms, \citep{thiery2022two} introduced the notion of upper submodularity ratio $\beta$ and also developed a more powerful distorted local search for $(\gamma,\beta)$-weakly submodular maximization, which can guarantee a $\frac{\gamma^{2}(1-e^{-(\beta(1-\gamma)+\gamma^2)})}{\beta(1-\gamma)+\gamma^2}$-approximation for the problem of maximizing a monotone $(\gamma,\beta)$-weakly submodular functions subject to a matroid constraint. 

\textbf{Weakly DR-Submodular Maximization:} The work of \citep{gatmiry2018non} is the first to study the maximization of a $\alpha$-weakly DR-submodular subject to general matroid constraints. Specifically, \citep{gatmiry2018non} proved that the greedy algorithm achieves an
approximation ratio of $\frac{\alpha}{1+\alpha}$ for the matroid-constrained $\alpha$-weakly DR-submodular maximization. Moreover, \citep{ijcai2022p666} also considered the impact of curvature~\citep{sviridenko2017optimal} on the standard greedy algorithm for $\alpha$-weakly DR-submodular maximization under partition matroid. Recently, \citep{JOGO-continuous-loss}  showed that the continuous greedy combined with the contention resolution scheme~\citep{chekuri2014submodular} can obtain a sub-optimal approaximation ratio of $\alpha(1-1/e)(1-e^{-\alpha})$ for the problem of maximizing a monotone $\alpha$-weakly DR-submodular functions subject to a matroid constraint. To achieve the tight $(1-e^{-\alpha})$-approximation guarantee, \citep{JOGO-Lu} recently have proposed a novel distorted local-search method.

\textbf{Online Submodular Maximization:} The study of online submodular maximization was initiated by  \citep{streeter2008online}, which gave the first algorithm attaining a $(1-e^{-1})$-regret of $\mathcal{O}(\sqrt{r\ln(n)T})$ for online monotone submodular maximization under a cardinality constraint. Subsequently, this result was generalized to a partition constraint and a general matroid constraint in \citep{harvey2020improved} with a slightly worse $(1-e^{-1})$-regret of $\mathcal{O}(\sqrt{rT\ln(n/r)})$. After that, in order to circumvent the computational bottleneck of repeatedly evaluating and differentiating the multilinear extension, \citep{chen2018projection,zhang2022boosting,pedramfar2023a,pedramfar2024from} shifted focus to stochastic online continuous DR-submodular maximization problems. Notably, a detailed comparison of our proposed algorithms with existing studies is presented in \cref{tab:Comparison1} and \cref{tab:Comparison2}.

\begin{table*}[t]
	\centering \small
    \begin{threeparttable}
	\caption{Comparison of theoretical guarantees for \emph{online} partition-constrained subset selection problems}\label{tab:Comparison2}
		\setlength{\tabcolsep}{2.2mm}{
			\begin{tabular}{ccccccc}
				\toprule[1.0pt]
				\textbf{Method} & \textbf{Objective}&
                \textbf{Para-free?}&\textbf{Proj-free?}&\textbf{\#Queries}&\textbf{Approx.Ratio}& \textbf{Regret} \\
           \midrule[1.0pt]
               MFW~\citep{chen2018projection,pedramfar2023a}& Submodular&  \ding{52}& \ding{52}   &$\mathcal{O}(nT^{5/2})$& $\left(1-e^{-1}\right)$&$\mathcal{O}(\sqrt{nrT})$\\\midrule[1.0pt]
               BOGA~\citep{zhang2022boosting,pedramfar2023a}& Submodular& \ding{52}  & \ding{56}   &$\mathcal{O}(nT)$& $(1-e^{-1})$ &$\mathcal{O}(\sqrt{nrT})$\\\midrule[1.0pt] 
               CTA\citep{harvey2020improved}& Submodular& \ding{52} & \ding{52}   &$\Omega(2^{n}T)$~\tnote{$\ddagger$}& $(1-e^{-1})$ &$\mathcal{O}
               (\sqrt{r\ln(n/r)T})$\\\midrule[1.0pt]
				\rowcolor{cyan!18}
				\Gape{\makecell[c]{\textbf{Multinoulli-OSCG}\\(\cref{thm:online} \& \cref{remark:regret_oscg})}} & \Gape{\makecell[c]{$\alpha$-weakly DR-Sub \\ $(\gamma,\beta)$-weakly Sub}} & \ding{52}&\ding{52}&$\mathcal{O}\left(rn^{2}T^{2}\right)$& \Gape{\makecell[c]{$\left(1-e^{-\alpha}\right)$\\ $\big(\frac{\gamma^{2}(1-e^{-\phi(\gamma,\beta)})}{\phi(\gamma,\beta)}\big)$}}&$\mathcal{O}(\sqrt{nrT})$\\
                \midrule[1.0pt]
                \rowcolor{cyan!18}
				\Gape{\makecell[c]{\textbf{Multinoulli-OSCG}\\(\cref{thm:online} \& \cref{remark:improved_regret_oscg})}} & \Gape{\makecell[c]{$\alpha$-weakly DR-Sub \\ $(\gamma,\beta)$-weakly Sub}} & \ding{52}&\ding{52}&$\mathcal{O}\left(\frac{rn^{3}T^{2}}{\ln(n/r)}\right)$& \Gape{\makecell[c]{$\left(1-e^{-\alpha}\right)$\\ $\big(\frac{\gamma^{2}(1-e^{-\phi(\gamma,\beta)})}{\phi(\gamma,\beta)}\big)$}}&$\mathcal{O}
              (\sqrt{r\ln(n/r)T})$\\
                \midrule[1.0pt]
               \rowcolor{cyan!18}
			\Gape{\makecell[c]{\textbf{Multinoulli-OSGA}\\(\cref{thm:online1})}}& \Gape{\makecell[c]{$\alpha$-weakly DR-Sub \\ $(\gamma,\beta)$-weakly Sub}} & \ding{52}&\ding{56}&$\mathcal{O}(nT)$ &\Gape{\makecell[c]{$\left(\frac{\alpha^{2}}{1+\alpha^{2}}\right)$\\ $\big(\frac{\gamma^{2}}{\beta+\beta(1-\gamma)+\gamma^{2}}\big)$}}&$\mathcal{O}(\sqrt{nrT})$\\\midrule[1.0pt]
				\rowcolor{cyan!18}
				\Gape{\makecell[c]{\textbf{Multinoulli-OSGA}~\tnote{$\dagger$}\\(\cref{thm:online1})}}& \Gape{\makecell[c]{$\alpha$-weakly DR-Sub \\ $(\gamma,\beta)$-weakly Sub}} & \ding{56}&\ding{56}&$\mathcal{O}(nT)$& \Gape{\makecell[c]{$\left(1-e^{-\alpha}\right)$\\ $\big(\frac{\gamma^{2}(1-e^{-\phi(\gamma,\beta)})}{\phi(\gamma,\beta)}\big)$}}&$\mathcal{O}(\sqrt{nrT})$\\
				\bottomrule
	\end{tabular}}
    \begin{tablenotes}
        \footnotesize
         \item[*] `\textbf{Proj-free}' indicates whether the method is projection-free, `\textbf{Approx.Ratio}' denotes the approximation ratio and $T$ is the time
horizon.
        \item[$\ddagger$] CTA~\citep{harvey2020improved} requires exact computation of the derivatives of the multilinear extension, which leads to exponential query complexity.
        \item[$\dagger$] \cref{alg:OGA_framework} with auxiliary gradient estimation.
    \end{tablenotes}
    \end{threeparttable}
\end{table*}

%% file: Preliminaries.tex
In this section we present several important notations and concepts 
that we will frequently use throughout this paper.

\noindent\textbf{Notations:} For any positive integer $K$, $[K]$ stands for the set $\{1,\dots, K\}$. The symbols $\langle\cdot,\cdot\rangle$ and  $\|\cdot\|_{2}$ denote the inner product and the Euclidean distance, respectively. Similarly, we also use $\|\cdot\|_{0}$ and $\|\cdot\|_{1}$ to represent the $l_{0}$ norm and $l_{1}$ norm for vectors, respectively. Moreover, $\Delta_{m}$ denotes the standard $m$-dimensional simplex, i.e., $\Delta_{m}\triangleq\{(x_{1},\dots,x_{m}): \sum_{i=1}^{m}x_{i}\le 1\ \text{and}\ x_{i}\ge0,\forall i\in[m]\}$. Especially, the symbol `Multi($\p$)' represents a multinoulli distribution with $(m+1)$ possible states where $\p\in\Delta_{m}$ denotes the probability vector. Note that the multinoulli distribution is also known as the categorical distribution in machine learning~\citep{MLAPP}.

\noindent\textbf{Partition of A Set:} Given a  finite ground set $\V$, we say $\{\V_{1},\dots,\V_{K}\}$ is a partition of set $\V$ if and only if i) $\V_{i}\cap\V_{j}=\emptyset$ for any $i\neq j\in[K]$; ii) $\V=\bigcup_{k=1}^{K}\V_{k}$.

\noindent\textbf{Subset Selection under Partition Constraints:} Let $f:2^{\V}\rightarrow\R_{+}$ be a set function that maps any subset of $\V$ to a non-negative utility. Given a partition $\{\V_{1},\dots,\V_{K}\}$ of $\V$ and a collection of budgets $\{B_{1},\dots,B_{K}\}$ where $0<B_{k}\le|\V_{k}|\ \ \forall k\in[K]$, the goal of the subset selection problems subject to partition constraints is aimed at finding a subset $S$ from $\V$ such that the utility set function $f$ is maximized within the constraints $|S\cap\V_{k}|\le B_{k}$ for any $k\in[K]$, i.e.,
\begin{equation}\label{equ_problem}
	\max_{S\subseteq\V} f(S)\ \ \text{s.t.}\ |S\cap\V_{k}|\le B_{k}\ \ \forall k\in[K].
\end{equation}

\noindent\textbf{Online Partition-Constrained Subset Selection and $\rho$-Regret:} Online subset selection problem is typically formulated as a repeated game between a player and an adversary. Specifically, given a set family $\mathcal{C}$ of a ground set $\V$, at each time step $t\in[T]$, the play will first select a subset $S_{t}$ from $\mathcal{C}$. Once this choice is committed, the adversary reveals a utility set function $f_{t}:2^{\V}\rightarrow\R_{+}$ to the player. Subsequently, the player receives the reward $f_{t}(S_{t})$ and observes the utility set function $f_{t}$. In order to measure the player's performance, we usually employ the \emph{$\rho$-Regret}~\citep{hazan2016introduction,streeter2009online} metric, that is,
\begin{equation*}
    \text{Reg}_{\rho}(T)\triangleq\rho\max_{S\in\mathcal{C}}\sum_{t=1}^{T}f_{t}(S)-\sum_{t=1}^{T}f_{t}(S_{t}),
\end{equation*} where $\rho\in(0,1]$ denotes the approximation ratio. Note that, in this paper, we only consider the scenarios that the set family $\mathcal{C}$ is a partition constraint, i.e., $\mathcal{C}\triangleq\{S\subseteq\V: |S\cap\V_{k}|\le B_{k}\ \forall k\in[K]\}$ where $\{\V_{1},\dots,\V_{K}\}$ is a partition of $\V$.

\noindent\textbf{Monotonicity:} We say that a set function $f:2^{\V}\rightarrow\R_{+}$ is \emph{monotone} if and only if $f(A)\le f(B)$ for any $A\subseteq B\subseteq\V$.

\noindent\textbf{Weak Submodularity:}  Given a set function $f:2^{\V}\rightarrow\R_{+}$ and any two subsets $A,B\subseteq\V$, we denote by $f(B|A)$ the marginal contribution of adding the elements of $B$ to $A$, i.e., $f(B|A)\triangleq f(A\cup B)-f(A)$. For simplicity, when $B$ is a singleton set $\{v\}$, we also use $f(v|A)$ to represent $f(\{v\}|A)$. Therefore, we say that a set function $f:2^{\V}\rightarrow\R_{+}$ is \emph{$\gamma$-weakly submodular from below} for some $\gamma\in(0,1]$ if and only if, for any two subsets $A\subseteq B\subseteq\V$,
\begin{equation}\label{equ_lower}
	\sum_{v\in B\setminus A}f(v|A)\ge \gamma\Big(f(B)-f(A)\Big),
\end{equation} where we denote $\gamma$ as the \emph{lower submodularity ratio}. Similarly, we also can define the \emph{weak submodularity from above}, that is,  a set function $f:2^{\V}\rightarrow\R_{+}$ is \emph{$\beta$-weakly submodular from above} for some $\beta\ge1$ if and only if, $\forall A\subseteq B\subseteq\V$,
\begin{equation}\label{equ_upper}
\sum_{v\in B\setminus A}f(v|B-\{v\})\le\beta\Big(f(B)-f(A)\Big),
\end{equation}where $\beta$ is called as the \emph{upper submodularity ratio}. When a set function $f$ satisfies both Eq.\eqref{equ_lower} and Eq.\eqref{equ_upper}, we say it is \emph{$(\gamma,\beta)$-weakly submodular}~\citep{thiery2022two}.

\noindent\textbf{Weak DR-submodularity:}  A set function $f:2^{\V}\rightarrow\R_{+}$ is \emph{$\alpha$-weakly DR-submodular} for some $\alpha\in(0,1]$ if and only if $f(v|A)\ge\alpha f(v|B)$ for any two subsets $A\subseteq B\subseteq\V$ and $v\in\V\setminus B$. In particular,  $\alpha$ is often called as the diminishing-return(DR) ratio~\citep{kuhnle2018fast}. Note that, from Eq.\eqref{equ_lower} and Eq.\eqref{equ_upper}, we can infer that an $\alpha$-weakly DR-submodular function automatically satisfies the conditions for being $(\alpha, \frac{1}{\alpha})$-weakly submodular. Moreover, when $\alpha=1$,  weakly DR-submodular objectives will reduce to the classical submodular functions~\citep{nemhauser1978analysis}.


%% file: Multinoulli_Extension.tex
Generally speaking, the discrete nature of subset selection problem~\eqref{equ_problem} poses a significant challenge in finding effective solutions. In recent years, compared to discrete optimization, continuous optimization developed an array of  efficient and advanced algorithmic tools. Thus, an alternative strategy to address the subset selection problem~\eqref{equ_problem}  is to bring it into the world of continuous optimization via \emph{relaxation-rounding}  frameworks, which typically involve three critical stages: first, converting the problem~\eqref{equ_problem} into a solvable continuous optimization; second, applying the gradient-based methods to output a high-quality continuous solution; and third, rounding the previous continuous solution back to the partition constraint of Eq.~\eqref{equ_problem} without any loss in terms of the function value.  In the subsequent part of this section, we will present  a novel relaxation-rounding framework named \underline{M}ultinoulli \underline{E}xtension(\texttt{ME}) for problem \eqref{equ_problem}.

Prior to this, \citep{calinescu2011maximizing} proposed a continuous-relaxation technique known as the multi-linear extension for \emph{submodular} subset selection problems. Unfortunately, this extension cannot be directly applied to the general subset selection problem~\eqref{equ_problem} because most known lossless rounding schemes for multi-linear extension, such as pipage rounding~\citep{ageev2004pipage}, swap rounding~\citep{chekuri2010submodular} and contention resolution~\citep{chekuri2014submodular}, are heavily dependent on the \emph{submodular} assumption. Up to now, how to losslessly round the multi-linear extension of \emph{non-submodular} set functions, e.g. $(\gamma,\beta)$-weakly submodular and $\alpha$-weakly DR-submodular functions, still remains an open question~\citep{thiery2022two}. Given the unsolved rounding challenge of multi-linear extension, this paper choose to introduce a new relaxation technique named \underline{M}ultinoulli \underline{E}xtension(\texttt{ME}) to address the problem~\eqref{equ_problem}.

To provide a clearer exposition of our proposed \texttt{ME}, we first make some assumptions regarding the problem~\eqref{equ_problem}: we define $\V_{k}\triangleq\{v_{k}^{1},\dots,v_{k}^{n_{k}}\}$ for any $k\in[K]$ and set $|\V|=n$, i.e.,  $n\triangleq\sum_{k=1}^{K}n_{k}$. More specifically, the core idea of our \texttt{ME} is to learn a prior multinoulli distribution `Multi($\p_{k}$)' for each community $\V_{k}$, where $\p_{k}\triangleq(p_{k}^{1},\dots,p_{k}^{n_{k}})\in\Delta_{n_{k}}$ and each $p_{k}^{m}$ denotes the probability that element $v_{k}^{m}$ is selected within its own community $\V_{k}$ for any $m\in[n_{k}]$ and $k\in[K]$. Subsequently, \texttt{ME} employs each prior distribution `Multi($\p_{k}$)' to conduct $B_{k}$ independent random selections for every community $\V_{k}$, which can ultimately yield a subset that adheres to the partition constraint of problem~\eqref{equ_problem}. In \cref{figue_intro_framwork_multi}, we present a  three-community example  of \texttt{ME}. 
It is noteworthy that, with the probability $1-\sum_{m=1}^{n_{k}}p_{k}^{m}$, the multinoulli prior `Multi($\p_{k}$)' won't pick  any member from $\V_{k}$. In other words, sometimes we might end up with no selection, i.e., $\emptyset$. Formally,  we can define the \texttt{ME} as follows:
\begin{figure}[ht]\vspace{-0.15em}
	\centering
	\includegraphics[scale=0.55]{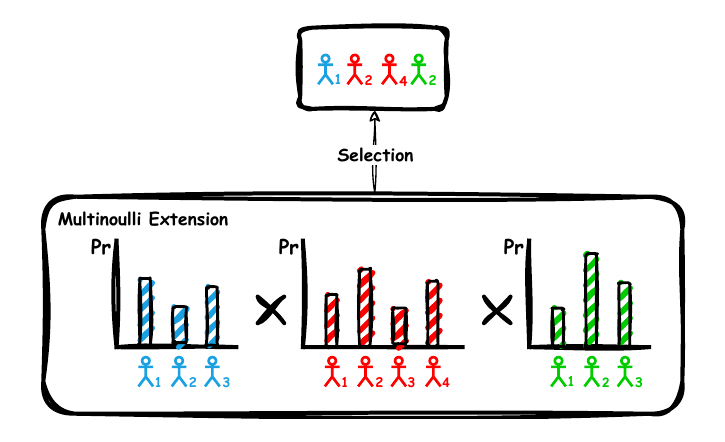} \caption{Diagram of Multinoulli Extension.}\label{figue_intro_framwork_multi}
	\vspace{-1.5em}
\end{figure}
\begin{definition}[Multinoulli Extension]\label{def_ME}	
  Given a set function $f:2^{\V}\rightarrow\R_{+}$,  its Multinoulli Extension $F:\prod_{k=1}^{K}\Delta_{n_{k}}\rightarrow\R_{+}$ for problem~\eqref{equ_problem} can be  defined as:
\begin{equation*}
	\begin{aligned}
		&F(\p_{1},\dots,\p_{K})\triangleq\E_{e_{\hat{k}}^{\hat{b}}\sim\text{Multi($\p_{\hat{k}}$)}}\Big(f\big(\cup_{\hat{k}=1}^{K}\cup_{\hat{b}=1}^{B_{\hat{k}}}\{e_{\hat{k}}^{\hat{b}}\}\big)\Big)\\
		&\triangleq\sum_{e_{\hat{k}}^{\hat{b}}\in\V_{\hat{k}}\cup\{\emptyset\}}\Big(f\big(\cup_{\hat{k}=1}^{K}\cup_{\hat{b}=1}^{B_{\hat{k}}}\{e_{\hat{k}}^{\hat{b}}\}\big)\prod_{\hat{k}=1}^{K}\prod_{\hat{b}=1}^{B_{\hat{k}}}\text{Pr}(e_{\hat{k}}^{\hat{b}}|\p_{\hat{k}})\Big),
	\end{aligned}
\end{equation*}where each $e_{\hat{k}}^{\hat{b}}$ denotes the element chosen at the $\hat{b}$-th random trail of community $\V_{\hat{k}}, \ \forall \hat{b}\in[B_{\hat{k}}] \forall \hat{k}\in[K]$ and the symbol $e_{\hat{k}}^{\hat{b}}\sim\text{Multi($\p_{\hat{k}}$)}$ indicates that element $e_{\hat{k}}^{\hat{b}}$
	is  randomly selected from $\V_{\hat{k}}\cup\{\emptyset\}$ based on the multinoulli distribution `Multi($\p_{\hat{k}}$)', that is, $\text{Pr}(v_{\hat{k}}^{m} | \p_{\hat{k}})\triangleq p^{m}_{\hat{k}},\ \forall m\in[n_{\hat{k}}],\forall\hat{k}\in[K]$ and $\text{Pr}(\emptyset | \p_{\hat{k}}) \triangleq 1 - \sum_{m=1}^{n_{\hat{k}}} p_{\hat{k}}^{m},\ \forall\hat{k}\in[K]$.
\end{definition}

\begin{remark}
The introduction of \texttt{ME} is aimed at converting the general subset selection problem~\eqref{equ_problem} into a continuous maximization task focused on identifying the optimal multinoulli priors across the partition $\{\V_{1},\dots,\V_{K}\}$. Specifically, we hope to address the following continuous optimization:
\begin{equation}\label{equ_prob_relaxed}
	\max_{p_{k}^{m}\ge0} F(\p_{1},\dots,\p_{K})\ \text{s.t.} \sum_{m=1}^{n_{k}}p_{k}^{m}\le1,\forall k\in[K].
\end{equation} 
\end{remark}
\begin{remark}\label{remark:rounding}
In comparison with the multi-linear extension~\citep{calinescu2011maximizing}, a notable  advantage of our \texttt{ME} is that it does not assign probabilities to any subsets that are out of the partition constraint of problem~\eqref{equ_problem}. This means that, for \emph{any} set function $f:2^{\V}\rightarrow\R_{+}$ and \emph{any} given $(\p_{1},\dots,\p_{K})\in\prod_{k=1}^{K}\Delta_{n_{k}}$, we can, through the definition of \texttt{ME}, produce a subset that conforms to the partition constraint of problem~\eqref{equ_problem} without any loss in terms of the expected function value $F(\p_{1},\dots,\p_{K})$, where $F$ represents the \texttt{ME} of $f$. For more details about multi-linear extension, please refer to Appendix~\ref{appendix:Multi_linear_Extension}.
\end{remark}
Although the \texttt{ME} is naturally  endowed with a lossless rounding scheme for \emph{any} set function, there are two crucial questions that must be answered in order to leverage this tool to tackle the subset selection problem~\eqref{equ_problem} : \textbf{i):} \emph{What is the relationship between the function value of our proposed Multinoulli Extension $F$ and the original set function $f$?} \textbf{ii):} \emph{How to solve the relaxed problem~\eqref{equ_prob_relaxed}?} For the rest of this section, we will focus on the first question. The exploration of the second question will be presented in  Section~\ref{sec:algorithm} and \ref{sec:algorithm1}.

\subsection{The Fundamental Properties of Multinoulli Extension}\label{Sec:ME:property}
In this subsection, we will concentrate on characterizing several important properties about our proposed \texttt{ME}. Specifically, we have the following theorem:
\begin{theorem}[Proof provided in Appendix~\ref{append:proof1}]\label{thm1}
For a set function $f:2^{\V}\rightarrow\R_{+}$, its Multinoulli Extension $F:\prod_{k=1}^{K}\Delta_{n_{k}}\rightarrow\R_{+}$ for problem~\eqref{equ_problem}, as described in the Definition~\ref{def_ME}, satisfies the following properties:

\textbf{1):} The first-order partial derivative of $F$ at any point $(\p_{1},\dots,\p_{K})\in\prod_{k=1}^{K}\Delta_{n_{k}}$ can be  written as follows:
\begin{equation*}
	\frac{\partial F}{\partial p_{k}^{m}}\triangleq B_{k}\Bigg(\E_{e_{\hat{k}}^{\hat{b}}\sim\emph{Multi($\p_{\hat{k}}$)}}\Big(f\Big(v_{k}^{m}\Big|\cup_{(\hat{k},\hat{b})\neq(k,1)}\{e_{\hat{k}}^{\hat{b}}\}\Big)\Big)\Bigg),
\end{equation*}for any $k\in[K]$ and $m\in[n_{k}]$;

\textbf{2):} If $f$ is monotone, then $\frac{\partial F}{\partial p_{k}^{m}}\ge 0, \forall k\in[K], m\in[n_{k}]$;

\textbf{3):} If $f$ is $\alpha$-weakly DR-submodular, then $F$ is $\alpha$-weakly continuous DR-submodular~\citep{hassani2017gradient,zhang2022boosting} over the domain $\prod_{k=1}^{K}\Delta_{n_{k}}$, that is, for any two point $(\p_{1},\dots,\p_{K})\in\prod_{k=1}^{K}\Delta_{n_{k}}$ and $(\hat{\p}_{1},\dots,\hat{\p}_{K})\in\prod_{k=1}^{K}\Delta_{n_{k}}$, if  $\hat{\p}_{k}\ge\p_{k}$ for any $k\in[K]$, we have that
\[\nabla F(\p_{1},\dots,\p_{K})\ge \alpha \nabla F(\hat{\p}_{1},\dots,\hat{\p}_{K});\]
\textbf{4):}  If $f$ is $\gamma$-weakly submodular from below, then $F$ is upper-linearizable~\citep{pedramfar2024from} over the domain $\prod_{k=1}^{K}\Delta_{n_{k}}$, that is, for any two point $\x\triangleq(\p_{1},\dots,\p_{K})\in\prod_{k=1}^{K}\Delta_{n_{k}}$ and $\hat{\x}\triangleq(\hat{\p}_{1},\dots,\hat{\p}_{K})\in\prod_{k=1}^{K}\Delta_{n_{k}}$, if $\hat{\p}_{k}\ge\p_{k}$ for any $k\in[K]$, we have that
\vspace{-0.3em}
\begin{equation*}
\gamma\Big(F(\hat{\x})-F(\x)\Big)\le\Big\langle\hat{\x}-\x, \nabla F(\x)\Big\rangle.
\end{equation*}  
\vspace{-0.8em}
\end{theorem}
\begin{remark}\label{remark:gradient_estimation}
	The first point of Theorem~\ref{thm1} provides a characterization about the gradient of our proposed \texttt{ME} $F$, which indicates that we can estimate $\nabla F$  by sampling a sequence of random elements. Specifically, when each $e_{\hat{k}}^{\hat{b}}$ is independently drawn from the multinoulli distribution `Multi($\p_{\hat{k}}$)' for any $\hat{k}\in[K]$ and $\hat{b}\in[B_{\hat{k}}]$, we can estimate $\frac{\partial F}{\partial p_{k}^{m}}$ as  $\widehat{\frac{\partial F}{\partial p_{k}^{m}}}\triangleq B_{k}\left(f\big(v_{k}^{m}\Big|\cup_{(\hat{k},\hat{b})\neq(k,1)}\{e_{\hat{k}}^{\hat{b}}\}\big)\right)$.
\end{remark}
\begin{remark}
    The second point indicates that the mononicity of the set function $f$ can be inheritable by its \texttt{ME}. Furthermore, the third and fourth points reveal that when the set function $f$ exhibits the weak DR-submodularity or weak submodularity, its corresponding \texttt{ME} is weakly continuous  DR-submodular or upper-linearizable over the domain $\prod_{k=1}^{K}\Delta_{n_{k}}$. 
\end{remark}
\begin{remark}
	 Note that both upper-linearizable and weakly continuous DR-submodular functions defined over the box constraint $[0,1]^{n}$ have been extensively studied by~\citep{hassani2020stochastic,hassani2017gradient,zhang2022boosting,pedramfar2024from,wan2023bandit}. However, it is crucial to emphasize that these former results  cannot be directly applied to our \texttt{ME}.  This is because all of them require the domain of objective functions to be closed under the coordinate-wise maximum operation $\vee$, i.e., $\x\vee\y\triangleq\max(\x,\y)$.	Unfortunately, the domain $\prod_{k=1}^{K}\Delta_{n_{k}}$ of our proposed \texttt{ME} does not meet this requirement. For further details, please refer to Appendix~\ref{appendix:non_close}.
\end{remark}
Next, we uncover the relationship between the function value of our proposed \texttt{ME} $F$ and that of the original  $f$. To be more precise, we have the following theorem:

\begin{theorem}[Proof provided in Appendix~\ref{append:proof2}]\label{thm2}
When the set function $f:2^{\V}\rightarrow\R_{+}$ is monotone and $\alpha$-weakly DR-submodular, for any subset $S$ within the partition constraint of problem~\eqref{equ_problem} and any point $\x\triangleq(\p_{1},\dots,\p_{K})\in\prod_{k=1}^{K}\Delta_{n_{k}}$, the following inequality holds:
\begin{equation*}
	\alpha\Big(f(S)-F(\x)\Big)\le\left\langle\sum_{k=1}^{K}\frac{1}{B_{k}}\one_{(S\cap\V_{k})},\nabla F(\x)\right\rangle,
\end{equation*} where the symbol $\one_{S}$ is the indicator function over the set $S$, meaning that, for any element $v_{k}^{m}\in S$, the vector $\one_{S}$ sets the corresponding coordinate of its probability  $p_{k}^{m}$ to be $1$; otherwise, $0$. Similarly, when the set function $f:2^{\V}\rightarrow\R_{+}$ is monotone and $(\gamma,\beta)$-weakly submodular, for any subset $S$ within the partition constraint of problem~\eqref{equ_problem} and any point $\x\triangleq(\p_{1},\dots,\p_{K})\in\prod_{k=1}^{K}\Delta_{n_{k}}$, we can infer that
\begin{equation*}
\gamma^{2}f\big(S\big)-(\beta(1-\gamma)+\gamma^{2})F\big(\x\big)\le\left\langle\sum_{k=1}^{K}\frac{\one_{(S\cap\V_{k})}}{B_{k}},\nabla F(\x)\right\rangle.
\end{equation*}
\end{theorem}
\begin{remark}
Theorem~\ref{thm2} implies that when $f$ is monotone and $\alpha$-weakly DR-submodular,  for any $S$ within the partition constraint of problem~\eqref{equ_problem} and any point $\x\in\prod_{k=1}^{K}\Delta_{n_{k}}$,
the discrepancy between $f(S)$ and $F(\x)$ can be bounded by the inner product $\left\langle\sum_{k=1}^{K}\frac{1}{B_{k}}\one_{(S\cap\V_{i})},\nabla F(\x)\right\rangle$. Similarly, if $f$ is monotone and $(\gamma,\beta)$-weakly submodular, this inner product $\left\langle\sum_{k=1}^{K}\frac{1}{B_{k}}\one_{(S\cap\V_{i})},\nabla F(\x)\right\rangle$ also can bound the gap between $\gamma^{2}f(S)$ and $(\beta(1-\gamma)+\gamma^{2})F(\x)$.
\end{remark}

%% file: Multinoulli_Stationary_Point.tex
Broadly speaking, with the aid of \texttt{ME} $F$, a feasible strategy for solving the subset selection problem~\eqref{equ_problem} is to initially apply the off-the-shelf gradient-based methods, such as gradient ascent and Frank Wolfe, to the relaxed problem~\eqref{equ_prob_relaxed}, and subsequently, to finalize our selection by rounding the resulting continuous solution. Furthermore, as extensively documented in the literature~\citep{hassani2017gradient,hassani2020stochastic}, under mild conditions, a wide range of gradient-based algorithms will converge to the stationary points of their target objectives. Motivated by these findings, we next investigate the stationary points of our proposed \texttt{ME}. 

Before going into the details,  we firstly revisit the definition of stationary points for maximization problem, that is, 
\begin{definition}\label{def:stationary} Given a differentiable  objective function $G:\mathcal{K}\rightarrow\R$ and a domain $\mathcal{M}\subseteq\mathcal{K}$, a point $\x\in\mathcal{M}$ is called as a stationary point for the function $G$ over the domain $\mathcal{M}$ if and only if $\max_{\y\in\mathcal{M}}\langle\y-\x,\nabla G(\x)\rangle\le 0$.
\end{definition}
Next, we will detail the specific performance of the stationary points of our proposed Multinoulli Extension relative to the maximum value of problem~\eqref{equ_problem}. Specifically, we have that

\begin{theorem}[Proof provided in Appendix~\ref{append:proof3}]\label{thm3} If the set function $f:2^{\V}\rightarrow\R_{+}$ is monotone and $\alpha$-weakly DR-submodular, then for any stationary point $(\p_{1},\dots,\p_{K})$ of its Multinoulli Extension $F$ over the domain $\prod_{k=1}^{K}\Delta_{n_{k}}$, the following inequality holds:
	\begin{equation*}
		F(\p_{1},\dots,\p_{K})\ge\Big(\frac{\alpha^{2}}{1+\alpha^{2}}\Big)f(S^{*}),
		\end{equation*} where $S^{*}$ is the optimal solution of problem~\eqref{equ_problem}. Similarly, when the set function $f:2^{\V}\rightarrow\R_{+}$ is monotone and $(\gamma,\beta)$-weakly submodular, for any stationary point $(\p_{1},\dots,\p_{K})$ over the domain $\prod_{k=1}^{K}\Delta_{n_{k}}$ , we also can show that
	\begin{equation*}
		F(\p_{1},\dots,\p_{K})\ge\Big(\frac{\gamma^{2}}{\beta+\beta(1-\gamma)+\gamma^{2}}\Big)f(S^{*}).
		\end{equation*} 
\end{theorem}
\begin{remark}It is worth noting that the approximation ratios established in Theorem~\ref{thm3} is tight for stationary points. In Appendix~\ref{appendix:bad_case_stationary_point}, we will present a simple instance of a submodular function, i.e., $c=\alpha=\gamma=\beta=1$, whose \texttt{ME}  can attain a $(1/2)$-approximation guarantee at a stationary point.
\end{remark}

\begin{wrapfigure}{l}{0.18\textwidth}
\includegraphics[width=0.18\textwidth]{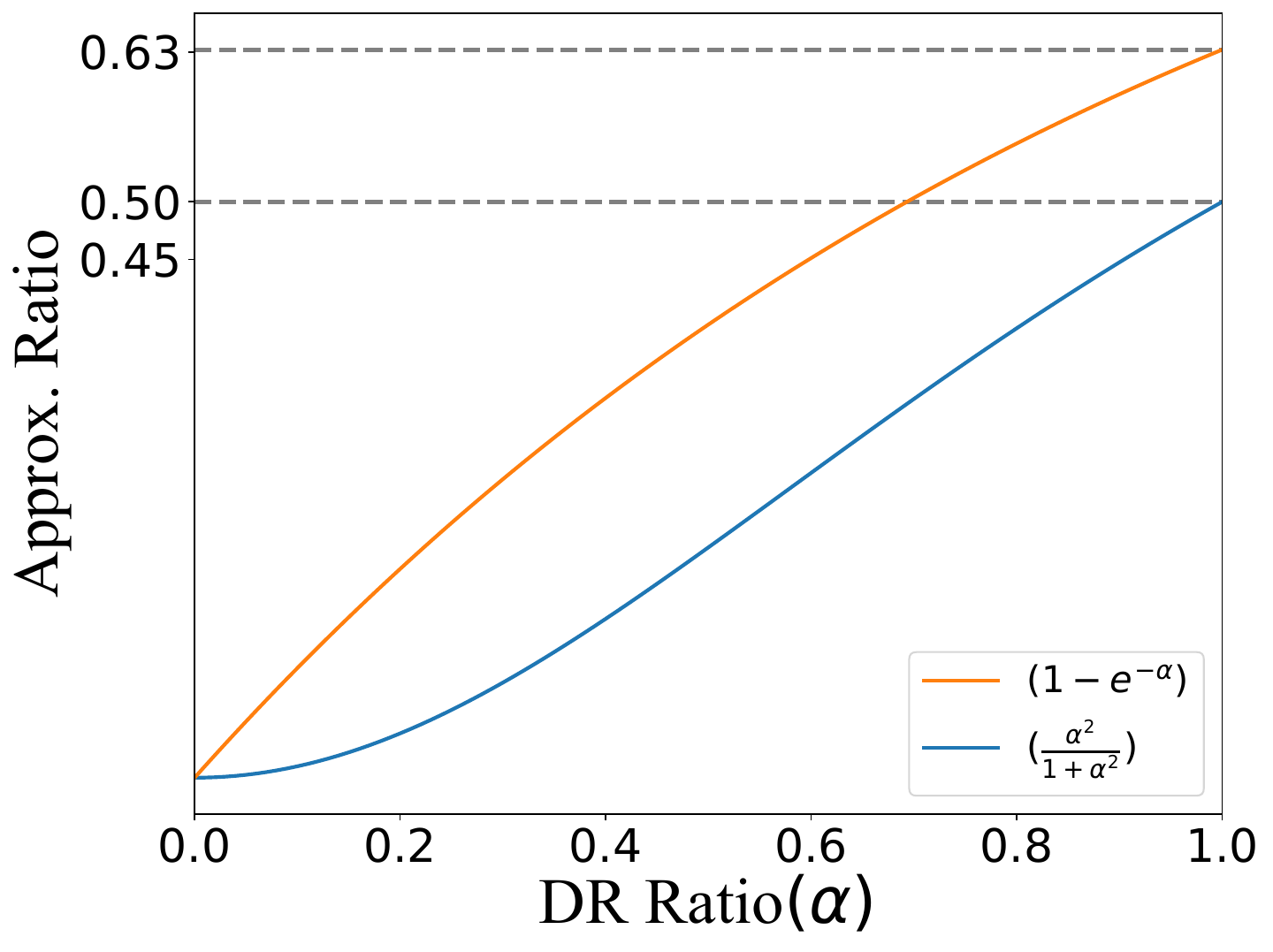}
  		\captionsetup{font=scriptsize}
  		\caption{$(\frac{\alpha^{2}}{1+\alpha^{2}})$ v.s. $(1-e^{-\alpha})$.}\label{figure:2}
 		\vspace{-1.0em}
  	\end{wrapfigure} As revealed by Theorem~\ref{thm3}, when the set function $f$ is  monotone $\alpha$-weakly DR-submodular or $(\gamma,\beta)$-weakly  submodular,   the direct deployment of gradient-based algorithms targeting at stationary points on the relaxed problem~\eqref{equ_prob_relaxed} only can secure an approximation ratio of $(\frac{\alpha^{2}}{1+\alpha^{2}})$ or $(\frac{\gamma^{2}}{\beta+\beta(1-\gamma)+\gamma^{2}})$ relative to the optimum of the subset selection problem~\eqref{equ_problem}. In comparison, \citep{harshaw2019submodular} demonstrated that the optimal approximation guarantee for maximizing monotone $\alpha$-weakly DR-submodular functions is $(1-e^{-\alpha})$, which substantially outperforms the aforementioned $(\frac{\alpha^{2}}{1+\alpha^{2}})$-approximation provided by stationary points of \texttt{ME} $F$, as demonstrated  in Figure~\ref{figure:2}. This significant disparity naturally raises the question: \emph{Can we devise approaches to close this non-negligible gap?}

Previously, the works~\citep{filmus2014monotone,zhang2022boosting,wan2023bandit} effectively employed the well-established \emph{Non-Oblivious Search} technique to achieve improved solution quality via the design of specialized auxiliary functions.  Motivated by these advances, we also hope to construct a surrogate function capable of strengthening the approximation bounds for stationary points of our proposed \texttt{ME}. Following the methodology developed in~\citep{filmus2014monotone,zhang2022boosting,wan2023bandit}, we consider a type of auxiliary functions $F^{a}(\x)$ re-weights gradients of our \texttt{ME} at scaled locations $z*\x$ through the following formulation:
\begin{equation}\label{equ:auxiliary}
        \nabla F^{a}(\x)\triangleq\int_{0}^{1} w(z)\nabla F(z*\x)\mathrm{d}z,
    \end{equation} where $w(z)$ denotes a positive weighting function defined on $[0,1]$. After carefully
selecting $w(z)$, we can establish that:
\begin{theorem}[Proof provided in Appendix~\ref{append:proof4+}]\label{thm4+} When $f$ is monotone $\alpha$-weakly DR-submodular and $w(z)\triangleq e^{\alpha(z-1)}$, for any stationary point $(\p^{a}_{1},\dots,\p^{a}_{K})$ of the auxiliary function $F^{a}$ over the domain $\prod_{k=1}^{K}\Delta_{n_{k}}$, the following inequality holds:
\begin{equation*}
   F(\p^{a}_{1},\dots,\p^{a}_{K})\ge\big(1-e^{-\alpha}\big)f(S^{*}),
\end{equation*} where $S^{*}$ is the optimal solution of problem~\eqref{equ_problem}. Similarly, if $f$ is  monotone $(\gamma,\beta)$-weakly submodular and $w(z)\triangleq e^{\phi(\gamma,\beta)(z-1)}$ where $\phi(\gamma,\beta)\triangleq\beta(1-\gamma)+\gamma^2$, for any stationary point $(\p^{a}_{1},\dots,\p^{a}_{K})$ of the auxiliary function $F^{a}$ over the domain $\prod_{k=1}^{K}\Delta_{n_{k}}$, then we can show that 
\begin{equation*}
F(\p^{a}_{1},\dots,\p^{a}_{K})\ge\big(\frac{\gamma^{2}(1-e^{-\phi(\gamma,\beta)})}{\phi(\gamma,\beta)}\big)f(S^{*}). \end{equation*}
	 \end{theorem}
\cref{thm4+} indicates that, when the set function $f$ is monotone $\alpha$-weakly DR-submodular or  $(\gamma,\beta)$-weakly submodular,  the direct application of gradient-based
methods targeting stationary points to the auxiliary function $F^{a}$ can achieve a tight approximation ratio of $(1-e^{-\alpha})$ or $\big(\frac{\gamma^{2}(1-e^{-\phi(\gamma,\beta)})}{\phi(\gamma,\beta)}\big)$ to the subset selection problem~\eqref{equ_problem}, respectively.

%% file: Rounding_without_Replacement.tex
\begin{algorithm}[t]
	\caption{Rounding without Replacement}\label{alg_without_replacement}
	{\bf Input:} Priors $(\p_{1},\dots,\p_{K})$, budgets $\{B_{1},\dots,B_{K}\}$ and partition $\{\V_{1},\dots,\V_{K}\}$ of set $\V$
	\begin{algorithmic}[1]
		\FOR{$k=1,\dots,K$}
        \IF[Normalization (Lines 2-6)]{$\|\p_{k}\|_{1}>0$}
        \STATE Normalize $\p_{k}\triangleq\frac{\p_{k}}{\|\p_{k}\|_{1}}$;
        \ELSE
		\STATE Set $\p_{k}\triangleq(\frac{1}{n_{k}},\dots,\frac{1}{n_{k}})\in\R^{n_{k}}$;
        \ENDIF\ \COMMENT{Sampling without Replacement (Lines 7-8)}
        \STATE Compute $B^{k}_{min}\triangleq\min(\|\p_{k}\|_{0},B_{k})$;
		\STATE Utilize multinoulli prior $\p_{k}$ to \emph{sample} $B^{k}_{min}$ different elements $\{e_{k}^{1},\dots,e_{k}^{B^{k}_{min}}\}$ from $\V_{k}\triangleq\{v_{k}^{1},\dots,v_{k}^{n_{k}}\}$ \emph{without replacement};
        \IF[{\small Complement Sampling(Lines 9-11)}]{$B^{k}_{min}<B_{k}$}
        \STATE Generate $B_{k}-B^{k}_{min}$ elements $\{e_{k}^{B^{k}_{min}+1},\dots,e_{k}^{B_{k}}\}$ \emph{randomly} from $\V_{k}\setminus\{e_{k}^{1},\dots,e_{k}^{B^{k}_{min}}\}$;
        \ENDIF
		\ENDFOR
		\STATE \textbf{Return} $S\triangleq\cup_{k=1}^{K}\cup_{b=1}^{B_{k}}\{e_{k}^{b}\}$;
	\end{algorithmic}
\end{algorithm}

In the previous two subsections, we have thoroughly investigated the relationships between our proposed \texttt{ME} $F$ and its original set function $f$, which provides a solid theoretical foundation for finding exceptional multinoulli priors $(\p_{1},\dots,\p_{K})$ for the relaxed problem~\eqref{equ_prob_relaxed}. However, a critical question remains: \emph{how can we losslessly and effectively transfer these resulting  multinoulli priors $(\p_{1},\dots,\p_{K})$ into  a subset of our target partition constraint?} A straightforward approach, as outlined  in \cref{def_ME}, is to independently utilize each multinoulli prior $\p_{k}$ a total of $B_{k}$ times to generate candidates $\{e_{k}^{1},\dots,e_{k}^{B_{k}}\}$ from $\V_{k}\cup\{\emptyset\}$. Yet, there is a chance that two different candidates  may choose the same element from $\V_{k}$ or some candidate may end up being empty, thereby causing the final subset size to be strictly less than $B_{k}$, i.e., $|\cup_{b=1}^{B_{k}}\{e_{k}^{b}\}|<B_{k}$. Particularly when the original set function$f$ is strictly \emph{monotone}, such issues will significantly impair  the practical effectiveness of our final selected subset. To overcome this drawback, we propose a simple and effective rounding-without-replacement method for \emph{monotone} set objective functions, as detailed in \cref{alg_without_replacement}. 
 This method not only ensures that the final selected subset contains exactly $B_{k}$ distinct elements in each $\V_{k},\forall k\in[K]$, but can guarantee that the function value of the resulting subset is greater than that of the corresponding multinoulli priors, that is to say,
 \begin{theorem}[Proof in Appendix~\ref{append:rounding_without_replacement}]\label{thm5+}
 When the set function $f$ is monotone, for any multinoulli priors $(\p_{1},\dots,\p_{K})\in\prod_{k=1}^{K}\Delta_{n_{k}}$, we can show that the subset $S$ yielded by \cref{alg_without_replacement} satisfies the following properties: \textbf{1):} $|S\cap\V_{k}|\equiv B_{k}$ for any $k\in[K]$; \textbf{2):} $\E(f(S))\ge F(\p_{1},\dots,\p_{K})$ where $F$ is the multinoulli extension of $f$ for problem~\eqref{equ_problem}. 
 \end{theorem}
  

%% file: Multinoulli-SCG.tex
This subsection aims to present an effective method named \textbf{Multinoulli-SCG} to maximize our introduced \texttt{ME}, which is primarily inspired by the \emph{continuous greedy}(CG) algorithm for the multi-linear extension~\citep{calinescu2011maximizing}.

The CG algorithm typically comprises two critical steps: First, it begins with $\x(1)\triangleq\textbf{0}$ and then, at each iteration $t\in[T]$,  the algorithm identifies the optimal ascent direction  $\v(t)\triangleq\mathop{\arg\max}_{\v\in\mathcal{M}}\left\langle\v,\nabla G\big(\x(t)\big)\right\rangle$ to update the current variable $\x(t)$ according to the rule $\x(t+1)\triangleq\x(t)+\frac{1}{T}\v(t)$ where $G$ is the multi-linear extension of a set function $f$ and $\mathcal{M}$ is a convex domain. The motivation behind the CG algorithm is that when $f$ is \emph{monotone submodular}, the inequality $\left\langle\y,\nabla G\big(\x\big)\right\rangle\ge G(\y)-G(\x)$ holds for any feasible $\x,\y$. Note that, in Theorem~\ref{thm2}, we establish a similar relationship between our proposed \texttt{ME} $F$ and its original set function $f$, that is , the inner product $\left\langle\sum_{k=1}^{K}\frac{1}{B_{k}}\one_{S\cap\V_{i}},\nabla F(\x)\right\rangle$ can bound a specific form of weighted discrepancy between $f(S)$ and $F(\x)$, when $f$ is monotone $\alpha$-weakly DR-submodular or $(\gamma,\beta)$-weakly submodular. Motivated by these findings, we naturally consider a two-step variant of CG algorithm to maximize our proposed \texttt{ME}: Initially, we set $\x(1)\triangleq\mathbf{0}$, and then, at each iteration $t \in [T]$, we find the optimal  subset $S(t)\triangleq\mathop{\arg\max}_{S\subseteq\V, |S\cap\V_{k}|\le B_{k},\forall k\in[K]}\left\langle \nabla F(\x(t)), \sum_{k=1}^{K} \frac{\one_{(S\cap \V_{k})}}{B_{k}}\right\rangle$ to update the current variable $\x(t)$ according to the rule $\x(t+1)\triangleq\x(t) + \frac{1}{T} \sum_{k=1}^{K} \frac{1}{B_{k}} \one_{(S(t) \cap \V_{k})}$.

However, the implementation of the aforementioned two-step algorithm requires accurately computing the gradients of  our proposed \texttt{ME} $F$, which is typically computationally intensive. To circumvent this obstacle,  we adopt the stochastic path-integrated differential estimator~\citep{hassani2020stochastic,fang2018spider}, namely, for a sequence of iterations $\{\x(t)\}_{t=1}^{T+1}$, we estimate each gradient $\nabla F\big(\x(t)\big)(t\ge2)$ by the following path-integral form: $\widetilde{\nabla}F\big(\x(t)\big)\triangleq\nabla F\big(\x(1)\big)+\sum_{s=2}^{t}\boldsymbol{\xi}_{s}$,
where each $\boldsymbol{\xi}_{s}$ is an unbiased estimator for the difference $\textbf{Diff}_{s}\triangleq\nabla F(\x(s))-\nabla F(\x(s-1))$. Note that, when $F$ is twice differentiable, for any $2\le s\le T+1$, the difference $\textbf{Diff}_{s}$ can be rewritten as: 
\begin{equation*}
\textbf{Diff}_{s}\triangleq\int_{0}^{1}\nabla^{2}F(\x^{a}(s))\mathrm{d}a\Big(\x(s)-\x(s-1)\Big),  
\end{equation*}where  $\x^{a}(s)\triangleq a\x(s)+(1-a)\x(s-1),\forall a\in[0,1]$ and $\nabla^{2}F$ is the Hessian of $F$. Hence, if we uniformly sample the parameter $a$ from $[0,1]$, the difference $\textbf{Diff}_{s}$ can be unbiasedly estimated by $\boldsymbol{\xi}_{s}\triangleq\widetilde{\nabla}^{2}F(\x^{a}(s))\Big(\x(s)-\x(s-1)\Big)$ where $\widetilde{\nabla}^{2}F$ is an unbiased estimator of $\nabla^{2}F$. Following this idea, we proceed to demonstrate how to estimate the second-order derivative  of our proposed \texttt{ME} $F$, that is to say, 
\begin{theorem}[Proof provided in \cref{append:proof4}]\label{thm4}
	For a set function $f:2^{\V}\rightarrow\R_{+}$, the second-order derivative of its Multinoulli Extension $F$ at any point $(\p_{1},\dots,\p_{K})\in\prod_{k=1}^{K}\Delta_{n_{k}}$ can be  written as follows: 
	
	\textbf{1):} If $k_{1}\neq k_{2}\in[K]$, for any $m_{1}\in[n_{k_1}]$ and $m_{2}\in[n_{k_2}]$,
	\begin{equation*}
		\frac{\partial^{2}F}{\partial p_{k_{1}}^{m_{1}}\partial p_{k_{2}}^{m_{2}}}\triangleq B_{k_{1}}B_{k_{2}}\E\Big(f\Big(v_{k_{1}}^{m_{1}}\Big|S\cup\{v_{k_{2}}^{m_{2}}\}\Big)-f\Big(v_{k_{1}}^{m_{1}}\Big|S\Big)\Big),
	\end{equation*}
	where $S\triangleq\cup_{(\hat{k},\hat{b})\neq\left\{(k_{1},1),(k_{2},1)\right\},\hat{k}\in[K],\hat{b}\in[B_{\hat{k}}]}\{e_{\hat{k}}^{\hat{b}}\}$ and each  $e_{\hat{k}}^{\hat{b}}$ is drawn from the multinoulli distribution Multi($\p_{\hat{k}}$); 
	
	\textbf{2):} As for $k_{1}=k_{2}=k\in[K]$, if $B_{k}\equiv1$, for any $m_{1},m_{2}\in[n_{k}]$, we have $\frac{\partial^{2}F}{\partial p_{k}^{m_{1}}\partial p_{k}^{m_{2}}}\triangleq 0$; Moreover, when $B_{k}\ge2$, 
	\begin{equation*}
		\frac{\partial^{2}F}{\partial p_{k}^{m_{1}}\partial p_{k}^{m_{2}}}\triangleq(B_{k}^{2}-B_{k})\E\Big(f\Big(v_{k}^{m_{1}}\Big|S\cup\{v_{k}^{m_{2}}\}\Big)-f\Big(v_{k}^{m_{1}}\Big|S\Big)\Big), 
	\end{equation*} where $S\triangleq\cup_{(\hat{k},\hat{b})\neq\left\{(k,1),(k,2)\right\}}\{e_{\hat{k}}^{\hat{b}}\}$ and each  $e_{\hat{k}}^{\hat{b}}$ is independently drawn from the multinoulli distribution Multi($\p_{\hat{k}}$). 
\end{theorem}
\begin{remark}~\label{remark:estimate_2_order}
	Theorem~\ref{thm4}  provides a detailed characterization of the second-order derivative of our proposed \texttt{ME} $F$, which implies that we can estimate the Hessian of $F$  by sampling a sequence of random elements. Specifically, when each $e_{\hat{k}}^{\hat{b}}$ is independently drawn from the multinoulli distribution `Multi($\p_{\hat{k}}$)' for any $\hat{k}\in[K]$ and $\hat{b}\in[B_{\hat{k}}] $, we can estimate $\frac{\partial^{2}F}{\partial p_{k_{1}}^{m_{1}}\partial p_{k_{2}}^{m_{2}}}(\p_{1},\dots,\p_{K})$ as:
	when $k_{1}\neq k_{2}\in[K]$,  
	\begin{equation*}
		\widehat{\frac{\partial^{2}F}{\partial p_{k_{1}}^{m_{1}}\partial p_{k_{2}}^{m_{2}}}}\triangleq B_{k_{1}}B_{k_{2}}\Big(f\Big(v_{k_{1}}^{m_{1}}\Big|S\cup\{v_{k_{2}}^{m_{2}}\}\Big)-f\Big(v_{k_{1}}^{m_{1}}\Big|S\Big)\Big),
	\end{equation*} where $S=\cup_{(\hat{k},\hat{b})\neq\left\{(k_{1},1),(k_{2},1)\right\},\hat{k}\in[K],\hat{b}\in[B_{\hat{k}}]}\{e_{\hat{k}}^{\hat{b}}\}$; As for $k_{1}=k_{2}=k\in[K]$ and $B_{k}\ge2$,  
	\begin{equation*}
		\widehat{\frac{\partial^{2}F}{\partial p_{k}^{m_{1}}\partial p_{k}^{m_{2}}}}\triangleq (B_{k}^{2}-B_{k})\Big(f\Big(v_{k}^{m_{1}}\Big|S\cup\{v_{k}^{m_{2}}\}\Big)-f\Big(v_{k}^{m_{1}}\Big|S\Big)\Big),
	\end{equation*} where $S=\cup_{(\hat{k},\hat{b})\neq\left\{(k,1),(k,2)\right\},\hat{k}\in[K], \hat{b}\in[B_{\hat{k}}]}\{e_{\hat{k}}^{\hat{b}}\}$. 
\end{remark}

Incorporating this second-order approximation and the idea of differential estimator into the previously mentioned two-step variant of CG algorithm, we can develop a stochastic algorithm for maximizing our  \texttt{ME}, as detailed in Algorithm~\ref{alg:scg_mutli}.

\begin{algorithm}[t]
	\caption{Stochastic Continuous Greedy Algorithm for Multinoulli Extension(\textbf{Multinoulli-SCG})}\label{alg:scg_mutli}
	{\bf Input:} Batch size $L$, number of iterations $T$, set function $f$ and partition $\{\V_{1},\dots,\V_{K}\}$ of set $\V$ 
	\begin{algorithmic}[1]
		\STATE {\bf Initialize:} $\x(1)=\left(\p_{1}(1),\dots,\p_{K}(1)\right)=\textbf{0}_{n}$;
		\FOR{$t=1,\dots,T$}
		\IF[Differential Estimator(Lines 3-12)]{$t=1$}
		\STATE Compute $\boldsymbol{\mathrm{g}}(1)\triangleq\nabla F(\textbf{0}_{n})$ based on \cref{thm1};	
		\ELSE
		\STATE Sample $\{a(1),\dots,a(L)\}$ uniformly from $[0,1]$;
		\STATE Set $\x_{l}(t)\triangleq a(l)\x(t)+(1-a(l))\x(t-1),\forall l\in[L]$; 
		\STATE Compute the Hessian estimator $\widehat{\nabla}^{2} F\big(\x_{l}(t)\big)$ based on Remark~\ref{remark:estimate_2_order} for any $l\in[L]$;
		\STATE Compute $\widehat{\nabla}^{2}_{t}\triangleq\frac{1}{L}\sum_{l=1}^{L} \widehat{\nabla}^{2}F\big(\x_{l}(t)\big)$;
		\STATE Compute  $\boldsymbol{\xi}_{t}\triangleq\widehat{\nabla}^{2}_{t}\big(\x(t)-\x(t-1)\big)$;
		\STATE Aggregate the estimator $\boldsymbol{\mathrm{g}}(t)=\boldsymbol{\mathrm{g}}(t-1)+\boldsymbol{\xi}_{t}$;
		\ENDIF\ \ \COMMENT{Stochastic Continuous Greedy(Lines 13-14)}
		\STATE $S(t)\triangleq\mathop{\arg\max}_{|S\cap\V_{k}|\le B_{k}}\left\langle\boldsymbol{\mathrm{g}}(t), \sum_{k=1}^{K}\frac{1}{B_{k}}\one_{S\cap\V_{k}}\right\rangle$;
		\STATE Update $\x(t+1)\triangleq\x(t)+\frac{1}{T}\sum_{k=1}^{K}\frac{1}{B_{k}}\one_{S(t)\cap\V_{k}}$;
		\ENDFOR
		\STATE  \textbf{Return} $S$ by rounding $\x(T+1)$ without replacement;
	\end{algorithmic}
\end{algorithm}

Furthermore, based on the results of Theorem~\ref{thm1} and Theorem~\ref{thm2}, we can show that
\begin{theorem}[Proof provided in Appendix~\ref{append:proof5}]\label{thm5}When the set function $f:2^{\V}\rightarrow\R_{+}$ is monotone and $\alpha$-weakly DR-submodular, if we set the batch size $L=\mathcal{O}(T)$, the subset $S$ output by Algorithm~\ref{alg:scg_mutli} satisfies:
	\begin{equation*}
		\E\big(f(S)\big)\ge\Big(1-e^{-\alpha}\Big)f(S^{*})-\mathcal{O}\Big(\frac{r\sqrt{n}}{T}\Big),
	\end{equation*} where $S^{*}$ is the optimal solution of problem~\eqref{equ_problem}, $r$ is the rank of partition constraint, i.e., $r=\sum_{k=1}^{K}B_{k}$ and $n=|\V|$. Similarly, if the set function $f:2^{\V}\rightarrow\R_{+}$ is  monotone $(\gamma,\beta)$-weakly submodular and $L=\mathcal{O}(T)$, the subset $S$ output by Algorithm~\ref{alg:scg_mutli} satisfies:
	\begin{equation*}
		\E\big(f(S)\big)\ge\Big(\frac{\gamma^{2}(1-e^{-(\beta(1-\gamma)+\gamma^2)})}{\beta(1-\gamma)+\gamma^2}\Big)f(S^{*})-\mathcal{O}\Big(\frac{r\sqrt{n}}{T}\Big).
	\end{equation*} 
\end{theorem}
\begin{remark}~\label{remark:scg_converge}
	Theorem~\ref{thm5} implies that, when the set function $f$  is $\alpha$-weakly monotone DR-submodular, if we set $T=\mathcal{O}(1/\epsilon)$, the subset yielded by our proposed Algorithm~\ref{thm2} can secure a value of $(1-e^{-\alpha})\text{OPT}-\epsilon$, where OPT is the  maximum value of problem~\eqref{equ_problem}. Moreover, when $f$ is $(\gamma,\beta)$-weakly monotone  submodular, Algorithm~\ref{thm2} also can achieve $\big(\frac{\gamma^{2}(1-e^{-(\beta(1-\gamma)+\gamma^2)})}{\beta(1-\gamma)+\gamma^2}\big)\text{OPT}-\epsilon$ after $\mathcal{O}(1/\epsilon)$ iterations. 
\end{remark}
\begin{remark}
It is important to note that the optimization objective $\left\langle \nabla F(\x(t)), \sum_{k=1}^{K} \frac{\one_{(S\cap \V_{k})}}{B_{k}}\right\rangle$ in line 13 is a modular set function with respect to $S$. As a result, we can employ a greedy method to efficiently find the optimal subset $S(t)\triangleq\mathop{\arg\max}_{S\subseteq\V, |S\cap\V_{k}|\le B_{k},\forall k\in[K]}\left\langle \nabla F(\x(t)), \sum_{k=1}^{K} \frac{\one_{(S\cap \V_{k})}}{B_{k}}\right\rangle$.
\end{remark}
\begin{remark}~\label{remark:scg_limited_queries}
From Line 13, we know that each $S(t)$ has at most $r$ non-zero entries, where $r=\sum_{k=1}^{K}B_{k}$. Moreover, from lines 10 and 14, we also can know $\boldsymbol{\xi}_{t}\triangleq\widehat{\nabla}^{2}_{t}\big(\x(t)-\x(t-1)\big)\triangleq\sum_{k=1}^{K}\frac{1}{TB_{k}}\left(\widehat{\nabla}^{2}_{t}\one_{\left(S(t-1)\cap\V_{k}\right)}\right)$. Hence, the computation of $\boldsymbol{\xi}_{t}$ only utilizes $\mathcal{O}(nr)$ entries in the Hessian estimation $\widehat{\nabla}_{t}^{2}$. This implies that it is sufficient to estimate up to $\mathcal{O}(nr)$ second-order partial derivatives for each point $\x_{l}(t)$ at Line 8.
\end{remark}
From \cref{remark:scg_converge} and \cref{remark:scg_limited_queries},  we know that when $f$ is monotone $\alpha$-weakly DR-submodular or $(\gamma,\beta)$-weakly submodular, if we set the batch size $L=\mathcal{O}(T)$ and $T=\mathcal{O}(\frac{r\sqrt{n}}{\epsilon})$, our Algorithm~\ref{thm2} can attain a $(1-e^{-\alpha})\text{OPT}-\epsilon$ or $\big(\frac{\gamma^{2}(1-e^{-(\beta(1-\gamma)+\gamma^2)})}{\beta(1-\gamma)+\gamma^2}\big)\text{OPT}-\epsilon$ approximation of the offline subset selection problem~\eqref{equ_problem} with $\mathcal{O}(\frac{r^{3}n^{2}}{\epsilon^{2}})$  function evaluations to the objective function $f$, respectively. Compared to the local-search methods~\citep{thiery2022two,JOGO-Lu}, Algorithm~\ref{alg:scg_mutli} not only eliminates the need for prior knowledge of parameters $\alpha$, $\gamma$, and $\beta$, but also can reduce the dependence on the error parameter $\epsilon$ from the previous $\tilde{\mathcal{O}}(1/\epsilon^{6})$ or $\tilde{\mathcal{O}}(1/\epsilon^{3})$ to $\mathcal{O}(1/\epsilon^{2})$.
\subsection{Online Variant of Continuous Greedy Algorithm}\label{sec_online_scg}
In this subsection, we aim to extend our previous \cref{alg:scg_mutli} into online settings. Compared to offline scenarios, the primary challenge in online settings is that, at each moment $t\in[T]$, the utility set function $f_{t}$ remains completely unknown until we commit to a subset $S_{t}$. If we assume every incoming utility objective function $f_{t}$ were known in advance at each timespot $t\in[T]$,  then we could have utilized \cref{alg:scg_mutli}, e.g., running it for $Q$-round iterations, to maximize the multinoulli extension $F_{t}$ of $f_{t}$. Specifically, at each $q$-th iteration where $q\in[Q]$, we would have found an optimal subset $S_{t}(q)\triangleq\mathop{\arg\max}_{S\subseteq\V, |S\cap\V_{k}|\le B_{k},\forall k\in[K]}\left\langle \nabla F_{t}(\x_{t}(q)), \sum_{k=1}^{K} \frac{\one_{(S\cap \V_{k})}}{B_{k}}\right\rangle$ to update the current variable $\x_{t}(q)$ according to the rule $\x_{t}(q+1)\triangleq\x_{t}(q) + \frac{1}{Q} \sum_{k=1}^{K} \frac{1}{B_{k}} \one_{(S_{t}(q) \cap \V_{k})}$. However due to the unknown $f_{t}$, it is impossible to find such subset $S_{t}(q)$.

Previously, when extending the \emph{continuous greedy}(CG) algorithm~\citep{calinescu2011maximizing} for multi-linear extension into online settings, \citep{streeter2009online,chen2018online,harvey2020improved} also encountered a similar issue, namely, we must make decisions at each moment without prior knowledge of future \emph{submodular} set functions. To address this, \citep{streeter2009online,chen2018online,harvey2020improved} proposed a novel technique named meta-action. The core idea of meta-action is to mimic the linear-optimization step of CG algorithm, i.e., $\max_{\v\in\mathcal{M}}\left\langle\v,\nabla G_{t}\big(\x_{t}(q)\big)\right\rangle$, by a sequence of online linear maximization algorithms $\{\mathcal{E}^{(q)}: 1\le q\le Q\}$, where $G_{t}$ is the multi-linear extension of the corresponding set function $f_{t}$ and $\mathcal{M}$ is a convex domain. More specifically, at each time $t\in[T]$, \citep{streeter2009online,chen2018online,harvey2020improved} choose to utilize the output $\v_{t}(q)\in\mathcal{M}$ of $q$-th linear oracle $\mathcal{E}^{(q)}$ as a surrogate of the optimal vector of  problem $\max_{\v\in\mathcal{M}}\left\langle\v,\nabla G_{t}\big(\x_{t}(q)\big)\right\rangle$ and then perform the iteration $\x_{t}(q+1)\triangleq\x_{t}(q)+\frac{1}{Q}\v_{t}(q)$ for any $q\in[Q]$. After committing the final $\x_{t}(Q+1)$, \citep{streeter2009online,chen2018online,harvey2020improved} learns the $f_{t}$ and feeds back the product $\left\langle\v_{t}(q),\nabla G\big(\x_{t}(q)\big)\right\rangle$ as the payoff to each linear oracle $\mathcal{E}^{(q)},\forall q\in[Q]$.	

Motivated by the idea of meta-action, we also employ a sequence of online linear maximization algorithms $\{\mathcal{E}^{(q)}: 1\le q\le Q\}$ to extend the \emph{offline} \cref{alg:scg_mutli} into online settings. It is worth noting that when the set $S$ is within the partition constraint of problem~\eqref{equ_problem}, the update direction $\sum_{k=1}^{K} \frac{1}{B_{k}} \one_{(S \cap \V_{k})}$ in Line 14 of \cref{alg:scg_mutli} lies within the domain $\prod_{k=1}^{K}\Delta_{n_{k}}$. As a result, we designate each oracle $\mathcal{E}^{(q)},\forall q\in[Q]$ as an online linear maximization algorithm over $\prod_{k=1}^{K}\Delta_{n_{k}}$ and then, at each time $t\in[T]$, we use the output $\v_{t}(q)$ of the $q$-th oracle $\mathcal{E}^{q}$ to replace the update direction $\sum_{k=1}^{K} \frac{1}{B_{k}}\one_{(S_{t}(q) \cap \V_{k})}$, as detailed in \cref{alg:online_scg_mutli}.

\begin{algorithm}[t]
	\caption{Online Stochastic Continuous Greedy Algorithm for Multinoulli Extension(\textbf{Multinoulli-OSCG})}\label{alg:online_scg_mutli}
	{\bf Input:} Number of oracles $Q$, batch size $L$, time horizon $T$ and $\{\V_{1},\dots,\V_{K}\}$ of set $\V$ 
	\begin{algorithmic}[1]
        \STATE {\bf Initialize:} $Q$ online  oracles $\{\mathcal{E}^{(1)},\dots,\mathcal{E}^{(Q)}\}$ for maximizing linear functions over the domain $\prod_{k=1}^{K}\Delta_{n_{k}}$; 
        \STATE {\bf Initialize:} $\x_{t}(1)\triangleq\textbf{0}_{n}$ for any $t\in[T]$;
		\FOR{$t=1,\dots,T$}
        \FOR[Continuous Greedy(Lines 4-7)]{$q=1,\dots,Q$}
        \STATE Obtain the output $\v_{t}(q)$ from the oracle $\mathcal{E}^{(q)}$;
        \STATE Update $\x_{t}(q+1)\triangleq \x_{t}(q)+\frac{1}{Q}\v_{t}(q)$;
        \ENDFOR\COMMENT{Rounding and Commitment(Lines 8-9)}
        \STATE Obtain $S_{t}$ by rounding $\x_{t}(Q+1)$ without replacement;
        \STATE Commit to $S_{t}$, receive reward $f_{t}(S_{t})$ and observe $f_{t}$;
        \FOR{$q=1,\dots,Q$}
		\IF[Differential Estimator(Lines 11-20)]{$q=1$}
		\STATE Compute $\boldsymbol{\mathrm{g}}_{t}(1)\triangleq\nabla F_{t}(\textbf{0}_{n})$ based on \cref{thm1};	
		\ELSE
		\STATE Sample $\{a_{t}(1),\dots,a_{t}(L)\}$ uniformly from $[0,1]$;
		\STATE Compute $\x_{t}^{(l)}(q)\triangleq a_{t}(l)\x_{t}(q)+(1-a_{t}(l))\x_{t}(q-1)$ for any $l\in\{1,2,\dots,L\}$;
		\STATE Compute Hessian estimator $\widehat{\nabla}^{2} F_{t}\big(\x_{t}^{(l)}(q)\big)$ based on Remark~\ref{remark:estimate_2_order} for any $l\in\{1,2,\dots,L\}$;
		\STATE Compute $\widehat{\nabla}^{2}_{t}(q)\triangleq\frac{1}{L}\sum_{l=1}^{L} \widehat{\nabla}^{2}F_{t}\big(\x_{t}^{(l)}(q)\big)$;
		\STATE Compute  $\boldsymbol{\xi}_{t}(q)\triangleq\widehat{\nabla}^{2}_{t}(q)\big(\x_{t}(q)-\x_{t}(q-1)\big)$;
		\STATE Aggregate the estimator $\boldsymbol{\mathrm{g}}_{t}(q)\triangleq\boldsymbol{\mathrm{g}}_{t}(q-1)+\boldsymbol{\xi}_{t}(q)$;
		\ENDIF\COMMENT{Update Online Linear Oracle(Line 21)}
        \STATE Feed back $\boldsymbol{\mathrm{g}}_{t}(q)$ to the linear oracle $\mathcal{E}^{(q)}$;
        \ENDFOR\ENDFOR
	\end{algorithmic}
\end{algorithm}

Similar to \cref{alg:scg_mutli}, in Lines 11-20 of \cref{alg:online_scg_mutli}, we utilize a path-integrated differential estimator $\boldsymbol{\mathrm{g}}_{t}(q)$ to approximate each gradient $\nabla F_{t}\big(\x_{t}(q)\big)$ for any $q\in[Q]$ and $t\in[T]$, where $F_{t}$ is the \texttt{ME} of the $t$-th objective function $f_{t}$. After that, like the meta-action, the Line 21 of \cref{alg:online_scg_mutli} feeds back all obtained gradient estimations $\boldsymbol{\mathrm{g}}_{t}(q),\forall q\in[Q]$ to their corresponding linear oracles $\mathcal{E}^{(q)},\forall q\in[Q]$.

Next, we provide the theoretical analysis for \cref{alg:online_scg_mutli}. Before that, we introduce some standard assumptions about the linear maximization oracles $\{\mathcal{E}^{(q)}: q\in[Q]\}$, namely, 
\begin{assumption}\label{ass:1}
	Each linear maximization oracle $\mathcal{E}^{(q)},\forall q\in[Q]$ can achieve a \emph{$1$-Regret} of $\mathcal{O}(\sqrt{nr}\sqrt{T})$ where $r$ is the rank of partition constraint of problem~\eqref{equ_problem}, that is, $r\triangleq\sum_{k=1}^{K}B_{k}$, $n\triangleq|\V|$ and $T$ is the time horizon.
\end{assumption}
\begin{remark}
	It is worth noting that there exist several off-the-shelf algorithms that can achieve a regret bound of $\mathcal{O}(\sqrt{nr}\sqrt{T})$  for online linear maximization problem, for instance, Online Gradient Ascent~\citep{zinkevich2003online} and Regularized-Follow-The-Leader~\citep{hazan2016introduction}.
\end{remark}
With this \cref{ass:1}, we then can get the following convergence results for our \cref{alg:online_scg_mutli}, that is,
\begin{theorem}[Proof provided in Appendix~\ref{appendix:online_scg}]\label{thm:online} If \cref{ass:1} holds, when each objective set  function $f_{t}$ is monotone $\alpha$-weakly DR-submodular or  $(\gamma,\beta)$-weakly submodular, the $\rho$-Regret of our proposed \cref{alg:online_scg_mutli} satisfies that
\begin{equation*}
\E\left(\text{Reg}_{\rho}\left(T\right)\right)\le\mathcal{O}\left(\sqrt{nr}\sqrt{T}+r\sqrt{n}\big(\frac{T}{L}+\frac{T}{Q}\big)\right),
\end{equation*} where $\rho=(1-e^{-\alpha})$ or $\rho=\frac{\gamma^{2}(1-e^{-(\beta(1-\gamma)+\gamma^2)})}{\beta(1-\gamma)+\gamma^2}$, respectively. Note that  $L$ represents the batch size for differential estimators and $Q$ denotes the number of linear maximization oracles.
\end{theorem}
\begin{remark}\label{remark:regret_oscg}
	Theorem~\ref{thm:online} suggests that, when every incoming set function $f_{t}$ is monotone $\alpha$-weakly DR-submodular or $(\gamma,\beta)$-weakly submodular, if we set the batch size $L=\mathcal{O}(\sqrt{rT})$ and employ $Q=\mathcal{O}(\sqrt{rT})$ linear maximization oracles, \cref{alg:online_scg_mutli} can achieve a regret of $\mathcal{O}(\sqrt{nr}\sqrt{T})$ against a $(1-e^{-\alpha})$-approximation or $\big(\frac{\gamma^{2}(1-e^{-(\beta(1-\gamma)+\gamma^2)})}{\beta(1-\gamma)+\gamma^2}\big)$-approximation to the best feasible solution in hindsight, respectively. 
\end{remark}
\begin{remark}\label{remark:improved_regret_oscg}
Note that recent works~\citep{harvey2020improved,fang2022online} proposed a more efficient \underline{O}nline \underline{D}ual \underline{A}veraging(ODA) algorithm for the online linear maximization problem over the domain $\prod_{k=1}^{K}\Delta_{n_{k}}$ and demonstrated that it can achieve a $1$-regret of $\mathcal{O}\big(\sqrt{rT\ln(n/r)}\big)$. Thus, in \cref{alg:online_scg_mutli}, if all linear oracles $\mathcal{E}^{(q)},q\in[Q]$ adopt this ODA algorithm, then when each set objective function $f_{t}$ is monotone $\alpha$-weakly DR-submodular or  $(\gamma,\beta)$-weakly submodular, the $\rho$-Regret of \cref{alg:online_scg_mutli} satisfies:
\begin{equation*}\E\left(\text{Reg}_{\rho}\left(T\right)\right)\le\mathcal{O}\left(\sqrt{r\ln(n/r)}\sqrt{T}+r\sqrt{n}\big(\frac{T}{L}+\frac{T}{Q}\big)\right).
\end{equation*}In particular, when $L=Q=\mathcal{O}\big(\sqrt{\frac{rn}{\ln(n/r)}}\sqrt{T}\big)$, \cref{alg:online_scg_mutli}  can achieve an improved $\rho$-regret bound of $\mathcal{O}\big(\sqrt{rT\ln(n/r)}\big)$.
\end{remark}
\begin{remark}
To the best of our knowledge, this is the \emph{first} result that not only attains a $(1-e^{-\alpha})$-regret of $\mathcal{O}(\sqrt{T})$ for online $\alpha$-weakly DR-submodular maximization problems over a partition constraint, but also can achieve a $\big(\frac{\gamma^{2}(1-e^{-(\beta(1-\gamma)+\gamma^2)})}{\beta(1-\gamma)+\gamma^2}\big)$-approximation to the online partition-constrained $(\gamma,\beta)$-weakly submodular maximization problems.
\end{remark}

%% file: Gradient_Ascent.tex
 Although the idea of meta-action enables us to extend the \emph{offline} \cref{alg:scg_mutli} to more complex online settings, it is important to note that, in \cref{alg:online_scg_mutli}, we need to maintain $Q$ distinct online linear maximization algorithms $\{\mathcal{E}^{(1)},\dots,\mathcal{E}^{(Q)}\}$, which can result in substantial  memory overhead. More critically, at the end of each round $t\in[T]$, \cref{alg:online_scg_mutli} requires computing $Q$ differential estimators $\{\boldsymbol{\mathrm{g}}_{t}(1),\dots,\boldsymbol{\mathrm{g}}_{t}(Q)\}$ to update the introduced oracles $\{\mathcal{E}^{(1)},\dots,\mathcal{E}^{(Q)}\}$, which will incur a total of $LQ$ Hessian approximations for each \texttt{ME} $F_{t}$, thereby leading to a significant number of function evaluations for every set objective $f_{t}$. To address these issues, an alternative algorithmic scheme is to employ the Online Gradient Ascent(OGA)~\citep{zhang2022boosting,zinkevich2003online}, whose framework is not only simple to implement but more query-efficient. Thus, in the subsequent part of this paper, we focus on designing an effective and efficient OGA algorithm for the partition-constrained subset selection problem.

 Note that, when every incoming set function $f_{t}$ is a fixed set objective $f$, the classic OGA algorithm will reduce to \emph{offline} gradient ascent algorithm.  Furthermore, it is well-known that \emph{offline} gradient ascent with a sufficiently small step size will converge to a stationary point of its target objective function. As a result, from the results of \cref{thm3},  we can know that directly applying OGA to the \texttt{ME} $F_{t}$ of every incoming objective set function $f_{t}$  only can ensure a  $(\frac{\alpha^{2}}{1+\alpha^{2}})$-approximation or  $(\frac{\gamma^{2}}{\beta+\beta(1-\gamma)+\gamma^{2}})$-approximation to the best feasible solution in hindsight, which is weaker than the previously established approximation ratios of $(1-e^{-\alpha})$ or $\big(\frac{\gamma^{2}(1-e^{-(\beta(1-\gamma)+\gamma^2)})}{\beta(1-\gamma)+\gamma^2}\big)$ provided by \cref{thm:online}. Fortunately, as demonstrated in \cref{thm4+}, the stationary points of auxiliary function $F^{a}$ can attain the same approximation guarantees as those proven in \cref{thm:online}. Here, $\nabla F^{a}(\x)\triangleq\int_{0}^{1} w(z)\nabla F(z*\x)\mathrm{d}z$ where $F$ is a \texttt{ME} of some set function $f$ and $w(z)$ is the positive weight function over $[0,1]$. Motivated by these findings, we then choose to apply the OGA  algorithm into the auxiliary functions $F_{t}^{a}$ of the \texttt{ME} $F_{t}$ of every set objective function $f_{t}$, as shown in \cref{alg:OGA_framework}.

 \begin{algorithm}[t]
	\caption{Online Stochastic Gradient Ascent Algorithm for Multinoulli Extension(\textbf{Multinoulli-OSGA})}\label{alg:OGA_framework}
 {\bf Input:} Step size $\eta$, Partition $\{\V_{1},\dots,\V_{K}\}$, Ratios $\alpha,\beta,\gamma$, Budgets $\{B_{1},\dots,B_{K}\}$
	\begin{algorithmic}[1]
    	\STATE \textbf{Initialize:} $\x_{1}\triangleq\sum_{k=1}^{K} \frac{1}{n_{k}} \one_{\V_{k}}$ and $\text{Auxiliary}\triangleq\textbf{True}$;
    	\FOR[Rounding and Commitment(Lines 2-3)]{$t\in[T]$}
        \STATE Obtain $S_{t}$ by rounding $\x_{t}$ without replacement;
        \STATE Commit to $S_{t}$, receive reward $f_{t}(S_{t})$ and observe $f_{t}$;
        \IF[Gradient Estimation(Lines 6-7)]{\emph{not} $\text{Auxiliary}$}
        \STATE Compute the gradient estimator $\widehat{\nabla} F_{t}\big(\x_{t}\big)$ based on the \cref{remark:gradient_estimation};
        \STATE Set $\boldsymbol{\mathrm{g}}_{t}\triangleq\widehat{\nabla} F_{t}\big(\x_{t}\big)$;
        \ELSE[Auxiliary Gradient Estimation(Lines 9-11)]
        \STATE Generate a random number $z_{t}$ from r.v. $\mathcal{Z}$ where $\text{Pr}(\mathcal{Z}\le z)\triangleq\frac{1}{\int_{0}^{1}w(a)\mathrm{d}a}\int_{0}^{z}w(a)\mathrm{d}a,\forall z\in[0,1]$;
        \STATE Compute the gradient estimator $\widehat{\nabla} F_{t}\big(z_{t}*\x_{t}\big)$ based on the \cref{remark:gradient_estimation};
         \STATE Set $\boldsymbol{\mathrm{g}}_{t}\triangleq(\int_{a=0}^{1}w(a)\mathrm{d}a)\widehat{\nabla} F_{t}\big(z_{t}*\x_{t}\big)$;
        \ENDIF\COMMENT{Stochastic Gradient Ascent(Lines 13-14)}
        \STATE Compute $\y_{t+1}=\x_{t}+\eta\boldsymbol{\mathrm{g}}_{t}$;
		\STATE Update $\x_{t+1}\triangleq\mathop{\arg\min}_{\mathbf{x}\in\prod_{k=1}^{K}\Delta_{n_{k}}}\|\mathbf{x}-\y_{t+1}\|_{2}$
	    \ENDFOR
	    \STATE \textbf{Return} $S_l$ and $l\triangleq\mathop{\arg\max}_{t\in[T]}f(S_{t})$;\COMMENT{Offline Setting}
	   \end{algorithmic} 
\end{algorithm}

Notably, in Lines 9-11 of \cref{alg:OGA_framework}, we employ a two-step sampling method to estimate each gradient $\nabla F^{a}_{t}(\x_{t})$. Specifically,  we initially samples a random number $z_t$ from the random variable $\mathcal{Z}$ with distribution $\text{Pr}(\mathcal{Z}\le z)\triangleq\frac{\int_{0}^{z}w(a)\mathrm{d}a}{\int_{0}^{1}w(a)\mathrm{d}a},\forall z\in[0,1]$ and then approximates each $\nabla F^{a}_{t}(\x_{t})$ by $(\int_{a=0}^{1}w(a)\mathrm{d}a)\widehat{\nabla} F_{t}\big(z_{t}*\x_{t}\big)$ where $\widehat{\nabla} F_{t}\big(z_{t}*\x_{t}\big)$ is an unbiased estimator of $\nabla F_{t}\big(z_{t}*\x_{t}\big)$. Additionally, \cref{alg:OGA_framework} provides a flexible option to directly use OGA algorithm over the \texttt{ME} $F_{t}$ instead of its auxiliary function $F_{t}^{a}$ by the ``Auxiliary" variable. Next, we provide the theoretical analysis for \cref{alg:OGA_framework}.

\begin{theorem}[Proof provided in Appendix~\ref{sec:multinoulli-sga}]\label{thm:online1} 
When every objective set  function $f_{t}$ is monotone $\alpha$-weakly DR-submodular or  $(\gamma,\beta)$-weakly submodular, if we utilize the standard gradient estimation(Lines 6-7), namely $\text{Auxiliary}\triangleq\textbf{False}$, the $\rho$-Regret of our proposed \cref{alg:OGA_framework} satisfies that
\begin{equation*}
\E\left(\text{Reg}_{\rho}\left(T\right)\right)\le\mathcal{O}\big(r/\eta+nT\eta\big),
\end{equation*}where $\rho=\frac{\alpha^{2}}{1+\alpha^{2}}$ or $\rho=\frac{\gamma^{2}}{\beta+\beta(1-\gamma)+\gamma^{2}}$, respectively. Note that $r$ is the rank of partition constraint, i.e., $r\triangleq\sum_{k=1}^{K}B_{k}$, $n=|\V|$ and $\eta$ is the step size. Similarly, if we use the auxiliary gradient estimation(Lines 9-11), i.e., $\text{Auxiliary}\triangleq\textbf{True}$, and set the weight function $w(z)$ according to \cref{thm4+}, the $\rho$-Regret of our proposed \cref{alg:OGA_framework} satisfies that
\begin{equation*}
\E\left(\text{Reg}_{\rho}\left(T\right)\right)\le\mathcal{O}\big(r/\eta+nT\eta\big),
\end{equation*} where $\rho=(1-e^{-\alpha})$ or $\rho=\frac{\gamma^{2}(1-e^{-(\beta(1-\gamma)+\gamma^2)})}{\beta(1-\gamma)+\gamma^2}$, respectively.
\end{theorem}
\begin{remark}
	Theorem~\ref{thm:online1} suggests that, when every incoming set function $f_{t}$ is monotone $\alpha$-weakly DR-submodular or $(\gamma,\beta)$-weakly submodular, if we adopt the auxiliary gradient estimation and set the step size $\eta\triangleq\mathcal{O}(\sqrt{\frac{r}{nT}})$, \cref{alg:OGA_framework} can achieve a regret of $\mathcal{O}\big(\sqrt{nr}T^{1/2}\big)$ against a $(1-e^{-\alpha})$-approximation or $\big(\frac{\gamma^{2}(1-e^{-(\beta(1-\gamma)+\gamma^2)})}{\beta(1-\gamma)+\gamma^2}\big)$-approximation to the best feasible solution in hindsight, respectively. 
\end{remark}
\begin{remark}\label{remark:online-to-offline}
   A byproduct of \cref{thm:online1} is that, in \emph{offline} settings, if we return the final subset as the Line 16 of \cref{alg:OGA_framework} and set the step size $\eta=\mathcal{O}(\frac{\epsilon}{n})$, then when the set objective function $f$ is monotone $\alpha$-weakly DR-submodular or $(\gamma,\beta)$-weakly submodular, the subset $S_{l}$ yielded by our proposed \cref{alg:OGA_framework} also can secure a value of $\rho\cdot\text{OPT}-\epsilon$ after $\mathcal{O}(\frac{nr}{\epsilon^{2}})$ iterations, where OPT is the  maximum value of problem~\eqref{equ_problem} and $\rho$ denotes the obtained approximation ratio. For further details, please refer to Appendix~\ref{sec:offline-to-online}.
\end{remark}

%% file: Experiments.tex
\bgroup
\begin{table*}[t]
	\caption{Results on video summarization. Note that `obj' denotes the utility function value, where a higher value is preferable, and `queries' represents the magnitude of the total number of function evaluations, that is , the $\log_{10}$ of the total number of value queries to the set objective function, with a smaller value being more favorable. `Distorted-LS-G' is the abbreviation for the distorted local-search method with $O(1/\epsilon)$-round guesses, namely, Algorithm B.1 in~\citep{thiery2022two}. Both \textbf{V1} and \textbf{V2} are sourced from websites, which are related with Cooking and Animation, respectively. \textbf{V3}–\textbf{V8} are derived from the VSUMM dataset and encompass topics such as Soccer, Live news, Broadcast, Concert as well as TV Show. The time following the name of each video indicates the video duration. In each column of `obj', 
		\textcolor{Red!50}{$\blacksquare$} indicates ranking the 1st and \textcolor{Blue!50}{$\blacksquare$} stands for the 2nd.} \label{tab:results_video_summarization}\vspace{0.4em}
	\small
	\centering
	\scalebox{0.7}{
		\newcommand{\rkone}[1]{{\setlength{\fboxsep}{1.5pt}\colorbox{Red!50}{#1}}}
		\newcommand{\rktwo}[1]{{\setlength{\fboxsep}{1.5pt}\colorbox{Blue!50}{#1}}}
		
		\begin{tabular}{l||c|c||c|c||c|c||c|c||c|c||c|c||c|c||c|c}
			\toprule[1.45pt]
			\multirow{2}[3]{*}{\diagbox{Method}{Video}} & \multicolumn{2}{c||}{\textbf{V1}(3min30s)} & \multicolumn{2}{c||}{\textbf{V2}(7min45s)} & \multicolumn{2}{c||}{\textbf{V3}(2min37s)} & \multicolumn{2}{c||}{\textbf{V4}(2min17s)} & \multicolumn{2}{c||}{\textbf{V5}(2min58s)} & \multicolumn{2}{c||}{\textbf{V6}(2min47s)} & \multicolumn{2}{c||}{\textbf{V7}(2min52s)}& \multicolumn{2}{c}{\textbf{V8}(5min52s)} \bigstrut\\
			\cline{2-17}   &obj$\uparrow$& queries$\downarrow$& obj$\uparrow$& queries$\downarrow$& obj$\uparrow$& queries$\downarrow$& obj$\uparrow$& queries$\downarrow$& obj$\uparrow$& queries$\downarrow$& obj$\uparrow$& queries$\downarrow$&obj$\uparrow$& queries$\downarrow$&obj$\uparrow$& queries$\downarrow$\bigstrut\\
			\hline	
			\hline\bigstrut
			Standard Greedy& 46.83  & 3.85 & 79.91  & 3.95 &48.27 & 3.61 & 19.45  & 3.48 &53.33 & 3.73& 34.77 & 4.11& 45.94   & 3.73& 39.89   & 3.69  \\
			Residual-Greedy &49.75  &3.85  & 77.23   &3.95 & 46.96  & 3.61& 19.06  & 3.48 &52.41 & 3.73 & 35.23 & 4.11 & 45.81  & 3.73& 38.86  & 3.69   \\
			Distorted-LS-G & \rktwo{50.56}  & 10.40 & 79.95   &10.50& 50.20 & 10.17 & \rktwo{19.48}  & 10.02 & 54.87 & 10.26 & \rktwo{36.64}& 10.70& \rktwo{47.97}   & 10.27& 40.70   & 10.24  \bigstrut[b] \\
			\hline
			\hline\bigstrut
			\textbf{Multinoulli-SCG} & \rkone{50.63} & 8.59 & \rktwo{79.98}  & 8.68 &\rktwo{50.25}  & 8.34 & \rkone{19.53}  & 8.21 &\rktwo{55.09}& 8.46& \rkone{36.75}  & 8.84 & \rkone{48.06}& 8.46 & \rktwo{40.72}& 8.42 \bigstrut[b] \\
            \textbf{Multinoulli-SGA} & \rkone{50.63}  & 6.53 & \rkone{79.99}  & 6.58 & \rkone{50.27}  & 6.41 & \rkone{19.53}  &  6.35& \rkone{55.14}  & 6.46 &\rkone{36.75}  & 6.66 & \rkone{48.06}   & 6.46& 40.69   & 6.45 \\
             \textbf{Multinoulli-ASGA} & \rkone{50.63}  & 6.53 & 79.92  & 6.58 & \rkone{50.27}  & 6.41 & \rkone{19.53}  &  6.35& 55.06  & 6.46 &\rkone{36.75}  & 6.66 & \rkone{48.06}   & 6.46& \rkone{40.77}   & 6.45 \\
			\midrule[1.45pt]
	\end{tabular}}\vspace{-1.5em}
\end{table*}
\egroup
\subsection{Offline Subset Selection}\label{sec:offline_subset_selection}
In this subsection, we empirically compare the performance of our proposed \textbf{Multinoulli-SCG}, \textbf{Multinoulli-SGA} and \textbf{Multinoulli-ASGA} against the standard greedy method~\citep{gatmiry2018non} and the residual random greedy method~\citep{chen2018weakly} as well as the distorted local search~\citep{thiery2022two} across three distinct applications: video summarization, special maximum coverage and bayesian A-optimal design. Note that \textbf{Multinoulli-SGA} and \textbf{Multinoulli-ASGA} respectively represent  the \emph{offline} classical and auxiliary stochastic gradient ascent applied to our proposed \texttt{ME}(See \cref{alg:OGA_framework}). 
Due to the space
limits, the detailed experimental setups are provided in \cref{append:Experiments}.
\subsubsection{\textbf{Video Summarization}}\label{sec:videos}
The objective of video summarization is aimed at picking a few representative frames from a given video such that these frames can capture as much content as possible. To achieve this, a common strategy is to formulate the frame selection problem as the maximization of a \underline{D}eterminantal \underline{P}oint \underline{P}rocess(DPP) objective function~\citep{gong2014diverse,mirzasoleiman2018streaming,chen2018weakly}. DPP has recently emerged as a powerful tool that favors subsets of a ground set of items with higher diversity~\citep{kulesza2012determinantal}. More specifically, for an $n$-frame video, we represent each frame by a $p$-dimensional vector. Then, we compute the Gramian matrix $X$ of the $n$ resulting vectors by setting each $X_{ij}$ as the Gaussian kernel between the $i$-th and $j$-th vectors. With this matrix $X$, the DPP objective function can be defined as $f(S)=\text{det}(I+X_{S})$ where $S\subseteq[n]$, $X_{S}$ is the principal submatrix of $X$ indexed by $S$ and $I$ is a $|S|$-dimensional identity matrix. Note that \citep{bian2017guarantees} has proven that this set function $f$ is monotone and weakly submodular from below. Moreover, \citep{ijcai2022p666} also verified the weak DR-submodularity of $f$.

For our experiments, we use six videos from the VSUMM dataset~\citep{de2011vsumm} and two videos about `Animation' and `Cooking' from websites like YouTube. Additionally, we utilize the method described in~\citep{gong2014diverse} to prune each video, namely, for long videos($\ge5$min), we uniformly sampled one frame per second, and for short videos, we sampled one frame every half second. Subsequently, we choose to create a summary of each video by extracting one representative frame from every $25$ frames, that is, we consider the following partition constraint:
\begin{equation*}
	\big|S\cap[25(i-1)+1,25i]\big|\le1\ \ \ \ \ \forall\ 1\le i\le\lceil n/25\rceil.
\end{equation*}\cref{tab:results_video_summarization} illustrates the performance of our proposed \textbf{Multinoulli-SCG} and \textbf{Multinoulli-SGA} algorithms against three benchmark methods, namely, `Standard Greedy', `Residual-Greedy' and `Distorted-LS-G'. It is quite easy to observe that our \textbf{Multinoulli-SCG} and \textbf{Multinoulli-SGA} algorithms produce summaries with higher diversity than the other three baselines. Specifically, \textbf{Multinoulli-SCG} achieves Top-2 performance in all 8 videos, while \textbf{Multinoulli-ASGA} and \textbf{Multinoulli-SGA} attain Top-1 performance on  at least 6 out of the 8 videos. Furthermore, the number of function evaluations required by our \textbf{Multinoulli-SCG}, \textbf{Multinoulli-SGA} and \textbf{Multinoulli-ASGA} is 2, 4 and 4 orders of magnitude lower than that of the state-of-the-art `Distorted-LS-G', respectively. This result aligns well with our previous theoretical findings in \cref{sec:algorithm} and \cref{sec:algorithm1}.

\input{Appendix_Experiment}
\subsection{Online Subset Selection}\label{sec:online_subset_selection}
\begin{wrapfigure}{r}{0.23\textwidth}
\vspace{-2.5em}
\includegraphics[width=0.23\textwidth]{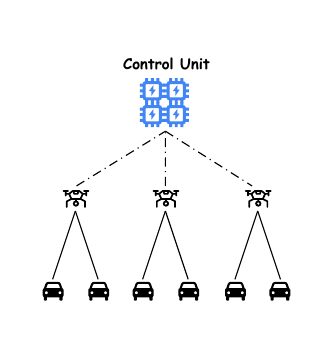}\vspace{-1.0em}
  		\captionsetup{font=scriptsize}
  		\caption{Target Tracking via UAVs}\label{figure_control_unit}
 		\vspace{-1.0em}
  	\end{wrapfigure}In this subsection, we test  the effectiveness of our proposed \textbf{Multinoulli-OSCG} and \textbf{Multinoulli-OSGA} via simulated multi-target tracking tasks, as depicted in \cref{figure_control_unit}. In this scenario, at each moment, a centralized control unit  needs to coordinate the actions of multiple mobile sensors, particularly unmanned aerial vehicles (UAVs), to effectively track all moving targets of interest. Due to the space
limits, the detailed experimental setups are provided in \cref{append:online_Experiments}.

\subsubsection{\textbf{Tracking with Facility-Location Objectives}}\label{sec:facility-location}Following experimental protocols from~\citep{zhang2025nearoptimal,xu2023online}, we deploy 20 UAVs in a planar environment to track 30 mobile targets over a 25-second horizon, partitioned into $T\triangleq1250$ discrete iterations. Per iteration, each UAV determines both a movement direction from the action set \{up, down, left, right, diagonal\} and a speed from the discrete levels \{5, 10, 15\} units/s. Target behaviors are classified into three categories: the stochastic `Random' type, the structured `Polyline' type, and the strategically evasive `Adversarial' type. `Random'  targets draw their heading angle $\theta$ uniformly from $[0,2\pi]$ and sample velocities independently from $[5,10]$ units/s at each step. `Polyline' targets maintain deterministic trajectories except at waypoint iterations $\{0, \lfloor\frac{T}{k}\rfloor, 2\lfloor\frac{T}{k}\rfloor,\dots, (k-1)\lfloor\frac{T}{k}\rfloor\}$ (with $k\in\{1,2,4\}$ drawn randomly), where they temporarily exhibit Random behavior. `Adversarial' targets behave identically to `Random' targets when all agents remain beyond 20 units distance; upon detecting any agent within this 20-unit threshold, they initiate a one-second evasion at 15 units/s toward the direction maximizing the average separation from all agents.

\begin{figure*}[t]
\subfigure[\small`R':`A':`P'=4:5:1\label{graph_show_1}]{\includegraphics[scale=0.14]{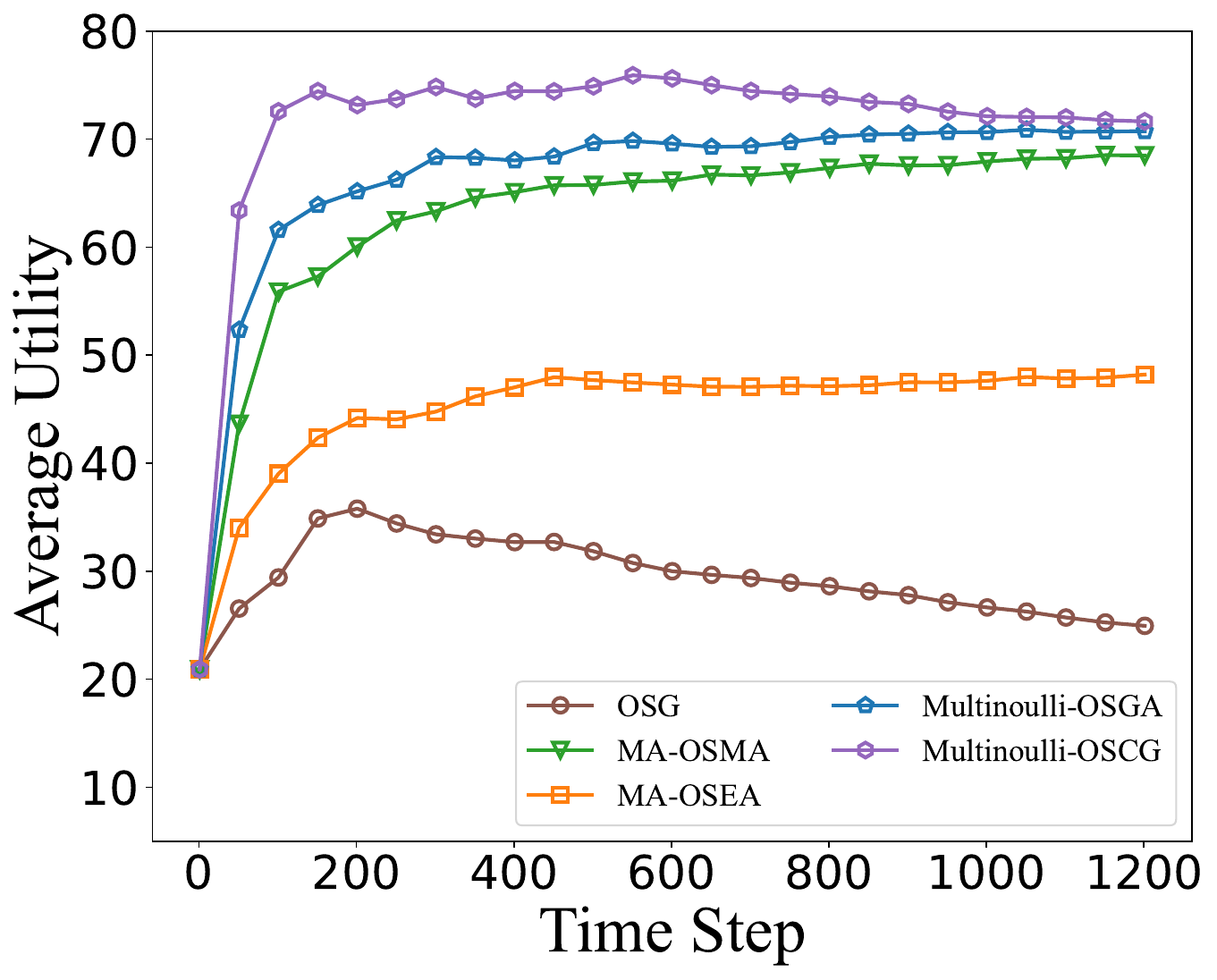}}
	\subfigure[\small`R':`A':`P'=6:3:1\label{graph_show_2}]{\includegraphics[scale=0.14]{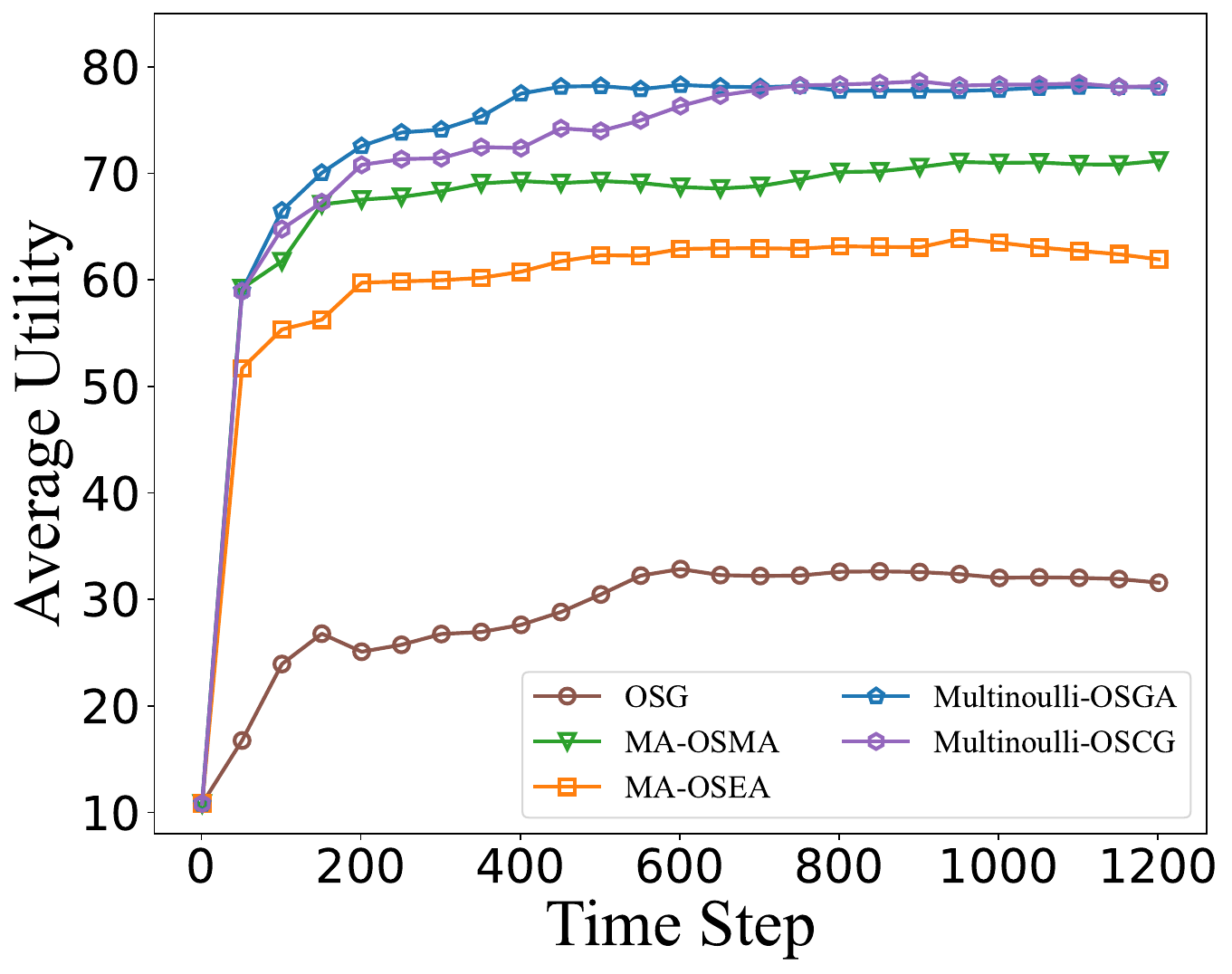}}
	\subfigure[\small`R':`A':`P'=8:1:1\label{graph_show_3}]{\includegraphics[scale=0.14]{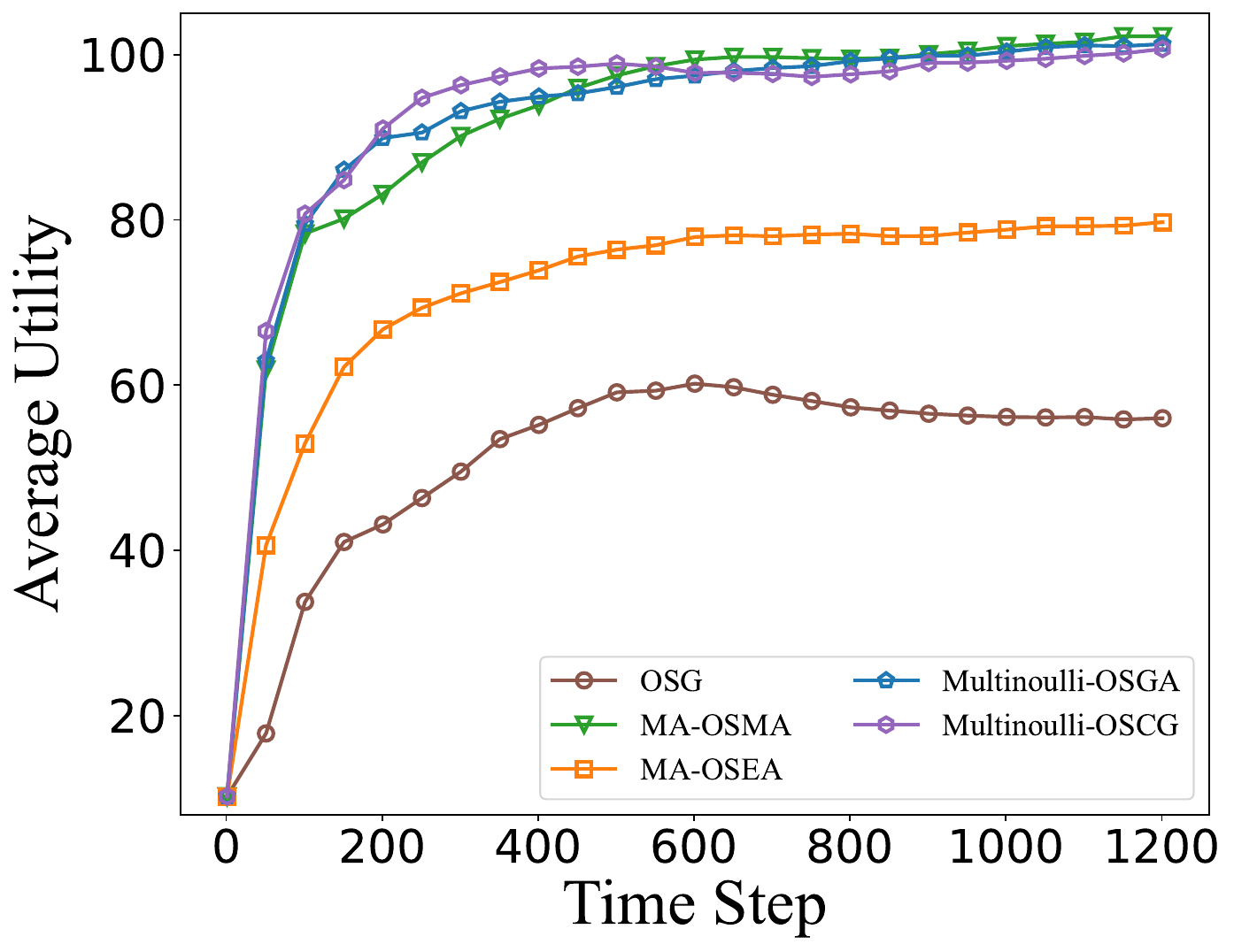}}
	\subfigure[\small Extended Kalman Filter\label{graph_show_4}]{\includegraphics[scale=0.14]{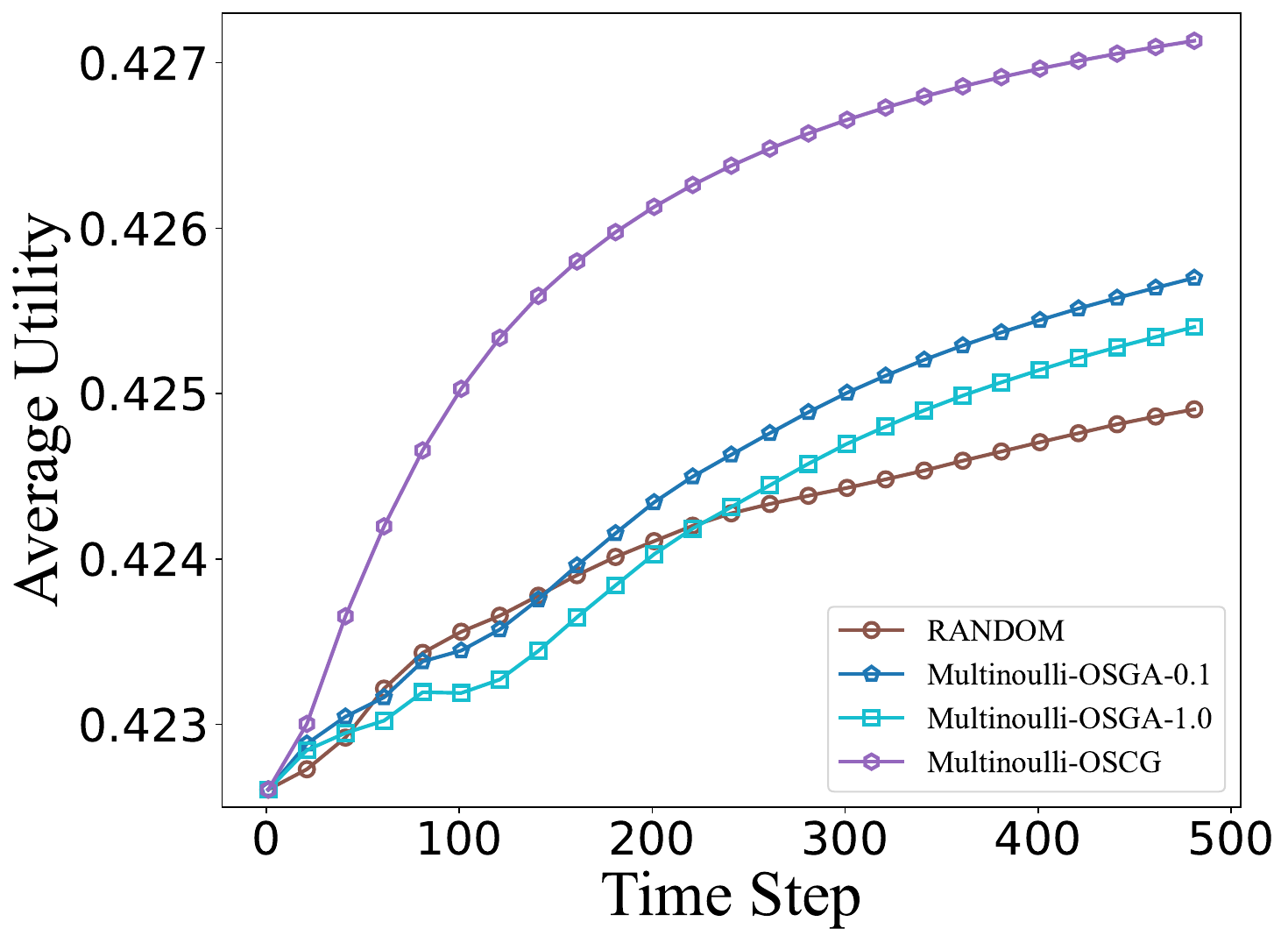}}
	\subfigure[\small $0.1$-Network Search\label{graph_show_5}]{\includegraphics[scale=0.14]{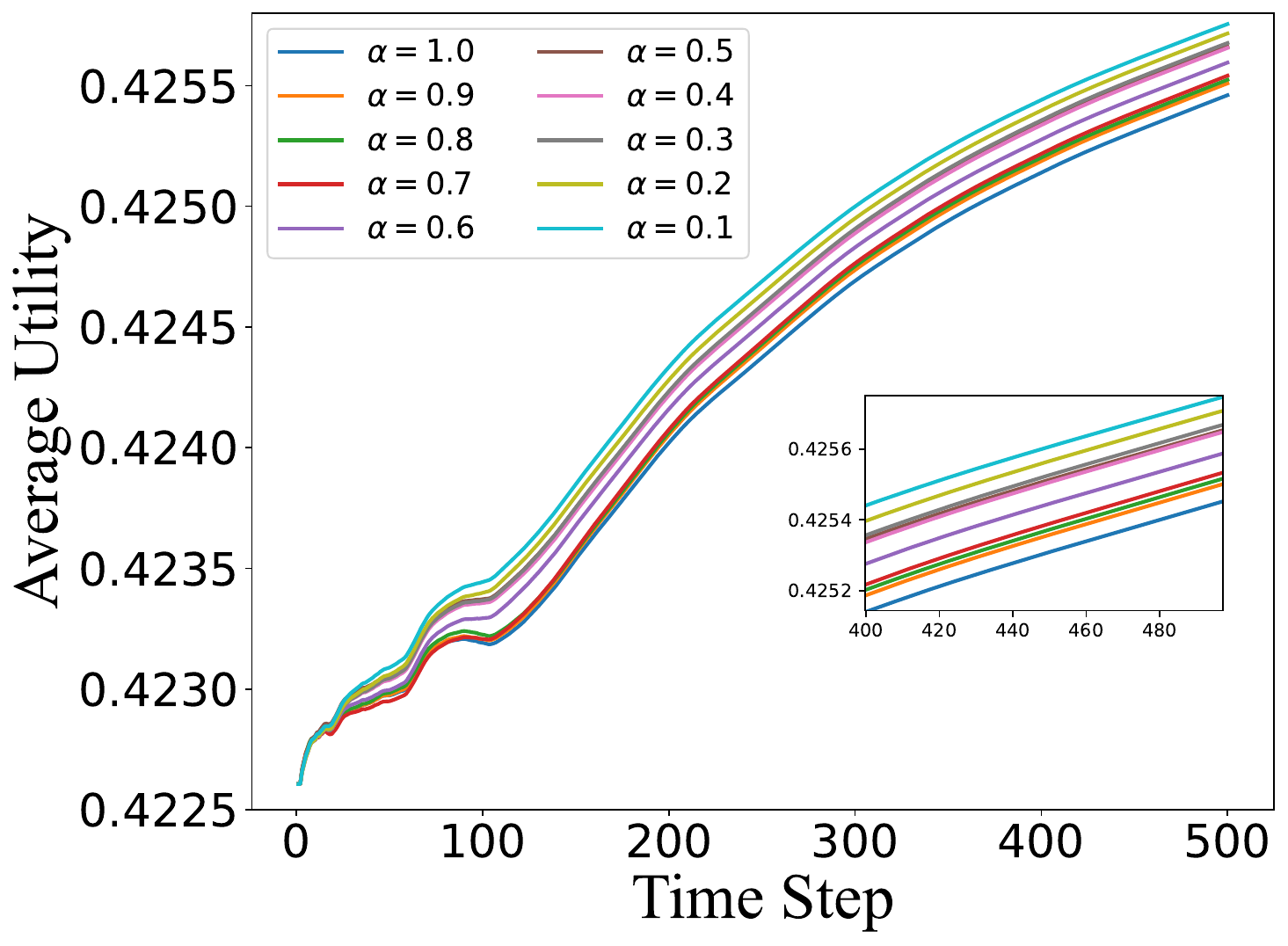}}
	\caption{A comparison of the running average utility on different multi-target tracking scenarios.}\label{graph_show}
\end{figure*}

Generally speaking, every motion of any UAV can be characterized by three key parameters, namely, its movement angle $\theta$, speed  $s$ and the unique identifier $i$. With all these three parameters, then the action set  $\V_{i}$  available to each UAV $i\in[20]$ can be mathematically represented as:
\begin{equation*}
\V_{i}\triangleq\left\{(\theta,s,i):s\in\{5,10,15\}\text{units/s},\theta\in\{\frac{\pi}{4},\frac{\pi}{2},\dots,2\pi\}\right\},
\end{equation*} where each tuple $(\theta,s,i)$ encodes a specific action of UAV $i$, that is, it will move at a speed of $s$ in the direction of  $\theta$. Furthermore, at each time step $t\in[T]$, we denote the location of each target $j\in[30]$ as $o_{j}(t)$. Similarly, we also utilize the symbol $o_{(\theta,s,i)}(t)$ to represent the new position of UAV $i$ after moving from its previous location at time $(t-1)$ with a movement angle $\theta$ and speed $s$.

To enhance the tracking quality, a common strategy is to minimize the distances between agents and targets. Inspired by this idea, many studies~\citep{xu2023online,zhang2025nearoptimal} naturally consider the following facility-location objective for UAVs at every time $t\in[T]$, i.e.,
\begin{equation*}
	f_{t}(S)\triangleq\sum_{j=1}^{30}\max_{(\theta,s,i)\in S}\frac{1}{\|o_{(\theta,s,i)}(t)-o_{j}(t)\|_{2}},
\end{equation*} where $\|o_{(\theta,s,i)}(t)-o_{j}(t)\|_{2}$ represents the Euclidean distance between the location  $o_{j}(t)$ of target $j$ and the new position $o_{(\theta,s,i)}(t)$ after UAV $i$ executing the action $(\theta,s,i)$ and $S$ is a subset of the ground action set $\V\triangleq\cup_{i=1}^{n}\V_{i}$. It is important to note that the larger the value of $\frac{1}{\|o_{(\theta,s,i)}(t)-o_{j}(t)\|_{2}}$ becomes, the closer the action  $(\theta,s,i)$ drives UAV $i$ to the target $j$. 

Considering this facility-location utility set function $f_{t}$ and the truth that each UAV only can execute  one decision from $\V_{i}$ at every time $t\in[T]$,  then we can easily transform the aforementioned multi-target tracking problem as a partition-constrained subset selection problem, namely,
\begin{equation}\label{appendix:simulation:problem1}
	\max f_{t}(S),\ \ \text{ s.t.}\ S\subseteq\V\ \text{and}\ |S\cap\V_{i}|\le1,\forall i\in[20].
\end{equation}  
Furthermore, numerous studies~\citep{xu2023online,zhang2025nearoptimal} have verified that this facility-location objective function $f_{t}$ is \emph{monotone submodular}. 

In simulations, we initialize the starting positions of all UAVs and targets randomly within 20-unit radius circle centered at the origin. Furthermore, we  consider different proportions of `\underline{R}andom', `\underline{P}olyline', and `\underline{A}dversarial' targets. Specifically, we set the proportions of targets as `\underline{R}andom':`\underline{A}dversarial':`\underline{P}olyline'=$4$:$5$:$1$ in \cref{graph_show_1} and $6$:$3$:$1$ in \cref{graph_show_2} as well as $8$:$1$:$1$ in \cref{graph_show_3}.  According to the results in \cref{graph_show}, we can find that the average utility $\frac{\sum_{\tau=1}^{t}f_{\tau}(S_{\tau})}{t},\forall t\in[T]$ of our proposed \textbf{Multinoulli-OSGA} and \textbf{Multinoulli-OSCG} algorithms can significantly exceed the state-of-the-art MA-OSMA and MA-OSEA algorithms in \citep{zhang2025nearoptimal} as well as OSG in \citep{xu2023online}.  It is worth noting that as the proportion of `Adversarial' targets increases, the maximum running average utility exhibits a downward trend for all curves.

\subsubsection{\textbf{Tracking under Extended Kalman-Filter Framework}}In previous section, we considered a simplified multi-target tracking model, that is, we assume that the monitoring quality of each moving target $j$ only depends on its nearest UAV. However, in many real-world scenarios,  relying \emph{solely} on the information collected by the nearest UAV for accurately tracking the state of each target is unrealistic. Instead, we generally need to aggregate the data from multiple UAVs to comprehensively reconstruct the targets' behaviors.

Recently, in order to obtain an accurate estimation of every target's location, \citep{hashemi2019submodular,hashemi2020randomized} employed an extend Kalman-filter(EKF) framework to process the observations of multiple distinct agents. Specifically, let us consider a target tracking task using a swarm of UAVs equipped with GPS and radar systems. In this scenario, at each time $t\in[T]$, each UAV $i$ can measure the range to each target $j$ throughout its radar system and also can obtain its own position  information from GPS. Furthermore, like the work~\citep{hashemi2019submodular}, we also assume that the range measurements of the radar systems follow a quadratic model, that is to say,
\begin{equation}\label{equality:EKF}
r_{(\theta,s,i)\rightarrow j}(t)\triangleq\frac{1}{2}\|o_{(\theta,s,i)}(t)-o_{j}(t)\|^{2}_{2}+\xi_{(\theta,s,i)\rightarrow j}(t),
\end{equation} where  the symbol $r_{(\theta,s,i)\rightarrow j}(t)$ denotes the range measurement of UAV $i$ to target $j$ after UAV $i$ executes the action $(\theta,s,i)$, $o_{j}(t)$ is the location of target $j$, $\xi_{(\theta,s,i)\rightarrow j}(t)$ is the additive independent noise and the symbol $o_{(\theta,s,i)}(t)$ represents the new position of UAV $i$ after moving from its previous location at time $t-1$ with a movement angle $\theta$ and speed $s$.

Note that when the new position $o_{(\theta,s,i)}(t)$ and the range measurement $r_{(\theta,s,i)\rightarrow j}(t)$ are known, we can view Eq.\eqref{equality:EKF}  as a random quadratic experiment of the unknown location $o_{j}(t)$. In other words, different action $(\theta,s,i)$ can lead to  distinct random quadratic observation of the unknown parameter $o_{j}(t)$. Inspired by this perspective,  we can formulated the action selection problem in multi-target tracking task as an experimental design problem, which aims at selecting a feasible subset from the whole collection of  experiments such that the measurements of these selected sub-experiments can accurately estimate the location of every target. Particularly, under the classical Van Trees’ inequality~\citep{van2004detection}, the work~\citep{hashemi2019submodular} established a lower bound for the covariance matrix associated with the EKF estimator of the location $o_{j}(t)$. Specifically, at time $t\in[T]$, if we consider the sub-experiment $S$,  then the lower-bound matrix $\textbf{B}^{j}_{S}(t)$ for the EKF estimate of the unknown location $o_{j}(t)$ can be expressed as (Please refer to Theorem 2 in \citep{hashemi2019submodular}):
\begin{equation*}
	\textbf{B}^{j}_{S}(t)\triangleq\left(\sum_{(\theta,s,i)\in S}\frac{\textbf{P}+\textbf{z}_{(\theta,s,i)\rightarrow j}\textbf{z}_{(\theta,s,i)\rightarrow j}^{T}}{\sigma^{2}_{(\theta,s,i)\rightarrow j}(t)}+\textbf{I}_{j}(t)\right)^{-1},
\end{equation*} where $\sigma^{2}_{(\theta,s,i)\rightarrow j}(t)\triangleq\text{Var}\left(\xi_{(\theta,s,i)\rightarrow j}(t)\right)$, $\textbf{z}_{(\theta,s,i)\rightarrow j}\triangleq o_{(\theta,s,i)}(t)-o_{j}(t-1)-\E\left(o_{j}(t)-o_{j}(t-1)\right)$, $\textbf{P}\triangleq\text{Cov}\left(o_{j}(t)-o_{j}(t-1)\right)$ and $\textbf{I}_{j}(t)$ is the  Fisher information matrix  associated with the normalized random gap $\big(o_{j}(t)-o_{j}(t-1)-\E\left(o_{j}(t)-o_{j}(t-1)\right)\big)$. 

With this lower-bound covariance matrix $\textbf{B}^{j}_{S}(t)$, we then can utilize various so-called alphabetical criteria~\citep{chaloner1995bayesian} to design a utility function for selecting a superior action set such that the resulting EKF estimation  can accurately approximate the location of each target. One commonly used strategy is to employ the A-optimality, i.e., we consider minimizing the trace of the lower-bound covariance matrix  $\textbf{B}^{j}_{S}(t)$  or equivalently maximize $f_{t}(S)\triangleq\sum_{j=1}^{30}\left(\text{Tr}\big(\textbf{I}^{-1}_{j}(t)\big)-\text{Tr}\big(\textbf{B}^{j}_{S}(t)\big)\right)$
 where ``$\text{Tr}$'' is the trace of matrix and we consider $30$ moving targets. 

 Thus, under the A-optimality criterion, we can model the UAVs' action selection task as a special instance of \emph{online} partition-constrained subset selection problem. Furthermore, according to  recent studies~\citep{harshaw2019submodular,thiery2022two}, we can show that the A-optimal utility function is monotone $\alpha$-weakly DR-submodular and $(\gamma,\beta)$-weakly submodular.

In our simulations, to simplify the computation,  we model the movement of each target as a two-dimensional Brownian motion. Specifically, we set $o_{j}(t)\triangleq o_{j}(t-1)+0.02\cdot\N(\textbf{0}_{2},\textbf{I}_{2})$ where $\textbf{I}_{2}$ is the 2-dimensional identity matrix.  Moreover, we assume that the noise $\xi_{(\theta,s,i)\rightarrow j}$ follows an independent normal distribution, i.e., $\xi_{(\theta,s,i)\rightarrow j}\sim\N(0,0.01)$. As for UAVs, at every iteration $t\in[T]$, we adjust their speeds from a set of $2$, $7$, or $12$ units/s and simultaneously change their movement directions from ``up", ``down", ``left", ``right", or ``diagonally". As a result, the action set  $\V_{i}$ available to each UAV $i\in[20]$ can be mathematically formulated as:
\begin{equation*}
	\V_{i}\triangleq
    \left\{(\theta,s,i):s\in\{2,7,12\}\text{units/s},\theta\in\{\frac{\pi}{4},\frac{\pi}{2},\dots,2\pi\}\right\},
\end{equation*} where $\theta$ denotes the movement angle , $s$ is the speed  and  $i$ represents the unique identifier.

Given the unknown DR ratio $\alpha$ of our investigated set function $f_{t}$, in \cref{graph_show_5}, we perform a brute-force $0.1$-network search to find the  optimal parameter settings for \textbf{Multinoulli-OSGA} algorithm.
Subsequently, we compare the best $\alpha=0.1$ case and the worse $\alpha=1$ scenario with our proposed \textbf{Multinoulli-OSGA} and \textbf{Multinoulli-OSCG} algorithms in \cref{graph_show_4}. According to the results in \cref{graph_show_4}, we can find that the running average utility of our proposed \textbf{Multinoulli-OSGA} and \textbf{Multinoulli-OSCG} algorithms can significantly exceed the baseline `RANDOM' algorithm, which is in accord with our theoretical findings. Moreover, we also find that the \emph{parameter-free} \textbf{Multinoulli-OSCG} algorithm can effectively outperform the \textbf{Multinoulli-OSCG} algorithm with a $0.1$-network search.
It is worth noting that no previous works explore the \emph{online} subset selection problem with weakly submodular objectives. Thus, we adopt the `RANDOM' algorithm as a baseline in \cref{graph_show_4}. In `RANDOM' baseline, we let  each agent $i$ randomly execute an action from its own action set $\V_{i}$ at every time step $t\in[T]$. 

%% file: Appendix_Experiment.tex
\subsubsection{\textbf{Special Maximum Coverage}}
In video summarization tasks, it is easily observed that our \textbf{Multinoulli-SGA} algorithm exhibits exceptional empirical performance. However, according to both \cref{thm3} and \cref{thm:online1}, we know that our \textbf{Multinoulli-SGA} only can ensure a sub-optimal approximation guarantee in the worst case. To verify the theoretical correctness of our \textbf{Multinoulli-SGA} algorithm,  we consider a special maximum coverage problem as discussed in \cref{appendix:bad_case_stationary_point}.

Let  the universe set $U$ consist of $n-1$ elements $\{x_{1},\dots,x_{n-1}\}$ and $n-k$ elements $\{y_{1},\dots,y_{n-k}\}$, all of weight $1$, and $n-1$ elements $\{\epsilon_{1},\dots,\epsilon_{n-1}\}$ of arbitrarily
small weight $\epsilon>0$. Then, we define two different set $A_{i}$ and $B_{i}$ for any $i\in[n]$, that is to say ,
\begin{equation*}
\begin{aligned}
&A_{i}\triangleq\{\epsilon_{i}\}\text{\ for $1\le i\le n-1$,}\ \ \ \ \ \ \ \ \ \ \ A_{n}\triangleq\{x_{1},\dots,x_{n-1}\},\\
&B_{i}\triangleq\{x_{i}\}\text{\ for $1\le i\le n-1$,}\ \ \ \ \ \ \ \ \ \ B_{n}\triangleq\{y_{1},\dots,y_{n-k}\}.\\
\end{aligned}
\end{equation*}
After that, we define a coverage set function $f:2^{\V}\rightarrow\R_{+}$ over these $2n$ distinct set $\{A_{1},\dots,A_{n},B_{1},\dots,B_{n}\}$ where $\V=[2n]$. Specifically, we have that, for any subset $\mathcal{F}\subseteq\V$,
\begin{equation}\label{equ:cover_set}
f(\mathcal{F})=\sum_{v\in\bigcup_{j\in\mathcal{F}}S_{j}}w(v),
\end{equation} where $w(v)$ is the weight of element $v$, $S_{j}=A_{j}$ when $1\le j\le n$ and $S_{j}=B_{j-n}$ as for $n+1\le j\le 2n$.

Moreover, we consider a partition constraint that contains at most one of $\{A_{i},B_{i}\}$ for any $i\in[n]$. If we set $\V_{i}=\{i,i+n\}$($\V=\bigcup_{i\in[n]}\V_{i}=[2n]$), we naturally obtain the following coverage maximization problem:
\begin{equation}\label{equ:coverage}
	\max_{\mathcal{F}\subseteq\V} f(\mathcal{F})\ \ \text{s.t.}\ |\mathcal{F}\cap\V_{i}|\le 1\ \ \forall i\in[n].	
\end{equation}
From the result of \citep{filmus2012power}, we know that the problem~\eqref{equ:coverage} is a submodular maximization problem subject to a partition matroid constraint, namely, $\alpha=\beta=\gamma=1$. A key feature of this coverage maximization problem is that we found that the 
point $\one_{[n]}$ is a local $(1/2)$-approximation stationary point of the \texttt{ME} of the cover function $f$ (See \cref{thm:bad_stationary_point} in  \cref{appendix:bad_case_stationary_point}). 

Motivated by this finding, we empirically compare the performance of our proposed \textbf{Multinoulli-SCG}, \textbf{Multinoulli-SGA} and \textbf{Multinoulli-ASGA} against the standard greedy method and the residual random greedy method as well as the distorted local search across distinct coverage maximization~\eqref{equ:coverage} with different $n$ and $k$, where we uniformly set $\epsilon_{i}=0.01,1\le i\le n-1$.

\begin{table*}[t]
	\caption{Results on coverage maximization(2nd-5th columns) and bayesian A-optimal design(the 6nd-9th column). Note that `obj' denotes the utility function value, where a higher value is preferable, and `queries' represents the magnitude of the total number of function evaluations, that is , the $\log_{10}$ of the total number of value queries to the set objective function, with a smaller value being more favorable.  In each column of `obj', 
		\textcolor{Red!50}{$\blacksquare$} indicates ranking the 1st and \textcolor{Blue!50}{$\blacksquare$} stands for the 2nd.} \label{tab:results_coverage}\vspace{0.4em}
	\small
	\centering
    \setlength{\tabcolsep}{1pt}
	\scalebox{0.87}{\newcommand{\rkone}[1]{{\setlength{\fboxsep}{1.5pt}\colorbox{Red!50}{#1}}}
		\newcommand{\rktwo}[1]{{\setlength{\fboxsep}{1.5pt}\colorbox{Blue!50}{#1}}}
		\begin{tabular}{l||c|c||c|c||c|c||c|c||||c|c||c|c||c|c||c|c}
			\toprule[1.45pt]
			\multirow{2}[3]{*}{\diagbox{Method}{Setting}} & \multicolumn{2}{c||}{\textbf{n=20, k=5}} & \multicolumn{2}{c||}{\textbf{n=30, k=6}} & \multicolumn{2}{c||}{\textbf{n=40, k=8}} & \multicolumn{2}{c||||}{\textbf{n=50, k=10}} & \multicolumn{2}{c||}{\textbf{Housing}}& \multicolumn{2}{c||}{\textbf{Eunite2001}}& \multicolumn{2}{c||}{\textbf{Ionosphere}}& \multicolumn{2}{c}{\textbf{Sonar}}\bigstrut\\
			\cline{2-17}      & obj$\uparrow$& queries$\downarrow$& obj$\uparrow$& queries$\downarrow$& obj$\uparrow$& queries$\downarrow$ & obj$\uparrow$& queries$\downarrow$& obj$\uparrow$& queries$\downarrow$& obj$\uparrow$& queries$\downarrow$& obj$\uparrow$& queries$\downarrow$& obj$\uparrow$& queries$\downarrow$\bigstrut\\
			\hline	
			\hline\bigstrut
			Standard Greedy& 19.19  & 2.90 & 29.29 & 3.26&39.39 & 3.51 & 49.49  & 3.70 &54.36 & 3.70&102.12&3.53&280.49 & 3.55&608.84&3.32 \\
			Residual-Greedy &\rktwo{29.55}  &2.90  & \rktwo{49.44}   & 3.26& \rktwo{64.67}  & 3.51& \rktwo{79.11}  & 3.70 &54.02 & 3.70&102.35&3.53&274.86 & 3.55&593.74&3.32 \\
			Distorted-LS-G & \rkone{34.00}  & 9.44& \rkone{34.00} &9.98& \rkone{71.00} & 10.36 & \rkone{89.00}  & 10.65& 53.98 &10.26&102.41&10.09&277.08 & 10.08&610.42&9.85 \bigstrut[b] \\
			\hline
			\hline\bigstrut
			\textbf{Multinoulli-SGA} & 28.07 & 5.56 & 31.63& 5.72 & 39.34  &  5.83 & 49.44  &  5.92& \rkone{54.64} & 6.61&\rkone{104.77}&6.44&\rkone{282.45} & 6.46&\rkone{613.84}&6.23 \\
			\textbf{Multinoulli-SCG} & \rkone{34.00} & 7.64 & \rkone{53.00}  & 7.99 &\rkone{71.00}  & 8.24 & \rkone{89.00}  & 8.43&\rktwo{54.45}& 8.43&\rktwo{102.97}& 8.26&279.19& 8.28&\rktwo{613.01}&8.05\\
            \textbf{Multinoulli-ASGA} & \rkone{34.00} & 5.56 & \rkone{53.00}  & 5.72 &\rkone{71.00}  & 5.83 & \rkone{89.00}  & 5.92&53.67& 6.61&102.24&6.44&\rktwo{280.74}& 6.46&609.82&6.23 \bigstrut[b] \\
			\midrule[1.45pt]
	\end{tabular}}\vspace{-1.5em}
\end{table*}

From the results in \cref{tab:results_coverage}, we found that both our proposed `\textbf{Multinoulli-SCG}' and `Distorted-LS-G' algorithm eventually select the optimal subset, i.e., $\{B_{1},\dots,B_{n}\}$. In contrast, in all settings of maximum coverage problems, `\textbf{Multinoulli-SGA}' and `Standard-Greedy' algorithm are trapped around the local-optimal set $\{A_{1},\dots,A_{n}\}$. Moreover, the `Residual-Greedy' method oscillates between the optimal subset $\{B_{1},\dots,B_{n}\}$ and the local maximum set  $\{A_{1},\dots,A_{n}\}$. Similar to the video summarization in \cref{sec:videos},   \cref{tab:results_coverage} also demonstrated that the number of function evaluations required by our \textbf{Multinoulli-SCG} is 2 orders of magnitude lower than that of the `Distorted-LS-G' algorithm,which highlights the efficiency and effectiveness of \textbf{Multinoulli-SCG} algorithm.

\subsubsection{\textbf{Bayesian A-Optimal Design}}
We begin by outlining the formulation of Bayesian A-optimal design~\cite{hashemi2019submodular,borsos2024data}. Consider an unknown parameter vector $\theta\in\R^{d}$ to be estimated from noisy linear observations via least squares regression. Our objective  is to choose a set $S$ of linear measurements which have low cost and also maximally reduce the variance of
our OLS estimate. Formally, let $X=[x_{1},\dots,x_{n}]\in\R^{d\times n}$ be a matrix including $n$ candidate measurement vectors. Given a set of
measurement vectors $S\subseteq[n]$, we run the experiments and obtain the noisy linear
observations $y_{S}=X_{S}^{T}\theta+\epsilon_{S}$, where $\epsilon_{S}\sim N(\textbf{0},\sigma^{2}I_{|S|})$. Then, the OLS estimate of $\theta$ becomes $\hat{\theta}=(X_{S}X_{S}^{T})^{-1}X_{S}^{T}y_{S}$. Specially, under a Gaussian prior $\theta\sim N(\textbf{0},\Sigma)$, the total estimation variance equals  $r(S)=tr(\Sigma+\frac{1}{\sigma^{2}}X_{S}X_{S}^{T})^{-1}$ where $tr(\cdot)$ is the trace of a matrix. By defining the variance reduction  $g(S)=r(\emptyset)-r(S)=tr(\Sigma)-tr(\Sigma+\frac{1}{\sigma^{2}}X_{S}X_{S}^{T})^{-1}$, we can reformulate the bayesian A-optimal design as a ($\gamma$,$\beta$)-weakly submodular maximization problem~\cite{thiery2022two} with $\beta=1/\gamma$.


In our experiment, we evaluate our methods on four datasets from \citep{chang2011libsvm}, that is, Housing, Eunite2001, Ionosphere and Sonar. Like \citep{harshaw2019submodular}, we preprocessed the data by normalizing the features to have a zero mean and a standard deviation of $1$. After that, we set $\sigma=1/\sqrt{d}$ and randomly generated a normal prior with covariance $\Sigma=ADA^{T}$ where $d$ is the dimension of data point, $A\sim N(\textbf{0},I)$, $D$ is a diagonal matrix with $D_{i,i}=(i/d)^{2}$. As for the partition constraint, we randomly cut the all data points into $10$ different groups and then select at most one element from each group. Then, we present the results of bayesian A-optimal design in \cref{tab:results_coverage}.

From last four columns of \cref{tab:results_coverage}, we found that the performance of our proposed \textbf{Multinoulli-SCG} and \textbf{Multinoulli-SGA} algorithms surpasses that of three benchmark methods, namely, `Standard Greedy', `Residual-Greedy' and `Distorted-LS-G'. Furthermore, like the previous experiments about video summarization and maximum coverage, the number of value queries to the set function required by our  \textbf{Multinoulli-SCG}, \textbf{Multinoulli-SGA} and  \textbf{Multinoulli-ASGA} is 2, 4 and 4 orders of magnitude lower than that of the state-of-the-art `Distorted-LS-G' algorithm~\citep{thiery2022two,JOGO-Lu}, respectively.  

%% file: Proofs_for_thm_1.tex
In this section, we prove the Theorem~\ref{thm1}.
\begin{proof}
\textbf{1):} At first, we review the definition of Multinoulli Extension. For any given set function $f:2^{\V}\rightarrow\R_{+}$ and any point $(\p_{1},\dots,\p_{K})\in\prod_{k=1}^{K}\Delta_{n_{k}}$, we define its Multinoulli Extension for the subset selection problem~\eqref{equ_problem} as:
\begin{equation}\label{equ_appendix_1}
	F(\p_{1},\dots,\p_{K}):
		=\sum_{e_{\hat{k}}^{\hat{b}}\in\V_{\hat{k}}\cup\{\emptyset\},\forall \hat{b}\in[B_{\hat{k}}],\forall \hat{k}\in[K]}\Big(f\big(\cup_{\hat{k}=1}^{K}\cup_{\hat{b}=1}^{B_{\hat{k}}}\{e_{\hat{k}}^{\hat{b}}\}\big)\prod_{\hat{k}=1}^{K}\prod_{\hat{b}=1}^{B_{\hat{k}}}\text{Pr}(e_{\hat{k}}^{\hat{b}}|\p_{\hat{k}})\Big),
\end{equation}where $\text{Pr}(v_{\hat{k}}^{m} | \p_{\hat{k}}) = p^{m}_{\hat{k}}$ and  $\text{Pr}(\emptyset | \p_{\hat{k}}) = 1 - \sum_{m=1}^{n_{\hat{k}}} p_{\hat{k}}^{m}$ for any $m\in[n_{\hat{k}}]$ and $\hat{k}\in[K]$.

From Eq.\eqref{equ_appendix_1}, for any parameter $p^{m}_{k}$ where $m\in[n_{k}]$ and $k\in[K]$, we have the following equality:
\begin{equation}\label{equ_appendix_2}
		\frac{\partial F}{\partial p^{m}_{k}}(\p_{1},\dots,\p_{K})
		=\sum_{e_{\hat{k}}^{\hat{b}}\in\V_{\hat{k}}\cup\{\emptyset\},\forall \hat{b}\in[B_{\hat{k}}],\forall \hat{k}\in[K]}\Big(f\big(\cup_{\hat{k}=1}^{K}\cup_{\hat{b}=1}^{B_{\hat{k}}}\{e_{\hat{k}}^{\hat{b}}\}\big)\frac{\partial\Big(\prod_{\hat{k}=1}^{K}\prod_{\hat{b}=1}^{B_{\hat{k}}}\text{Pr}(e_{\hat{k}}^{\hat{b}}|\p_{\hat{k}})\Big)}{\partial p^{m}_{k}}\Big).
\end{equation}
Then, according to the definition of $\text{Pr}(e_{\hat{k}}^{\hat{b}}|\p_{\hat{k}})$, we can show that, if $k\neq\hat{k}$, $\frac{\partial\text{Pr}(e_{\hat{k}}^{\hat{b}}|\p_{\hat{k}})}{\partial p^{m}_{k}}=0$ for any $e_{\hat{k}}^{\hat{b}}\in\V_{\hat{k}}\cup\{\emptyset\}$. When $k=\hat{k}$, we also can show that, if $e_{\hat{k}}^{\hat{b}}=v_{k}^{m}$, $\frac{\partial\text{Pr}(e_{\hat{k}}^{\hat{b}}|\p_{\hat{k}})}{\partial p^{m}_{k}}=1$ and if $e_{\hat{k}}^{\hat{b}}=\emptyset$, $\frac{\partial\text{Pr}(e_{\hat{k}}^{\hat{b}}|\p_{\hat{k}})}{\partial p^{m}_{k}}=-1$. As for $e_{\hat{k}}^{\hat{b}}\notin\{v_{k}^{m},\emptyset\}$, when $k=\hat{k}$, $\frac{\partial\text{Pr}(e_{\hat{k}}^{\hat{b}}|\p_{\hat{k}})}{\partial p^{m}_{k}}=0$. As a result, we can rewrite the Eq.\eqref{equ_appendix_2} as:
\begin{equation}\label{equ_appendix_3}
	\begin{aligned}
	\frac{\partial F}{\partial p^{m}_{k}}
&=\sum_{e_{\hat{k}}^{\hat{b}}\in\V_{\hat{k}}\cup\{\emptyset\},\forall \hat{b}\in[B_{\hat{k}}],\forall \hat{k}\in[K]}\Bigg(f\big(\cup_{\hat{k}=1}^{K}\cup_{\hat{b}=1}^{B_{\hat{k}}}\{e_{\hat{k}}^{\hat{b}}\}\big)\sum_{\hat{b}_{1}\in[B_{k}]}\Bigg(\Big(\prod_{(\hat{k},\hat{b})\neq(k,\hat{b}_{1})}\text{Pr}(e_{\hat{k}}^{\hat{b}}|\p_{\hat{k}})\Big)\frac{\partial\text{Pr}(e_{k}^{\hat{b}_{1}}|\p_{k})}{\partial p^{m}_{k}}\Bigg)\Bigg)\\
&=\sum_{\hat{b}_{1}\in[B_{k}]}\Bigg(\sum_{e_{\hat{k}}^{\hat{b}}\in\V_{\hat{k}}\cup\{\emptyset\},\forall \hat{b}\in[B_{\hat{k}}],\forall \hat{k}\in[K]}\Big(f\big(\cup_{\hat{k}=1}^{K}\cup_{\hat{b}=1}^{B_{\hat{k}}}\{e_{\hat{k}}^{\hat{b}}\}\big)\frac{\partial\text{Pr}(e_{k}^{\hat{b}_{1}}|\p_{k})}{\partial p^{m}_{k}}\Big(\prod_{(\hat{k},\hat{b})\neq(k,\hat{b}_{1})}\text{Pr}(e_{\hat{k}}^{\hat{b}}|\p_{\hat{k}})\Big)\Big)\Bigg)\\
&=\sum_{\hat{b}_{1}\in[B_{k}]}\Bigg(\sum_{e_{\hat{k}}^{\hat{b}}\in\V_{\hat{k}}\cup\{\emptyset\},(\hat{k},\hat{b})\neq(k,\hat{b}_{1})}\Bigg( \sum_{e_{k}^{\hat{b}}\in\V_{k}\cup\{\emptyset\}}\Bigg(f\big(\cup_{\hat{k}=1}^{K}\cup_{\hat{b}=1}^{B_{\hat{k}}}\{e_{\hat{k}}^{\hat{b}}\}\big)\frac{\partial\text{Pr}(e_{k}^{\hat{b}_{1}}|\p_{k})}{\partial p^{m}_{k}}\prod_{(\hat{k},\hat{b})\neq(k,\hat{b}_{1})}\text{Pr}(e_{\hat{k}}^{\hat{b}}|\p_{\hat{k}})\Bigg)\Bigg)\Bigg)\\
&=\sum_{\hat{b}_{1}\in[B_{k}]}\Bigg(\sum_{e_{\hat{k}}^{\hat{b}}\in\V_{\hat{k}}\cup\{\emptyset\},(\hat{k},\hat{b})\neq(k,\hat{b}_{1})}\Bigg(f\Big(v_{k}^{m}|\cup_{(\hat{k},\hat{b})\neq(k,\hat{b}_{1})}\{e_{\hat{k}}^{\hat{b}}\}\Big)\prod_{(\hat{k},\hat{b})\neq(k,\hat{b}_{1})}\text{Pr}(e_{\hat{k}}^{\hat{b}}|\p_{\hat{k}})\Bigg)\Bigg)\\
&=B_{k}\Bigg(\sum_{e_{\hat{k}}^{\hat{b}}\in\V_{\hat{k}}\cup\{\emptyset\},(\hat{k},\hat{b})\neq(k,1)}\Bigg(f\Big(v_{k}^{m}|\cup_{(\hat{k},\hat{b})\neq(k,1)}\{e_{\hat{k}}^{\hat{b}}\}\Big)\prod_{(\hat{k},\hat{b})\neq(k,1)}\text{Pr}(e_{\hat{k}}^{\hat{b}}|\p_{\hat{k}})\Bigg)\Bigg)\\
&=B_{k}\Bigg(\Big(\sum_{e_{k}^{1}\in\V_{k}\cup\{\emptyset\}}\text{Pr}(e_{k}^{1}|\p_{k})\Big)\sum_{e_{\hat{k}}^{\hat{b}}\in\V_{\hat{k}}\cup\{\emptyset\},(\hat{k},\hat{b})\neq(k,1)}\Bigg(f\Big(v_{k}^{m}|\cup_{(\hat{k},\hat{b})\neq(k,1)}\{e_{\hat{k}}^{\hat{b}}\}\Big)\prod_{(\hat{k},\hat{b})\neq(k,1)}\text{Pr}(e_{\hat{k}}^{\hat{b}}|\p_{\hat{k}})\Bigg)\Bigg)\\
&=B_{k}\Bigg(\sum_{e_{\hat{k}}^{\hat{b}}\in\V_{\hat{k}}\cup\{\emptyset\},\forall \hat{b}\in[B_{\hat{k}}],\forall \hat{k}\in[K]}\Bigg(f\Big(v_{k}^{m}|\cup_{(\hat{k},\hat{b})\neq(k,1)}\{e_{\hat{k}}^{\hat{b}}\}\Big)\prod_{\hat{k}\in[K]}\prod_{\hat{b}\in[B_{k}]}\text{Pr}(e_{\hat{k}}^{\hat{b}}|\p_{\hat{k}})\Bigg)\Bigg)\\
&=B_{k}\Bigg(\E_{e_{\hat{k}}^{\hat{b}}\sim\emph{Multi($\p_{\hat{k}}$)}}\Big(f\Big(v_{k}^{m}\Big|\cup_{(\hat{k},\hat{b})\neq(k,1)}\{e_{\hat{k}}^{\hat{b}}\}\Big)\Big)\Bigg),
	\end{aligned}
\end{equation}  where the fifth equality comes from that each random element $e_{k}^{\hat{b}_{1}}$ is independently drawn the same multinoulli distribution Multi($\p_{k}$) for any $\hat{b}_{1}\in[B_{k}]$ and the sixth equality follows from  $\sum_{e_{k}^{1}\in\V_{k}\cup\{\emptyset\}}\text{Pr}(e_{k}^{1}|\p_{k})=1$.

\textbf{2):} Next, we prove the second point of Theorem~\ref{thm1}.  At first, the monotonicity of $f$ implies that for any two subsets $A\subseteq B\subseteq\V$, $f(A)\le f(B)$ such that we know that $f\Big(v_{k}^{m}\Big|\cup_{(\hat{k},\hat{b})\neq(k,1)}\{e_{\hat{k}}^{\hat{b}}\}\Big)\ge0$ for any random elements $e_{\hat{k}}^{\hat{b}},\forall\hat{b}\in[B_{\hat{k}}],\forall\hat{k}\in[K]$. As a result, we have If $f$ is monotone, then $\frac{\partial F}{\partial p_{k}^{m}}\ge 0, \forall k\in[K], m\in[n_{k}]$.

\textbf{3):} As for the third point, we firstly unify the process of generating random elements regarding two different multinoulli distributions with parameters $\p_{k}=(p_{k}^{1},\dots,p_{k}^{n_{k}})\in\Delta_{n_{k}}$ and  $\hat{\p}_{k}=(\hat{p}_{k}^{1},\dots,\hat{p}_{k}^{n_{k}})\in\Delta_{n_{k}}$ where $\hat{p}_{k}^{m}\ge p_{k}^{m}$ for any $m\in[n_{k}]$ and $k\in[K]$. More specifically, we will transform the sampling process from each multinoulli distribution Multi($\p_{k}$) or Multi($\hat{\p}_{k}$)
into a function of two independent uniform random variables on the interval $[0,1]$.  Namely, for any two   independent uniform random variables $X,Y$ on the interval $[0,1]$, we define that
\begin{equation*}
	e(X,Y,\p_{k},\hat{\p}_{k})=\left\{\begin{aligned}
		&v_{k}^{1}\ \ \ \text{If}\ X\in[0,p_{k}^{1})\\
		&v_{k}^{c}\ \ \ \text{If}\ X\in[\sum_{m=1}^{c-1}p_{k}^{m},\sum_{m=1}^{c}p_{k}^{m})\ \text{for some integer}\ c\in [2,n_{k}]\\
		&\emptyset\ \ \ \ \ \text{If}\ X\ge\sum_{m=1}^{n_{k}}p_{k}^{m}
	\end{aligned}\right.
\end{equation*}
We can easily check that, when $X$ and $Y$ are uniform random variables over the interval $[0,1]$, the random element $e(X,Y,\p_{k},\hat{\p}_{k})$
follows the multinoulli distribution Multi($\p_{k}$) over the community $\V_{k}:=\{v_{k}^{1},\dots,v_{k}^{n_{k}}\}$, where  $\p_{k}=(p_{k}^{1},\dots,p_{k}^{n_{k}})\in\Delta_{n_{k}}$  and  $\text{Pr}\big(e(X,Y,\p_{k},\hat{\p}_{k})=v_{k}^{m}\big)=p_{k}^{m}$. Similarly, we also can generate a random element $\widehat{e}$  from the community $\V_{k}$ according  to the multinoulli distribution Multi($\hat{\p}_{k}$) if we set
\begin{equation*}
	 \widehat{e}(X,Y,\p_{k},\hat{\p}_{k})=\left\{\begin{aligned}
		&e\ \ \ \ \ \ \text{If}\ X<\sum_{m=1}^{n_{k}}p_{k}^{m}\\
		&v_{k}^{1}\ \ \ \ \text{If}\ X\ge\sum_{m=1}^{n_{k}}p_{k}^{m}\ \text{and}\ Y\in[0,\frac{\hat{p}_{k}^{1}-p_{k}^{1}}{1-\sum_{m=1}^{n_{k}}p_{k}^{m}})\\
		&v_{k}^{c}\ \ \ \ \text{If}\ X\ge\sum_{m=1}^{n_{k}}p_{k}^{m}\ \text{and}\ Y\in\Big[\frac{\sum_{m=1}^{c-1}(\hat{p}_{k}^{m}-p_{k}^{m})}{1-\sum_{m=1}^{n_{k}}p_{k}^{m}},\frac{\sum_{m=1}^{c}(\hat{p}_{k}^{m}-p_{k}^{m})}{1-\sum_{m=1}^{n_{k}}p_{k}^{m}}\Big)\ \text{for}\ c\in [2,n_{k}]\\
		&\emptyset\ \ \ \ \ \ \text{If}\ X\ge\sum_{m=1}^{n_{k}}p_{k}^{m}\ \text{and}\ Y\ge\sum_{m=1}^{n_{k}}(\hat{p}_{k}^{m}-p_{k}^{m})
	\end{aligned}\right.
\end{equation*} 
From the definition of $e(X,Y,\p_{k},\hat{\p}_{k})$ and $\widehat{e}(X,Y,\p_{k},\hat{\p}_{k})$, we can show that, for any fixed $X,Y$, $\{e(X,Y,\p_{k},\hat{\p}_{k})\}\subseteq\{\widehat{e}(X,Y,\p_{k},\hat{\p}_{k})\}$. Thus, if we generate multiple independent pairs $(X_{\hat{k}}^{\hat{b}},Y_{\hat{k}}^{\hat{b}})$ for any $\hat{b}\in[B_{\hat{k}}]$ and $\hat{k}\in[K]$,  from the first point of Theorem~\ref{thm1}, we have that,
\begin{equation*}
	\begin{aligned}
		&\frac{\partial F}{\partial p_{k}^{m}}(\p_{1},\dots,\p_{K})=B_{k}\Bigg(\E\Big(f\Big(v_{k}^{m}\Big|\cup_{(\hat{k},\hat{b})\neq(k,1)}\{e_{\hat{k}}^{\hat{b}}(X_{\hat{k}}^{\hat{b}},Y_{\hat{k}}^{\hat{b}},\p_{k},\hat{\p}_{k})\}\Big)\Big)\Bigg)\\
			&\frac{\partial F}{\partial p_{k}^{m}}(\hat{\p}_{1},\dots,\hat{\p}_{K})=B_{k}\Bigg(\E\Big(f\Big(v_{k}^{m}\Big|\cup_{(\hat{k},\hat{b})\neq(k,1)}\{\widehat{e}_{\hat{k}}^{\hat{b}}(X_{\hat{k}}^{\hat{b}},Y_{\hat{k}}^{\hat{b}},\p_{k},\hat{\p}_{k})\}\Big)\Big)\Bigg).
	\end{aligned}
\end{equation*}
Due to $\{e_{\hat{k}}^{\hat{b}}(X_{\hat{k}}^{\hat{b}},Y_{\hat{k}}^{\hat{b}},\p_{k},\hat{\p}_{k})\}\subseteq\{\widehat{e}_{\hat{k}}^{\hat{b}}(X_{\hat{k}}^{\hat{b}},Y_{\hat{k}}^{\hat{b}},\p_{k},\hat{\p}_{k})\}$, we can show that, 
\begin{equation*}
\Big(\cup_{(\hat{k},\hat{b})\neq(k,1)}\{e_{\hat{k}}^{\hat{b}}(X_{\hat{k}}^{\hat{b}},Y_{\hat{k}}^{\hat{b}},\p_{k},\hat{\p}_{k})\Big)\subseteq\Big( \cup_{(\hat{k},\hat{b})\neq(k,1)}\{\widehat{e}_{\hat{k}}^{\hat{b}}(X_{\hat{k}}^{\hat{b}},Y_{\hat{k}}^{\hat{b}},\p_{k},\hat{\p}_{k})\}\Big).
\end{equation*}
As a result, from the definition of weakly DR-submodularity, we know that if $f$ is $\alpha$-weakly DR-submodular, we have $\frac{\partial F}{\partial p_{k}^{m}}(\p_{1},\dots,\p_{K})\ge\alpha\frac{\partial F}{\partial p_{k}^{m}}(\hat{\p}_{1},\dots,\hat{\p}_{K})$ such that $\nabla F(\p_{1},\dots,\p_{K})\ge \alpha \nabla F(\hat{\p}_{1},\dots,\hat{\p}_{K})$.

\textbf{4):} We prove the fourth point of Theorem~\ref{thm1}. From the definition of Multinoulli Extension and previous definitions of $e(X,Y,\p_{k},\hat{\p}_{k})$ and $\widehat{e}(X,Y,\p_{k},\hat{\p}_{k})$, we can infer that the Multinoulli Extension $F$ of set function $f$ satisfies the following  relationships:
\begin{equation*}
	\begin{aligned}
		& F(\p_{1},\dots,\p_{K})=\E\Bigg(f\Big(\cup_{\hat{k}=1}^{K}\cup_{\hat{b}=1}^{B_{k}}\{e_{\hat{k}}^{\hat{b}}(X_{\hat{k}}^{\hat{b}},Y_{\hat{k}}^{\hat{b}},\p_{k},\hat{\p}_{k})\}\Big)\Bigg)\\
		& F(\hat{\p}_{1},\dots,\hat{\p}_{K})=\E\Bigg(f\Big(\cup_{\hat{k}=1}^{K}\cup_{\hat{b}=1}^{B_{k}}\{\widehat{e}_{\hat{k}}^{\hat{b}}(X_{\hat{k}}^{\hat{b}},Y_{\hat{k}}^{\hat{b}},\p_{k},\hat{\p}_{k})\}\Big)\Bigg).
	\end{aligned}
\end{equation*}
As a result, we have 
\begin{equation*}
 F(\hat{\p}_{1},\dots,\hat{\p}_{K})-F(\p_{1},\dots,\p_{K})=\E\Bigg(f\Big(\cup_{\hat{k}=1}^{K}\cup_{\hat{b}=1}^{B_{k}}\{\widehat{e}_{\hat{k}}^{\hat{b}}(X_{\hat{k}}^{\hat{b}},Y_{\hat{k}}^{\hat{b}},\p_{k},\hat{\p}_{k})\}\Big)-f\Big(\cup_{\hat{k}=1}^{K}\cup_{\hat{b}=1}^{B_{k}}\{e_{\hat{k}}^{\hat{b}}(X_{\hat{k}}^{\hat{b}},Y_{\hat{k}}^{\hat{b}},\p_{k},\hat{\p}_{k})\}\Big)\Bigg).
\end{equation*}
If $f$ is  $\gamma$-weakly submodular from below and $\hat{\p}_{k}\ge \p_{k}$ for any $k\in[K]$, we have that
\begin{equation*}
	\begin{aligned}
\gamma\Big(F(\hat{\p}_{1},\dots,\hat{\p}_{K})-F(\p_{1},\dots,\p_{K})\Big)&=\gamma\E\Bigg(f\Big(\cup_{\hat{k}=1}^{K}\cup_{\hat{b}=1}^{B_{k}}\{\widehat{e}_{\hat{k}}^{\hat{b}}(X_{\hat{k}}^{\hat{b}},Y_{\hat{k}}^{\hat{b}},\p_{k},\hat{\p}_{k})\}\Big)-f\Big(\cup_{\hat{k}=1}^{K}\cup_{\hat{b}=1}^{B_{k}}\{e_{\hat{k}}^{\hat{b}}(X_{\hat{k}}^{\hat{b}},Y_{\hat{k}}^{\hat{b}},\p_{k},\hat{\p}_{k})\}\Big)\Bigg)\\
	&\le\sum_{\bar{k}=1}^{K}\sum_{\bar{b}=1}^{B_{\bar{k}}}\E\Bigg(f\Big(\widehat{e}_{\bar{k}}^{\bar{b}}(X_{\bar{k}}^{\bar{b}},Y_{\bar{k}}^{\bar{b}},\p_{\bar{k}},\hat{\p}_{\bar{k}})\Big|\cup_{\hat{k}=1}^{K}\cup_{\hat{b}=1}^{B_{k}}\{e_{\hat{k}}^{\hat{b}}(X_{\hat{k}}^{\hat{b}},Y_{\hat{k}}^{\hat{b}},\p_{k},\hat{\p}_{k})\}\Big)\Bigg),
	\end{aligned}
\end{equation*} where the final inequality follows from the $\gamma$-weakly submodular property of $f$, namely, for any two subsets $A\subseteq B\subseteq\V$, $	\sum_{v\in B}f(v|A)=\sum_{v\in B\setminus A}f(v|A)\ge \gamma\Big(f(B)-f(A)\Big)$.

Then, if $X_{\bar{k}}^{\bar{b}}<\sum_{m=1}^{n_{\bar{k}}}p_{\bar{k}}^{m}$, we know that $\widehat{e}_{\bar{k}}^{\bar{b}}(X_{\bar{k}}^{\bar{b}},Y_{\bar{k}}^{\bar{b}},\p_{\bar{k}},\hat{\p}_{\bar{k}})=e_{\bar{k}}^{\bar{b}}(X_{\bar{k}}^{\bar{b}},Y_{\bar{k}}^{\bar{b}},\p_{\bar{k}},\hat{\p}_{\bar{k}})$ such that $f\Big(\widehat{e}_{\bar{k}}^{\bar{b}}(X_{\bar{k}}^{\bar{b}},Y_{\bar{k}}^{\bar{b}},\p_{\bar{k}},\hat{\p}_{\bar{k}})\Big|\cup_{\hat{k}=1}^{K}\cup_{\hat{b}=1}^{B_{k}}\{e_{\hat{k}}^{\hat{b}}(X_{\hat{k}}^{\hat{b}},Y_{\hat{k}}^{\hat{b}},\p_{k},\hat{\p}_{k})\}\Big)=0$. As for the case that $X_{\bar{k}}^{\bar{b}}\ge\sum_{m=1}^{n_{\bar{k}}}p_{\bar{k}}^{m}$, namely, $e_{\bar{k}}^{\bar{b}}(X_{\bar{k}}^{\bar{b}},Y_{\bar{k}}^{\bar{b}},\p_{\bar{k}},\hat{\p}_{\bar{k}})=\emptyset$, which means that we have the following equality:
\begin{equation}\label{equ:appendix:4}
	\begin{aligned}
	&f\Big(\widehat{e}_{\bar{k}}^{\bar{b}}(X_{\bar{k}}^{\bar{b}},Y_{\bar{k}}^{\bar{b}},\p_{\bar{k}},\hat{\p}_{\bar{k}})\Big|\cup_{\hat{k}=1}^{K}\cup_{\hat{b}=1}^{B_{k}}\{e_{\hat{k}}^{\hat{b}}(X_{\hat{k}}^{\hat{b}},Y_{\hat{k}}^{\hat{b}},\p_{k},\hat{\p}_{k})\}\Big)\\
	&=f\Big(\widehat{e}_{\bar{k}}^{\bar{b}}(X_{\bar{k}}^{\bar{b}},Y_{\bar{k}}^{\bar{b}},\p_{\bar{k}},\hat{\p}_{\bar{k}})\Big|\cup_{(\hat{k},\hat{b})\neq(\bar{k},\bar{b})}\{e_{\hat{k}}^{\hat{b}}(X_{\hat{k}}^{\hat{b}},Y_{\hat{k}}^{\hat{b}},\p_{k},\hat{\p}_{k})\}\Big).
	\end{aligned}
\end{equation}
Therefore, we have 
\begin{equation*}
	\begin{aligned}
		&\gamma\Big(F(\hat{\p}_{1},\dots,\hat{\p}_{K})-F(\p_{1},\dots,\p_{K})\Big)\\&\le\sum_{\bar{k}=1}^{K}\sum_{\bar{b}=1}^{B_{\bar{k}}}\E\Bigg(f\Big(\widehat{e}_{\bar{k}}^{\bar{b}}(X_{\bar{k}}^{\bar{b}},Y_{\bar{k}}^{\bar{b}},\p_{\bar{k}},\hat{\p}_{\bar{k}})\Big|\cup_{\hat{k}=1}^{K}\cup_{\hat{b}=1}^{B_{k}}\{e_{\hat{k}}^{\hat{b}}(X_{\hat{k}}^{\hat{b}},Y_{\hat{k}}^{\hat{b}},\p_{k},\hat{\p}_{k})\}\Big)\Bigg)\\
		&=\sum_{\bar{k}=1}^{K}\sum_{\bar{b}=1}^{B_{\bar{k}}}\E\Bigg(\E\Bigg(f\Big(\widehat{e}_{\bar{k}}^{\bar{b}}(X_{\bar{k}}^{\bar{b}},Y_{\bar{k}}^{\bar{b}},\p_{\bar{k}},\hat{\p}_{\bar{k}})\Big|\cup_{\hat{k}=1}^{K}\cup_{\hat{b}=1}^{B_{k}}\{e_{\hat{k}}^{\hat{b}}(X_{\hat{k}}^{\hat{b}},Y_{\hat{k}}^{\hat{b}},\p_{k},\hat{\p}_{k})\}\Big)\Bigg|(X_{k}^{b},Y_{k}^{b}),(k,b)\neq(\bar{k},\bar{b})\Bigg)\Bigg)\\
		&=\sum_{\bar{k}=1}^{K}\sum_{\bar{b}=1}^{B_{\bar{k}}}\E\Bigg(\E\Bigg(f\Big(\widehat{e}_{\bar{k}}^{\bar{b}}(X_{\bar{k}}^{\bar{b}},Y_{\bar{k}}^{\bar{b}},\p_{\bar{k}},\hat{\p}_{\bar{k}})\Big|\cup_{\hat{k}=1}^{K}\cup_{\hat{b}=1}^{B_{k}}\{e_{\hat{k}}^{\hat{b}}(X_{\hat{k}}^{\hat{b}},Y_{\hat{k}}^{\hat{b}},\p_{k},\hat{\p}_{k})\}\Big)I(X_{\bar{k}}^{\bar{b}}<\sum_{m=1}^{n_{\bar{k}}}p_{\bar{k}}^{m})\Bigg|(X_{k}^{b},Y_{k}^{b}),(k,b)\neq(\bar{k},\bar{b})\Bigg)\Bigg)\\
		&+\sum_{\bar{k}=1}^{K}\sum_{\bar{b}=1}^{B_{\bar{k}}}\E\Bigg(\E\Bigg(f\Big(\widehat{e}_{\bar{k}}^{\bar{b}}(X_{\bar{k}}^{\bar{b}},Y_{\bar{k}}^{\bar{b}},\p_{\bar{k}},\hat{\p}_{\bar{k}})\Big|\cup_{\hat{k}=1}^{K}\cup_{\hat{b}=1}^{B_{k}}\{e_{\hat{k}}^{\hat{b}}(X_{\hat{k}}^{\hat{b}},Y_{\hat{k}}^{\hat{b}},\p_{k},\hat{\p}_{k})\}\Big)I(X_{\bar{k}}^{\bar{b}}\ge\sum_{m=1}^{n_{\bar{k}}}p_{\bar{k}}^{m})\Bigg|(X_{k}^{b},Y_{k}^{b}),(k,b)\neq(\bar{k},\bar{b})\Bigg)\Bigg)\\
		&=\sum_{\bar{k}=1}^{K}\sum_{\bar{b}=1}^{B_{\bar{k}}}\E\Bigg(\E\Bigg(f\Big(\widehat{e}_{\bar{k}}^{\bar{b}}(X_{\bar{k}}^{\bar{b}},Y_{\bar{k}}^{\bar{b}},\p_{\bar{k}},\hat{\p}_{\bar{k}})\Big|\cup_{\hat{k}=1}^{K}\cup_{\hat{b}=1}^{B_{k}}\{e_{\hat{k}}^{\hat{b}}(X_{\hat{k}}^{\hat{b}},Y_{\hat{k}}^{\hat{b}},\p_{k},\hat{\p}_{k})\}\Big)I(X_{\bar{k}}^{\bar{b}}\ge\sum_{m=1}^{n_{\bar{k}}}p_{\bar{k}}^{m})\Bigg|(X_{k}^{b},Y_{k}^{b}),(k,b)\neq(\bar{k},\bar{b})\Bigg)\Bigg),
	\end{aligned}
\end{equation*} where the final equality comes from  $f\Big(\widehat{e}_{\bar{k}}^{\bar{b}}(X_{\bar{k}}^{\bar{b}},Y_{\bar{k}}^{\bar{b}},\p_{\bar{k}},\hat{\p}_{\bar{k}})\Big|\cup_{\hat{k}=1}^{K}\cup_{\hat{b}=1}^{B_{k}}\{e_{\hat{k}}^{\hat{b}}(X_{\hat{k}}^{\hat{b}},Y_{\hat{k}}^{\hat{b}},\p_{k},\hat{\p}_{k})\}\Big)=0$ when $X_{\bar{k}}^{\bar{b}}<\sum_{m=1}^{n_{\bar{k}}}p_{\bar{k}}^{m}$.
\end{proof}

Then, from Eq.\eqref{equ:appendix:4}, we also can show that
\begin{equation}\label{equ:appendix:5}
	\begin{aligned}
	&\E\Bigg(f\Big(\widehat{e}_{\bar{k}}^{\bar{b}}(X_{\bar{k}}^{\bar{b}},Y_{\bar{k}}^{\bar{b}},\p_{\bar{k}},\hat{\p}_{\bar{k}})\Big|\cup_{\hat{k}=1}^{K}\cup_{\hat{b}=1}^{B_{k}}\{e_{\hat{k}}^{\hat{b}}(X_{\hat{k}}^{\hat{b}},Y_{\hat{k}}^{\hat{b}},\p_{k},\hat{\p}_{k})\}\Big)I(X_{\bar{k}}^{\bar{b}}\ge\sum_{m=1}^{n_{\bar{k}}}p_{\bar{k}}^{m})\Bigg|(X_{k}^{b},Y_{k}^{b}),(k,b)\neq(\bar{k},\bar{b})\Bigg)\\
	&=\sum_{c=1}^{n_{\bar{k}}}\text{Pr}\Big(\widehat{e}_{\bar{k}}^{\bar{b}}(X_{\bar{k}}^{\bar{b}},Y_{\bar{k}}^{\bar{b}},\p_{\bar{k}},\hat{\p}_{\bar{k}})=v_{\bar{k}}^{c},X_{\bar{k}}^{\bar{b}}\ge\sum_{m=1}^{n_{\bar{k}}}p_{\bar{k}}^{m}\Big)f\Big(v_{\bar{k}}^{c}\Big|\cup_{(\hat{k},\hat{b})\neq(\bar{k},\bar{b})}\{e_{\hat{k}}^{\hat{b}}(X_{\hat{k}}^{\hat{b}},Y_{\hat{k}}^{\hat{b}},\p_{k},\hat{\p}_{k})\}\Big)\\
	&=\sum_{c=1}^{n_{\bar{k}}}(\hat{p}^{c}_{\bar{k}}-p^{c}_{\bar{k}})f\Big(v_{\bar{k}}^{c}\Big|\cup_{(\hat{k},\hat{b})\neq(\bar{k},\bar{b})}\{e_{\hat{k}}^{\hat{b}}(X_{\hat{k}}^{\hat{b}},Y_{\hat{k}}^{\hat{b}},\p_{k},\hat{\p}_{k})\}\Big),
	\end{aligned}
\end{equation} where the equality comes from that
\begin{equation*}
\begin{aligned}
  \text{Pr}\Big(\widehat{e}_{\bar{k}}^{\bar{b}}(X_{\bar{k}}^{\bar{b}},Y_{\bar{k}}^{\bar{b}},\p_{\bar{k}},\hat{\p}_{\bar{k}})=v_{\bar{k}}^{c},X_{\bar{k}}^{\bar{b}}\ge\sum_{m=1}^{n_{\bar{k}}}p_{\bar{k}}^{m}\Big)&=\text{Pr}\Big(X_{\bar{k}}^{\bar{b}}\ge\sum_{m=1}^{n_{\bar{k}}}p_{\bar{k}}^{m},Y_{\bar{k}}^{\bar{b}}\in \Big[\frac{\sum_{m=1}^{c-1}(\hat{p}_{\bar{k}}^{m}-p_{\bar{k}}^{m})}{1-\sum_{m=1}^{n_{\bar{k}}}p_{\bar{k}}^{m}},\frac{\sum_{m=1}^{c}(\hat{p}_{\bar{k}}^{m}-p_{\bar{k}}^{m})}{1-\sum_{m=1}^{n_{\bar{k}}}p_{\bar{k}}^{m}}\Big)\Big)\\&=(\hat{p}^{m}_{\bar{k}}-p^{m}_{\bar{k}}).  
\end{aligned}
\end{equation*}

As a result, we have 
\begin{equation*}
	\begin{aligned}
		&\gamma\Big(F(\hat{\p}_{1},\dots,\hat{\p}_{K})-F(\p_{1},\dots,\p_{K})\Big)\\
		&=\sum_{\bar{k}=1}^{K}\sum_{\bar{b}=1}^{B_{\bar{k}}}\E\Bigg(\sum_{m=1}^{n_{\bar{k}}}(\hat{p}^{m}_{\bar{k}}-p^{m}_{\bar{k}})f\Big(v_{\bar{k}}^{m}\Big|\cup_{(\hat{k},\hat{b})\neq(\bar{k},\bar{b})}\{e_{\hat{k}}^{\hat{b}}(X_{\hat{k}}^{\hat{b}},Y_{\hat{k}}^{\hat{b}},\p_{k},\hat{\p}_{k})\}\Big)\Bigg)\\
	&=\sum_{\bar{k}=1}^{K}\sum_{\bar{b}=1}^{B_{\bar{k}}}\E\Bigg(\sum_{m=1}^{n_{\bar{k}}}(\hat{p}^{m}_{\bar{k}}-p^{m}_{\bar{k}})f\Big(v_{\bar{k}}^{m}\Big|\cup_{(\hat{k},\hat{b})\neq(\bar{k},1)}\{e_{\hat{k}}^{\hat{b}}(X_{\hat{k}}^{\hat{b}},Y_{\hat{k}}^{\hat{b}},\p_{k},\hat{\p}_{k})\}\Big)\Bigg)\\
	&=\sum_{\bar{k}=1}^{K}\sum_{m=1}^{n_{\bar{k}}}(\hat{p}^{m}_{\bar{k}}-p^{m}_{\bar{k}})B_{\bar{k}}\E\Bigg(f\Big(v_{\bar{k}}^{m}\Big|\cup_{(\hat{k},\hat{b})\neq(\bar{k},1)}\{e_{\hat{k}}^{\hat{b}}(X_{\hat{k}}^{\hat{b}},Y_{\hat{k}}^{\hat{b}},\p_{k},\hat{\p}_{k})\}\Big)\Bigg)\\
	&=\sum_{\bar{k}=1}^{K}\sum_{m=1}^{n_{\bar{k}}}(\hat{p}^{m}_{\bar{k}}-p^{m}_{\bar{k}})	\frac{\partial F}{\partial p^{m}_{\bar{k}}}(\p_{1},\dots,\p_{K})=\Big\langle(\hat{\p}_{1},\dots,\hat{\p}_{K})-(\p_{1},\dots,\p_{K}),\nabla F(\p_{1},\dots,\p_{K})\Big\rangle,
	\end{aligned}
\end{equation*}  where the second equality follows from the independence of $e_{\hat{k}}^{\hat{b}}$ for any $\hat{b}\in[B_{\hat{k}}]$ and $k\in[K]$.

%% file: Proofs_for_Thm2.tex
In this section, we verify the Theorem~\ref{thm2}.
\begin{proof}
At first, we recall that $\V_{k}:=\{v_{k}^{1},\dots,v_{k}^{n_{k}}\}$ for any $k\in[K]$. Therefore, for any subset $S$ within the partition constraint of problem~\eqref{equ_problem}, we assume $|S\cap\V_{k}|=s_{k}\le B_{k}$ and we can represent each $S\cap\V_{k}$ as 
\begin{equation*}
S\cap\V_{k}=\{v^{m_{k}^{1}}_{k},\dots,v^{m_{k}^{s_{k}}}_{k}\},
\end{equation*}where $m_{k}^{b_{1}}\neq m_{k}^{b_{2}}\in[n_{k}]$ for any $b_1,b_2\in[s_{k}]$ for any $k\in[K]$.

As a result, we can rewrite that
\begin{equation}\label{equ_appendix2_1}
	\begin{aligned}
	&\left\langle\sum_{k=1}^{K}\frac{\one_{S\cap\V_{k}}}{B_{k}},\nabla F(\p_{1},\dots,\p_{K})\right\rangle=\sum_{k=1}^{K}\sum_{b=1}^{s_k}\frac{1}{B_{k}}\frac{\partial F}{\partial p_{k}^{m_k^b}}(\p_{1},\dots,\p_{K})\\
	&=\sum_{k=1}^{K}\sum_{b=1}^{s_k}\Bigg(\E_{e_{\hat{k}}^{\hat{b}}\sim\emph{Multi($\p_{\hat{k}}$)}}\Big(f\Big(v_{k}^{m_{k}^{b}}\Big|\cup_{(\hat{k},\hat{b})\neq(k,1)}\{e_{\hat{k}}^{\hat{b}}\}\Big)\Big)\Bigg)\\
	&=\sum_{k=1}^{K}\sum_{b=1}^{s_k}\Bigg(\E_{e_{\hat{k}}^{\hat{b}}\sim\emph{Multi($\p_{\hat{k}}$)}}\Big(f\Big(v_{k}^{m_{k}^{b}}\Big|\cup_{(\hat{k},\hat{b})\neq(k,b)}\{e_{\hat{k}}^{\hat{b}}\}\Big)\Big)\Bigg),
\end{aligned}
\end{equation} where the final equality follows from that, for any fixed $\hat{k}\in[K]$, the elements $e_{\hat{k}}^{b},\forall b\in[B_{\hat{k}}]$ are independently drawn from the same multinoulli distribution Multi($\p_{\hat{k}}$). 

When the set function $f$ is $\alpha$-weakly monotone DR-submodular, we can show that
\begin{equation}\label{equ_appendix2_2}
	\begin{aligned}
		&\left\langle\sum_{k=1}^{K}\frac{\one_{S\cap\V_{k}}}{B_{k}},\nabla F(\p_{1},\dots,\p_{K})\right\rangle\\
		&=\sum_{k=1}^{K}\sum_{b=1}^{s_k}\Bigg(\E_{e_{\hat{k}}^{\hat{b}}\sim\emph{Multi($\p_{\hat{k}}$)}}\Big(f\Big(v_{k}^{m_{k}^{b}}\Big|\cup_{(\hat{k},\hat{b})\neq(k,b)}\{e_{\hat{k}}^{\hat{b}}\}\Big)\Big)\Bigg)\\
		&\ge\alpha\sum_{k=1}^{K}\sum_{b=1}^{s_k}\Bigg(\E_{e_{\hat{k}}^{\hat{b}}\sim\emph{Multi($\p_{\hat{k}}$)}}\Big(f\Big(v_{k}^{m_{k}^{b}}\Big|\cup_{\hat{k}=1}^{K}\cup_{\hat{b}=1}^{B_{\hat{k}}}\{e_{\hat{k}}^{\hat{b}}\}\bigcup\big(\cup_{(\bar{k},\bar{b})<(k,b)}\{v_{\bar{k}}^{m_{\bar{k}}^{\bar{b}}}\}\big)\Big)\Big)\Bigg)\\
		&=\alpha\Bigg(\E_{e_{\hat{k}}^{\hat{b}}\sim\emph{Multi($\p_{\hat{k}}$)}}\Big(f\Big(\cup_{k=1}^{K}\cup_{b=1}^{s_k}\{v_{k}^{m_{k}^{b}}\}\Big|\cup_{\hat{k}=1}^{K}\cup_{\hat{b}=1}^{B_{\hat{k}}}\{e_{\hat{k}}^{\hat{b}}\}\Big)\Big)\Bigg)\\
		&\ge\alpha\Bigg(\E_{e_{\hat{k}}^{\hat{b}}\sim\emph{Multi($\p_{\hat{k}}$)}}\Big(f\Big(\cup_{k=1}^{K}\cup_{b=1}^{s_k}\{v_{k}^{m_{k}^{b}}\}\Big)-f\Big(\cup_{\hat{k}=1}^{K}\cup_{\hat{b}=1}^{B_{\hat{k}}}\{e_{\hat{k}}^{\hat{b}}\}\Big)\Big)\Bigg)\ \ \ \text{(Monotonicity)}\\
		&=\alpha\Big(f(S)-F(\p_{1},\dots,\p_{K})\Big),
	\end{aligned}
\end{equation}  where the first inequality follows from the $\alpha$-weakly DR-submodularity, namely, $f(v|A)\ge\alpha f(v|B)$ for any two subsets $A\subseteq B\subseteq\V$ and the partially ordered set $\{(\bar{k},\bar{b})<(k,b)\}=\{(\bar{k},\bar{b})\big| \bar{k}<k\ \text{or}\ \bar{b}<b\ \text{when}\ \bar{k}=k\}$.

As for the settings that the set function $f$ is $(\gamma,\beta)$-weakly monotone submodular,  we firstly can show that for any two elements $e_{1},e_{2}\in\V$ and the subset $B\subseteq\V$, we have that  
\begin{equation*}
\gamma\big(f(e_1| B\cup\{e_{2}\})+f(e_{2}|B)\big)=\gamma\big(f(B\cup\{e_{1},e_{2}\})-f(B)\big)\le f(e_{1}|B)+ f(e_{2}|B),
\end{equation*}such that 
\begin{equation}\label{equ_appendix2_3}
f(e_{1}|B)\ge \gamma f(e_1| B\cup\{e_{2}\})-(1-\gamma)f(e_{2}|B).
\end{equation}
From Eq.\eqref{equ_appendix2_3}, we can show that
\begin{equation}\label{equ_appendix2_4}
	\begin{aligned}
		&\left\langle\sum_{k=1}^{K}\frac{\one_{S\cap\V_{k}}}{B_{k}},\nabla F(\p_{1},\dots,\p_{K})\right\rangle\\
		&=\sum_{k=1}^{K}\sum_{b=1}^{s_k}\Bigg(\E_{e_{\hat{k}}^{\hat{b}}\sim\emph{Multi($\p_{\hat{k}}$)}}\Big(f\Big(v_{k}^{m_{k}^{b}}\Big|\cup_{(\hat{k},\hat{b})\neq(k,b)}\{e_{\hat{k}}^{\hat{b}}\}\Big)\Big)\Bigg)\\
	 &=\sum_{k=1}^{K}\sum_{b=1}^{s_k}\Bigg(\E_{e_{\hat{k}}^{\hat{b}}\sim\emph{Multi($\p_{\hat{k}}$)}}\Big(\gamma f\Big(v_{k}^{m_{k}^{b}}\Big|\cup_{\hat{k}=1}^{K}\cup_{\hat{b}=1}^{B_{\hat{k}}}\{e_{\hat{k}}^{\hat{b}}\}\Big)-(1-\gamma)f\Big(e_{k}^{b}\Big|\cup_{(\hat{k},\hat{b})\neq(k,b)}\{e_{\hat{k}}^{\hat{b}}\}\Big)\Big)\Big)\Bigg)\\
	\end{aligned}
\end{equation}
Then, from the $\gamma$-weakly submodularity,  
\begin{equation}\label{equ_appendix2_5}
	\begin{aligned}
	&\Bigg(\E_{e_{\hat{k}}^{\hat{b}}\sim\emph{Multi($\p_{\hat{k}}$)}}\Big( f\Big(v_{k}^{m_{k}^{b}}\Big|\cup_{\hat{k}=1}^{K}\cup_{\hat{b}=1}^{B_{\hat{k}}}\{e_{\hat{k}}^{\hat{b}}\}\Big)\Big)\Big)\Bigg)\\
	&\ge\gamma\E_{e_{\hat{k}}^{\hat{b}}\sim\emph{Multi($\p_{\hat{k}}$)}}\Big( f\Big(\cup_{k=1}^{K}\cup_{b=1}^{s_k}\{v_{k}^{m_{k}^{b}}\}\Big|\cup_{\hat{k}=1}^{K}\cup_{\hat{b}=1}^{B_{\hat{k}}}\{e_{\hat{k}}^{\hat{b}}\}\Big)\Big)\Big)\\
	&\ge\gamma\Big(f(S)-F(\p_{1},\dots,\p_{K})\Big),
\end{aligned}
\end{equation} where the final inequality follows from the monotonicity.

Furthermore,  from the $\beta$-weakly  upper submodularity,we also have, 
\begin{equation}\label{equ_appendix2_6}
	\begin{aligned}
		&\sum_{k=1}^{K}\sum_{b=1}^{s_k}\Bigg(\E_{e_{\hat{k}}^{\hat{b}}\sim\emph{Multi($\p_{\hat{k}}$)}}\Big(f\Big(e_{k}^{b}\Big|\cup_{(\hat{k},\hat{b})\neq(k,b)}\{e_{\hat{k}}^{\hat{b}}\}\Big)\Big)\Big)\Bigg)\\
		&\le\sum_{k=1}^{K}\sum_{b=1}^{B_{k}}\Bigg(\E_{e_{\hat{k}}^{\hat{b}}\sim\emph{Multi($\p_{\hat{k}}$)}}\Big(f\Big(e_{k}^{b}\Big|\cup_{(\hat{k},\hat{b})\neq(k,b)}\{e_{\hat{k}}^{\hat{b}}\}\Big)\Big)\Big)\Bigg)\\
		&\le\beta \E_{e_{\hat{k}}^{\hat{b}}\sim\emph{Multi($\p_{\hat{k}}$)}}\Big(f\Big(\cup_{\hat{k}=1}^{K}\cup_{\hat{b}=1}^{B_{\hat{k}}}\{e_{\hat{k}}^{\hat{b}}\}\Big)-f(\emptyset)\Big)\le\beta F(\p_{1},\dots,\p_{K}), 
	\end{aligned}
\end{equation}where the second inequality follows from the definition of  $\beta$-upper submodularity, namely, Eq.\eqref{equ_upper}.
	
Merging Eq.\eqref{equ_appendix2_6} and Eq.\eqref{equ_appendix2_5} into Eq.\eqref{equ_appendix2_4}, we have that
	\begin{equation*}
		\begin{aligned}
			&\left\langle\sum_{k=1}^{K}\frac{\one_{S\cap\V_{k}}}{B_{k}},\nabla F(\p_{1},\dots,\p_{K})\right\rangle\\
			&=\sum_{k=1}^{K}\sum_{b=1}^{s_k}\Bigg(\E_{e_{\hat{k}}^{\hat{b}}\sim\emph{Multi($\p_{\hat{k}}$)}}\Big(\gamma f\Big(v_{k}^{m_{k}^{b}}\Big|\cup_{\hat{k}=1}^{K}\cup_{\hat{b}=1}^{B_{\hat{k}}}\{e_{\hat{k}}^{\hat{b}}\}\Big)-(1-\gamma)f\Big(e_{k}^{b}\Big|\cup_{(\hat{k},\hat{b})\neq(k,b)}\{e_{\hat{k}}^{\hat{b}}\}\Big)\Big)\Big)\Bigg)\\
			&\ge \gamma^{2}(f(S)-F(\p_{1},\dots,\p_{K}))-(1-\gamma)\beta F(\p_{1},\dots,\p_{K})=\gamma^{2}f(S)-(\beta(1-\gamma)+\gamma^{2})F(\p_{1},\dots,\p_{K}).
		\end{aligned}
	\end{equation*}
\end{proof}

%% file: Proofs_for_Thm4.tex
In this section, we verify the Theorem~\ref{thm4}.
\begin{proof}
Firstly, from the fifth equality in Eq.\eqref{equ_appendix_3}, we know that 
\begin{equation*}
	\frac{\partial F}{\partial p^{m}_{k}}(\p_{1},\dots,\p_{K})
	=B_{k}\Bigg(\sum_{e_{\hat{k}}^{\hat{b}}\in\V_{\hat{k}}\cup\{\emptyset\},(\hat{k},\hat{b})\neq(k,1)}\Bigg(f\Big(v_{k}^{m}|\cup_{(\hat{k},\hat{b})\neq(k,1)}\{e_{\hat{k}}^{\hat{b}}\}\Big)\prod_{(\hat{k},\hat{b})\neq(k,1)}\text{Pr}(e_{\hat{k}}^{\hat{b}}|\p_{\hat{k}})\Bigg)\Bigg).
\end{equation*} 
Therefore,  if $k_{1}\neq k_{2}\in[K]$, for any $m_{1}\in[n_{k_1}]$ and $m_{2}\in[n_{k_2}]$,
\begin{small}
 \begin{equation*}
	\begin{aligned}
	&\frac{\partial^{2}F}{\partial p_{k_{1}}^{m_{1}}\partial p_{k_{2}}^{m_{2}}}(\p_{1},\dots,\p_{K})\\&=B_{k_{1}}\Bigg(\sum_{e_{\hat{k}}^{\hat{b}}\in\V_{\hat{k}}\cup\{\emptyset\},(\hat{k},\hat{b})\neq(k_{1},1)}\Bigg(f\Big(v_{k_{1}}^{m_{1}}|\cup_{(\hat{k},\hat{b})\neq(k_{1},1)}\{e_{\hat{k}}^{\hat{b}}\}\Big)\frac{\partial\Big(\prod_{(\hat{k},\hat{b})\neq(k_{1},1)}\text{Pr}(e_{\hat{k}}^{\hat{b}}|\p_{\hat{k}})\Big)}{\partial p_{k_{2}}^{m_{2}}}\Bigg)\Bigg)\\
	&=B_{k_{1}}\Bigg(\sum_{e_{\hat{k}}^{\hat{b}}\in\V_{\hat{k}}\cup\{\emptyset\},(\hat{k},\hat{b})\neq(k_{1},1)}\Bigg(f\Big(v_{k_{1}}^{m_{1}}|\cup_{(\hat{k},\hat{b})\neq(k_{1},1)}\{e_{\hat{k}}^{\hat{b}}\}\Big)\sum_{\hat{b}_{1}\in[B_{k_{2}}]}\frac{\partial\text{Pr}(e_{k_{2}}^{\hat{b}_{1}}|\p_{k_{2}})}{\partial p^{m_{2}}_{k_{2}}}\Big(\prod_{(\hat{k},\hat{b})\neq\{(k_{1},1),(k_{2},\hat{b}_{1})\}}\text{Pr}(e_{\hat{k}}^{\hat{b}}|\p_{\hat{k}})\Big)\Bigg)\Bigg)\\
	&=B_{k_{1}}\sum_{\hat{b}_{1}\in[B_{k_{2}}]}\Bigg(\sum_{e_{\hat{k}}^{\hat{b}}\in\V_{\hat{k}}\cup\{\emptyset\},(\hat{k},\hat{b})\neq(k_{1},1)}\Bigg(f\Big(v_{k_{1}}^{m_{1}}|\cup_{(\hat{k},\hat{b})\neq(k_{1},1)}\{e_{\hat{k}}^{\hat{b}}\}\Big)\frac{\partial\text{Pr}(e_{k_{2}}^{\hat{b}_{1}}|\p_{k_{2}})}{\partial p^{m_{2}}_{k_{2}}}\Big(\prod_{(\hat{k},\hat{b})\neq\{(k_{1},1),(k_{2},\hat{b}_{1})\}}\text{Pr}(e_{\hat{k}}^{\hat{b}}|\p_{\hat{k}})\Big)\Bigg)\Bigg)\\
	&=B_{k_{1}}B_{k_{2}}\Bigg(\sum_{(\hat{k},\hat{b})\neq\left\{(k_{1},1),(k_{2},1)\right\}}\Bigg(\sum_{e_{k_{2}}^{1}\in\V_{k_{2}}\cup\{\emptyset\}}f\Big(v_{k_{1}}^{m_{1}}|\cup_{(\hat{k},\hat{b})\neq(k_{1},1)}\{e_{\hat{k}}^{\hat{b}}\}\Big)\frac{\partial\text{Pr}(e_{k_{2}}^{1}|\p_{k_{2}})}{\partial p^{m_{2}}_{k_{2}}}\Big(\prod_{(\hat{k},\hat{b})\neq\{(k_{1},1),(k_{2},1)\}}\text{Pr}(e_{\hat{k}}^{\hat{b}}|\p_{\hat{k}})\Big)\Bigg)\Bigg).
	\end{aligned}
\end{equation*}   
\end{small}

Due to that $\text{Pr}(v_{\hat{k}}^{m} | \p_{\hat{k}}) = p^{m}_{\hat{k}}$ and $\text{Pr}(\emptyset | \p_{\hat{k}}) = 1 - \sum_{m=1}^{n_{\hat{k}}} p_{\hat{k}}^{m}$,  we can show that
\begin{equation*}
\begin{aligned}
   &\sum_{e_{k_{2}}^{1}\in\V_{k_{2}}\cup\{\emptyset\}}f\Big(v_{k_{1}}^{m_{1}}|\cup_{(\hat{k},\hat{b})\neq(k_{1},1)}\{e_{\hat{k}}^{\hat{b}}\}\Big)\frac{\partial\text{Pr}(e_{k_{2}}^{1}|\p_{k_{2}})}{\partial p^{m_{2}}_{k_{2}}}\\&=f\Big(v_{k_{1}}^{m_{1}}|\cup_{(\hat{k},\hat{b})\neq\{(k_{1},1),(k_{2},1)\}}\{e_{\hat{k}}^{\hat{b}}\}\bigcup\{v_{k_{2}}^{m_{2}}\}\Big)-f\Big(v_{k_{1}}^{m_{1}}|\cup_{(\hat{k},\hat{b})\neq\{(k_{1},1),(k_{2},1)\}}\{e_{\hat{k}}^{\hat{b}}\}\Big). 
\end{aligned}
\end{equation*}

Therefore, if we set $S=\cup_{(\hat{k},\hat{b})\neq\left\{(k_{1},1),(k_{2},1)\right\}}\{e_{\hat{k}}^{\hat{b}}\}$, we can show that
\begin{small}
   \begin{equation*}
	\begin{aligned}
		&\frac{\partial^{2}F}{\partial p_{k_{1}}^{m_{1}}\partial p_{k_{2}}^{m_{2}}}(\p_{1},\dots,\p_{K})\\
		&=B_{k_{1}}B_{k_{2}}\Bigg(\sum_{(\hat{k},\hat{b})\neq\left\{(k_{1},1),(k_{2},1)\right\}}\Bigg(\sum_{e_{k_{2}}^{1}\in\V_{k_{2}}\cup\{\emptyset\}}f\Big(v_{k_{1}}^{m_{1}}|\cup_{(\hat{k},\hat{b})\neq(k_{1},1)}\{e_{\hat{k}}^{\hat{b}}\}\Big)\frac{\partial\text{Pr}(e_{k_{2}}^{1}|\p_{k_{2}})}{\partial p^{m_{2}}_{k_{2}}}\Big(\prod_{(\hat{k},\hat{b})\neq\{(k_{1},1),(k_{2},1)\}}\text{Pr}(e_{\hat{k}}^{\hat{b}}|\p_{\hat{k}})\Big)\Bigg)\Bigg)\\
		&=B_{k_{1}}B_{k_{2}}\Bigg(\sum_{e_{\hat{k}}^{\hat{b}}\in\V_{\hat{k}}\cup\{\emptyset\},(\hat{k},\hat{b})\neq\left\{(k_{1},1),(k_{2},1)\right\}}\Bigg(f\Big(v_{k_{1}}^{m_{1}}|S\cup\{v_{k_{2}}^{m_{2}}\}\Big)-f\Big(v_{k_{1}}^{m_{1}}|S\Big)\Big(\prod_{(\hat{k},\hat{b})\neq\{(k_{1},1),(k_{2},1)\}}\text{Pr}(e_{\hat{k}}^{\hat{b}}|\p_{\hat{k}})\Big)\Bigg)\Bigg)\\
		&=B_{k_{1}}B_{k_{2}}\E_{e_{\hat{k}}^{\hat{b}}\sim\text{Multi}(\p_{\hat{k}})}\Bigg(f\Big(v_{k_{1}}^{m_{1}}|S\cup\{v_{k_{2}}^{m_{2}}\}\Big)-f\Big(v_{k_{1}}^{m_{1}}|S\Big)\Bigg),
	\end{aligned} 
\end{equation*} 
\end{small}where the final equality comes from that the subset $S=\cup_{(\hat{k},\hat{b})\neq\left\{(k_{1},1),(k_{2},1)\right\}}\{e_{\hat{k}}^{\hat{b}}\}$ is unrelated with the random elements $e_{k_{1}}^{1}$ and $e_{k_{2}}^{1}$.

When $k_{1}=k_{2}=k\in[K]$ and $B_{k}=1$, it is easy to verify that $\frac{\partial^{2}F}{\partial p_{k}^{m_{1}}\partial p_{k}^{m_{2}}}(\p_{1},\dots,\p_{K})=0$. As for $B_{k}\ge 2$, we have that
\begin{small}
   \begin{equation*}
	\begin{aligned}
		&\frac{\partial^{2}F}{\partial p_{k}^{m_{1}}\partial p_{k}^{m_{2}}}(\p_{1},\dots,\p_{K})\\&=B_{k}\Bigg(\sum_{e_{\hat{k}}^{\hat{b}}\in\V_{\hat{k}}\cup\{\emptyset\},(\hat{k},\hat{b})\neq(k,1)	}\Bigg(f\Big(v_{k}^{m_{1}}|\cup_{(\hat{k},\hat{b})\neq(k,1)}\{e_{\hat{k}}^{\hat{b}}\}\Big)\frac{\partial\Big(\prod_{(\hat{k},\hat{b})\neq(k,1)}\text{Pr}(e_{\hat{k}}^{\hat{b}}|\p_{\hat{k}})\Big)}{\partial p_{k}^{m_{2}}}\Bigg)\Bigg)\\
		&=B_{k}\Bigg(\sum_{e_{\hat{k}}^{\hat{b}}\in\V_{\hat{k}}\cup\{\emptyset\},(\hat{k},\hat{b})\neq(k,1)}\Bigg(f\Big(v_{k}^{m_{1}}|\cup_{(\hat{k},\hat{b})\neq(k,1)}\{e_{\hat{k}}^{\hat{b}}\}\Big)\sum_{2\le\hat{b}_{1}\le B_{k}}\frac{\partial\text{Pr}(e_{k}^{\hat{b}_{1}}|\p_{k})}{\partial p^{m_{2}}_{k}}\Big(\prod_{(\hat{k},\hat{b})\neq\{(k,1),(k,\hat{b}_{1})\}}\text{Pr}(e_{\hat{k}}^{\hat{b}}|\p_{\hat{k}})\Big)\Bigg)\Bigg)\\
		&=B_{k}\sum_{2\le\hat{b}_{1}\le B_{k}}\Bigg(\sum_{e_{\hat{k}}^{\hat{b}}\in\V_{\hat{k}}\cup\{\emptyset\},(\hat{k},\hat{b})\neq(k,1)}\Bigg(f\Big(v_{k}^{m_{1}}|\cup_{(\hat{k},\hat{b})\neq(k,1)}\{e_{\hat{k}}^{\hat{b}}\}\Big)\frac{\partial\text{Pr}(e_{k}^{\hat{b}_{1}}|\p_{k})}{\partial p^{m_{2}}_{k}}\Big(\prod_{(\hat{k},\hat{b})\neq\{(k,1),(k,\hat{b}_{1})\}}\text{Pr}(e_{\hat{k}}^{\hat{b}}|\p_{\hat{k}})\Big)\Bigg)\Bigg)\\
		&=B_{k}(B_{k}-1)\Bigg(\sum_{e_{\hat{k}}^{\hat{b}}\in\V_{\hat{k}}\cup\{\emptyset\},(\hat{k},\hat{b})\neq(k,1)}\Bigg(f\Big(v_{k}^{m_{1}}|\cup_{(\hat{k},\hat{b})\neq(k,1)}\{e_{\hat{k}}^{\hat{b}}\}\Big)\frac{\partial\text{Pr}(e_{k}^{2}|\p_{k})}{\partial p^{m_{2}}_{k}}\Big(\prod_{(\hat{k},\hat{b})\neq\{(k,1),(k,2)\}}\text{Pr}(e_{\hat{k}}^{\hat{b}}|\p_{\hat{k}})\Big)\Bigg)\Bigg)\\
		&=B_{k}(B_{k}-1)\Bigg(\sum_{(\hat{k},\hat{b})\neq\{(k,1),(k,2)\}}\Bigg(\sum_{e_{k}^{2}\in\V_{k}\cup\{\emptyset\}} f\Big(v_{k}^{m_{1}}|\cup_{(\hat{k},\hat{b})\neq(k,1)}\{e_{\hat{k}}^{\hat{b}}\}\Big)\frac{\partial\text{Pr}(e_{k}^{2}|\p_{k})}{\partial p^{m_{2}}_{k}}\Big(\prod_{(\hat{k},\hat{b})\neq\{(k,1),(k,2)\}}\text{Pr}(e_{\hat{k}}^{\hat{b}}|\p_{\hat{k}})\Big)\Bigg)\Bigg).
	\end{aligned}
\end{equation*} 
\end{small}
Also, because $\text{Pr}(v_{\hat{k}}^{m} | \p_{\hat{k}}) = p^{m}_{\hat{k}}$ and $\text{Pr}(\emptyset | \p_{\hat{k}}) = 1 - \sum_{m=1}^{n_{\hat{k}}} p_{\hat{k}}^{m}$,  we can show that
\begin{equation*}
\begin{aligned}
 &\sum_{e_{k}^{2}\in\V_{k}\cup\{\emptyset\}}f\Big(v_{k}^{m_{1}}|\cup_{(\hat{k},\hat{b})\neq(k,1)}\{e_{\hat{k}}^{\hat{b}}\}\Big)\frac{\partial\text{Pr}(e_{k}^{2}|\p_{k})}{\partial p^{m_{2}}_{k_{2}}}\\&=f\Big(v_{k}^{m_{1}}|\cup_{(\hat{k},\hat{b})\neq\{(k,1),(k,2)\}}\{e_{\hat{k}}^{\hat{b}}\}\bigcup\{v_{k}^{m_{2}}\}\Big)-f\Big(v_{k}^{m_{1}}|\cup_{(\hat{k},\hat{b})\neq\{(k,1),(k,2)\}}\{e_{\hat{k}}^{\hat{b}}\}\Big).   
\end{aligned}
\end{equation*}
As a result, if we set $S=\cup_{(\hat{k},\hat{b})\neq\left\{(k,1),(k,2)\right\}}\{e_{\hat{k}}^{\hat{b}}\}$, we have that
\begin{small}
   \begin{equation*}
	\begin{aligned}
		&\frac{\partial^{2}F}{\partial p_{k}^{m_{1}}\partial p_{k}^{m_{2}}}(\p_{1},\dots,\p_{K})\\
		&=B_{k}(B_{k}-1)\Bigg(\sum_{e_{\hat{k}}^{\hat{b}}\in\V_{\hat{k}}\cup\{\emptyset\},(\hat{k},\hat{b})\neq\{(k,1),(k,2)\}}\Bigg(\Big( f\Big(v_{k}^{m_{1}}|S\cup\{v_{k}^{m_{2}}\}\Big)-f\Big(v_{k}^{m_{1}}|S\Big)\Big)\Big(\prod_{(\hat{k},\hat{b})\neq\{(k,1),(k,2)\}}\text{Pr}(e_{\hat{k}}^{\hat{b}}|\p_{\hat{k}})\Big)\Bigg)\Bigg)\\
		=&(B_{k}^{2}-B_{k})\E_{e_{\hat{k}}^{\hat{b}}\sim\text{Multi}(\p_{\hat{k}})}\Bigg(f\Big(v_{k}^{m_{1}}|S\cup\{v_{k}^{m_{2}}\}\Big)-f\Big(v_{k}^{m_{1}}|S\Big)\Bigg),
	\end{aligned}
\end{equation*} 
\end{small}where the final equality comes from that the subset $S=\cup_{(\hat{k},\hat{b})\neq\left\{(k,1),(k,2)\right\}}\{e_{\hat{k}}^{\hat{b}}\}$ is unrelated with the random elements $e_{k}^{1}$ and $e_{k}^{2}$.\end{proof}

%% file: Proofs_for_Thm3.tex
In this section, we prove the Theorem~\ref{thm3}.
Firstly, we prove a lower bound about  $\Big\langle(\p_{1},\dots,\p_{K}),\nabla F(\p_{1},\dots,\p_{K})\Big\rangle$ for any  point $(\p_{1},\dots,\p_{K})\in\prod_{k=1}^{K}\Delta_{n_{k}}$, that is to say, 
\begin{theorem}\label{thm_appendix_1}
	When the set function $f:2^{\V}\rightarrow\R_{+}$ is $\alpha$-weakly DR-submodular, for any point $(\p_{1},\dots,\p_{K})\in\prod_{k=1}^{K}\Delta_{n_{k}}$, the following inequality holds:
	\begin{equation}\label{equ:thm_appendix_1}
	\Big\langle(\p_{1},\dots,\p_{K}),\nabla F(\p_{1},\dots,\p_{K})\Big\rangle\ge\frac{1}{\alpha}F(\p_{1},\dots,\p_{K}),
	\end{equation} where $\alpha\in(0,1]$. Similarly, when the set function $f:2^{\V}\rightarrow\R_{+}$ is $\beta$-weakly submodular from above, for any point $(\p_{1},\dots,\p_{K})\in\prod_{k=1}^{K}\Delta_{n_{k}}$, we also can infer that
	\begin{equation*}
		\Big\langle(\p_{1},\dots,\p_{K}),\nabla F(\p_{1},\dots,\p_{K})\Big\rangle\ge\beta F(\p_{1},\dots,\p_{K}),
	\end{equation*} where $\beta\ge 1$.
\end{theorem}
\begin{proof}
At first, we assume each $\p_{k}:=(p_{k}^{1},\dots,p_{k}^{n_{k}})$ for any $k\in[K]$. Then, from the first point of Theorem~\ref{thm1} and the fifth equality in Eq.\eqref{equ_appendix_3}, we have that 
\begin{equation*}
	\begin{aligned}
		\frac{\partial F}{\partial p_{k}^{m}}(\p_{1},\dots,\p_{K})&=B_{k}\Bigg(\E_{e_{\hat{k}}^{\hat{b}}\sim\emph{Multi($\p_{\hat{k}}$)}}\Big(f\Big(v_{k}^{m}\Big|\cup_{(\hat{k},\hat{b})\neq(k,1)}\{e_{\hat{k}}^{\hat{b}}\}\Big)\Big)\Bigg)\\&=B_{k}\Bigg(\sum_{e_{\hat{k}}^{\hat{b}}\in\V_{\hat{k}}\cup\{\emptyset\},(\hat{k},\hat{b})\neq(k,1)}\Bigg(f\Big(v_{k}^{m}|\cup_{(\hat{k},\hat{b})\neq(k,1)}\{e_{\hat{k}}^{\hat{b}}\}\Big)\prod_{(\hat{k},\hat{b})\neq(k,1)}\text{Pr}(e_{\hat{k}}^{\hat{b}}|\p_{\hat{k}})\Bigg)\Bigg)
	\end{aligned}
\end{equation*} 
As a result, we have that
\begin{equation*}
\begin{aligned}
	&\Big\langle(\p_{1},\dots,\p_{K}),\nabla F(\p_{1},\dots,\p_{K})\Big\rangle\\&=\sum_{k=1}^{K}\sum_{m=1}^{n_{k}}p_{k}^{m}\frac{\partial F}{\partial p_{k}^{m}}(\p_{1},\dots,\p_{K})\\
	&=\sum_{k=1}^{K}\sum_{m=1}^{n_{k}}p_{k}^{m}B_{k}\Bigg(\E_{e_{\hat{k}}^{\hat{b}}\sim\emph{Multi($\p_{\hat{k}}$)}}\Big(f\Big(v_{k}^{m}\Big|\cup_{(\hat{k},\hat{b})\neq(k,1)}\{e_{\hat{k}}^{\hat{b}}\}\Big)\Big)\Bigg)\\
	&=\sum_{k=1}^{K}\sum_{m=1}^{n_{k}}p_{k}^{m}B_{k}\Bigg(\sum_{e_{\hat{k}}^{\hat{b}}\in\V_{\hat{k}}\cup\{\emptyset\},(\hat{k},\hat{b})\neq(k,1)}\Bigg(f\Big(v_{k}^{m}|\cup_{(\hat{k},\hat{b})\neq(k,1)}\{e_{\hat{k}}^{\hat{b}}\}\Big)\prod_{(\hat{k},\hat{b})\neq(k,1)}\text{Pr}(e_{\hat{k}}^{\hat{b}}|\p_{\hat{k}})\Bigg)\Bigg)\\
    &=\sum_{k=1}^{K}\sum_{m=1}^{n_{k}}\text{Pr}(v_{k}^{m}|\p_{k})B_{k}\Bigg(\sum_{e_{\hat{k}}^{\hat{b}}\in\V_{\hat{k}}\cup\{\emptyset\},(\hat{k},\hat{b})\neq(k,1)}\Bigg(f\Big(v_{k}^{m}|\cup_{(\hat{k},\hat{b})\neq(k,1)}\{e_{\hat{k}}^{\hat{b}}\}\Big)\prod_{(\hat{k},\hat{b})\neq(k,1)}\text{Pr}(e_{\hat{k}}^{\hat{b}}|\p_{\hat{k}})\Bigg)\Bigg)\\
    &=\sum_{k=1}^{K}B_{k}\sum_{m=1}^{n_{k}}\text{Pr}(v_{k}^{m}|\p_{k})\Bigg(\sum_{e_{\hat{k}}^{\hat{b}}\in\V_{\hat{k}}\cup\{\emptyset\},(\hat{k},\hat{b})\neq(k,1)}\Bigg(f\Big(v_{k}^{m}|\cup_{(\hat{k},\hat{b})\neq(k,1)}\{e_{\hat{k}}^{\hat{b}}\}\Big)\prod_{(\hat{k},\hat{b})\neq(k,1)}\text{Pr}(e_{\hat{k}}^{\hat{b}}|\p_{\hat{k}})\Bigg)\Bigg)\\
    &=\sum_{k=1}^{K}B_{k}\Bigg(\sum_{m=1}^{n_{k}}\sum_{e_{\hat{k}}^{\hat{b}}\in\V_{\hat{k}}\cup\{\emptyset\},(\hat{k},\hat{b})\neq(k,1)}\Bigg(f\Big(v_{k}^{m}|\cup_{(\hat{k},\hat{b})\neq(k,1)}\{e_{\hat{k}}^{\hat{b}}\}\Big)\text{Pr}(v_{k}^{m}|\p_{k})\prod_{(\hat{k},\hat{b})\neq(k,1)}\text{Pr}(e_{\hat{k}}^{\hat{b}}|\p_{\hat{k}})\Bigg)\Bigg)\\
    &=\sum_{k=1}^{K}B_{k}\Bigg(\sum_{e_{\hat{k}}^{\hat{b}}\in\V_{\hat{k}}\cup\{\emptyset\},\forall \hat{b}\in[B_{\hat{k}}],\forall \hat{k}\in[K]}\Bigg(f\Big(e_{k}^{1}|\cup_{(\hat{k},\hat{b})\neq(k,1)}\{e_{\hat{k}}^{\hat{b}}\}\Big)\prod_{\hat{k}=1}^{K}\prod_{\hat{b}=1}^{B_{\hat{k}}}\text{Pr}(e_{\hat{k}}^{\hat{b}}|\p_{\hat{k}})\Bigg)\Bigg)\\
    &=\sum_{k=1}^{K}B_{k}\E_{e_{\hat{k}}^{\hat{b}}\sim\emph{Multi($\p_{\hat{k}}$)}}\Bigg(f\Big(e_{k}^{1}|\cup_{(\hat{k},\hat{b})\neq(k,1)}\{e_{\hat{k}}^{\hat{b}}\}\Big)\Bigg)\\
    &=\sum_{k=1}^{K}\sum_{b=1}^{n_{k}}\E_{e_{\hat{k}}^{\hat{b}}\sim\emph{Multi($\p_{\hat{k}}$)}}\Bigg(f\Big(e_{k}^{b}|\cup_{(\hat{k},\hat{b})\neq(k,b)}\{e_{\hat{k}}^{\hat{b}}\}\Big)\Bigg),
\end{aligned}
\end{equation*} where the fourth equality follows from $\text{Pr}(v_{k}^{m}|\p_{k})=p_{k}^{m}$; the seventh equality comes from $e_{k}^{1}\sim\text{Multi}(\p_{k})$ and $f(\emptyset|B)=0$ for any $B\subseteq\V$ as well as  the final equality follows from that each $e_{k}^{b}$ is independently drawn from the multinoulli distribution  $\text{Multi}(\p_{k})$.

If $f$ is $\beta$-weakly submodular from above, we can show that 
\begin{equation*}
\sum_{k=1}^{K}\sum_{b=1}^{n_{k}}f\Big(e_{k}^{b}|\cup_{(\hat{k},\hat{b})\neq(k,b)}\{e_{\hat{k}}^{\hat{b}}\}\Big)\le\beta\Bigg(f\Big(\cup_{\hat{k}=1}^{K}\cup_{\hat{b}=1}^{B_{\hat{k}}}\Big)-f(\emptyset)\Bigg)\le\beta f\Big(\cup_{\hat{k}=1}^{K}\cup_{\hat{b}=1}^{B_{\hat{k}}}\Big),
\end{equation*} where the final inequality follows from $f(\emptyset)\ge0$ (Note that we define $f:2^{\V}\rightarrow\R_{+}$).

As a result, when  $f$ is $\beta$-weakly submodular from above,  we have that
\begin{equation*}
	\begin{aligned}
		&\Big\langle(\p_{1},\dots,\p_{K}),\nabla F(\p_{1},\dots,\p_{K})\Big\rangle\\
		&=\sum_{k=1}^{K}\sum_{b=1}^{n_{k}}\E_{e_{\hat{k}}^{\hat{b}}\sim\emph{Multi($\p_{\hat{k}}$)}}\Bigg(f\Big(e_{k}^{b}|\cup_{(\hat{k},\hat{b})\neq(k,b)}\{e_{\hat{k}}^{\hat{b}}\}\Big)\Bigg)\\
		&\le \beta\E_{e_{\hat{k}}^{\hat{b}}\sim\text{Multi($\p_{\hat{k}}$)}}\Bigg( f\Big(\cup_{\hat{k}=1}^{K}\cup_{\hat{b}=1}^{B_{\hat{k}}}\Big)\Bigg)=\beta F(\p_{1},\dots,\p_{K}).
	\end{aligned}
\end{equation*}
Note that an $\alpha$-weakly DR-submodular function automatically satisfies the conditions for being $ \frac{1}{\alpha}$-weakly submodular from above. Thus, we get the Eq.\eqref{equ:thm_appendix_1}.
\end{proof}

Merging Theorem~\ref{thm_appendix_1} into Theorem~\ref{thm2}, we also can get that 
\begin{theorem}\label{thm_appendix_2}
	When the set function $f:2^{\V}\rightarrow\R_{+}$ is monotone and $\alpha$-weakly DR-submodular, for any subset $S$ within the partition constraint of problem~\eqref{equ_problem} and any point $(\p_{1},\dots,\p_{K})\in\prod_{k=1}^{K}\Delta_{n_{k}}$, the following inequality holds:
		\begin{equation*}
		\left\langle\sum_{k=1}^{K}\frac{1}{B_{k}}\one_{S\cap\V_{k}}-\Big(\p_{1},\dots,\p_{K}\Big),\nabla F\Big(\p_{1},\dots,\p_{K}\Big)\right\rangle\ge\alpha f(S)-(\alpha+\frac{1}{\alpha})F\Big(\p_{1},\dots,\p_{K}\Big).
	\end{equation*} Similarly, when the set function $f:2^{\V}\rightarrow\R_{+}$ is monotone and $(\gamma,\beta)$-weakly submodular, for any subset $S$ within the partition constraint of problem~\eqref{equ_problem} and any point $(\p_{1},\dots,\p_{K})\in\prod_{k=1}^{K}\Delta_{n_{k}}$, we also can infer that
	\begin{equation*}
	\left\langle\sum_{k=1}^{K}\frac{1}{B_{k}}\one_{S\cap\V_{k}}-\Big(\p_{1},\dots,\p_{K}\Big),\nabla F\Big(\p_{1},\dots,\p_{K}\Big)\right\rangle\ge\gamma^{2} f(S)-(\beta+\beta(1-\gamma)+\gamma^{2})F\Big(\p_{1},\dots,\p_{K}\Big).
\end{equation*} 
\end{theorem}
From the definition of stationary point, we know that if $(\p_{1},\dots,\p_{K})\in\prod_{k=1}^{K}\Delta_{n_{k}}$ is the stationary point over the domain $\prod_{k=1}^{K}\Delta_{n_{k}}$, for any point $\y\in\prod_{k=1}^{K}\Delta_{n_{k}}$, 
\begin{equation}\label{equ_appendix3_3}
	\left\langle\y-\Big(\p_{1},\dots,\p_{K}\Big),\nabla F\Big(\p_{1},\dots,\p_{K}\Big)\right\rangle\le0.
\end{equation}
Also, for any $S$ within the partition constraint of problem~\eqref{equ_problem}, we can easily show that $\sum_{k=1}^{K}\frac{1}{B_{k}}\one_{S\cap\V_{k}}\in\prod_{k=1}^{K}\Delta_{n_{k}}$ such that we know that, for any $S$ within the partition constraint of problem~\eqref{equ_problem},
\begin{equation*}
	\left\langle\sum_{k=1}^{K}\frac{1}{B_{k}}\one_{S\cap\V_{k}}-\Big(\p_{1},\dots,\p_{K}\Big),\nabla F\Big(\p_{1},\dots,\p_{K}\Big)\right\rangle\le0.
\end{equation*}
Therefore, when the set function $f:2^{\V}\rightarrow\R_{+}$ is monotone and $\alpha$-weakly DR-submodular, we have that 
$\alpha f(S)-(\alpha+\frac{1}{\alpha})F\Big(\p_{1},\dots,\p_{K}\Big)\le0$ such that $F\Big(\p_{1},\dots,\p_{K}\Big)\ge\frac{\alpha^{2}}{1+\alpha^{2}}f(S^{*})$ where $S^{*}$ is the optimal solution of problem~\eqref{equ_problem}. Similarly, when the set function $f:2^{\V}\rightarrow\R_{+}$ is monotone and $(\gamma,\beta)$-weakly submodular,, we have that $\gamma^{2} f(S)-(\beta+\beta(1-\gamma)+\gamma^{2})F\Big(\p_{1},\dots,\p_{K}\Big)\le 0$ such that $F\Big(\p_{1},\dots,\p_{K}\Big)\ge\Big(\frac{\gamma^{2}}{\beta+\beta(1-\gamma)+\gamma^{2}}\Big)f(S^{*})$.
\subsection{A Multinoulli Extension with \texorpdfstring{$1/2$}{}-Approximation Stationary Point}\label{appendix:bad_case_stationary_point}
In this subsection, we consider a special set function. At first, let  the universe set $U$ consist of $n-1$ elements $\{x_{1},\dots,x_{n-1}\}$ and $n-k$ elements $\{y_{1},\dots,y_{n-k}\}$, all of weight $1$, and $n-1$ elements $\{\epsilon_{1},\dots,\epsilon_{n-1}\}$ of arbitrarily
small weight $\epsilon>0$. Then, we define two different types of sets namely, $A_{i}$ and $A_{i+n}$ for any $i\in[n]$, that is to say ,
\begin{equation*}
	\begin{aligned}
		&A_{i}\triangleq\{\epsilon_{i}\}\text{\ for $1\le i\le n-1$,}\ \ \ \ \ \ \ \ \ \ \ \ \ \ \ \ \ \ \ A_{n}\triangleq\{x_{1},\dots,x_{n-1}\},\\
	&A_{n+i}\triangleq\{x_{i}\}\text{\ for $1\le i\le n-1$,}\ \ \ \ \ \ \ \ \ \ \ \ \ \ \ A_{2n}\triangleq\{y_{1},\dots,y_{n-k}\}.\\
	\end{aligned}
\end{equation*}
After that, we define a coverage set function $f:2^{\V}\rightarrow\R_{+}$ over these $2n$ distinct set $\{A_{1},\dots,A_{n},A_{n+1},\dots,A_{2n}\}$ where $\V=[2n]$. Specifically, for any subset $S\subseteq\V$, we have that
\begin{equation}\label{equ:cover_set}
	f(S)\triangleq\sum_{v\in\bigcup_{i\in S}A_{i}}w(v),
\end{equation} where $w(v)$ is the weight of element $v$.

Moreover, we consider a partition constraint that contains at most one of $\{A_{i},A_{n+i}\}$ for any $i\in[n]$. If we set $\V_{i}\triangleq\{i,i+n\}$ and $\V\triangleq\bigcup_{i\in[n]}\V_{i}\triangleq[2n]$, we naturally obtain the following coverage maximization problem:
\begin{equation}\label{equ:coverage}
	\max_{S\subseteq\V} f(S)\ \ \text{s.t.}\ |S\cap\V_{i}|\le 1\ \ \forall i\in[n].	
\end{equation}
Note that this problem~\eqref{equ:coverage} is a special case of the concern subset selection problem~\eqref{equ_problem}. From the result of \citep{filmus2012power}, we know that the coverage function $f$ in Eq.\eqref{equ:cover_set} is a submodular set function, namely, $\alpha=\beta=\gamma=1$.

A key feature of the coverage maximization problem~\eqref{equ:coverage} is that \citep{filmus2012power} found that the standard greedy~\citep{nemhauser1978analysis} will be stuck at a local maximum subset $\{A_{1},\dots,A_{n}\}$ where $S=[n]$ and $f(S)=(1+\epsilon)n$. In contrast, when $\epsilon$ is very small,  the optimal subset for the  problem~\eqref{equ:coverage}  is $\{A_{n+1},\dots,A_{2n}\}$ where $S=\{n+1,\dots,2n\}$ and $f(S)=2n-k-1$. Note that $\lim_{n\rightarrow\infty}\lim_{\epsilon\rightarrow0}\frac{(1+\epsilon)n}{2n-k-1}=\frac{1}{2}$.
Motivated by this finding of \citep{filmus2012power}, we also can show that the 
point $\one_{[n]}\triangleq(\underbrace{1,\dots,1}_{n},0,\dots,0)$ is a local stationary point of multinoulli extension(\texttt{ME}) of the set function $f$ in Eq.\eqref{equ:cover_set}. More specifically, we have the following theorem:
\begin{theorem}\label{thm:bad_stationary_point}
The point $\one_{[n]}$ is a stationary point of the multinoulli extension $F$ of the set function $f$ for problem~\eqref{equ:coverage}. Furthermore, we can show $\frac{F(\one_{[n]})}{f(\{n+1,\dots,2n\})}=\frac{(1+\epsilon)n}{2n-k-1}\rightarrow \frac{1}{2}$.
\end{theorem}
\begin{remark}
\cref{thm:bad_stationary_point} indicates that when the objective set function is submodular, namely, $c=\alpha=\gamma=\beta=1$, the approximation guarantees established in \cref{thm3} is \emph{tight}.
\end{remark}
\begin{proof}
From the problem~\eqref{equ:coverage}, we know that each $\p_{i}=(p_{i}^{1},p_{i}^{2})$. Then, from \cref{def_ME}, we can show that
\begin{equation}
F(\p_{1},\dots,\p_{n})=\sum_{i=1}^{n}\sum_{e_{i}^{1}\in\{i,i+n,\emptyset\}}\Bigg(f\big(\cup_{i=1}^{n}\{e_{i}^{1}\}\big)\prod_{i=1}^{n}\text{Pr}(e_{i}^{1}|\p_{i})\Bigg),
\end{equation} where $\text{Pr}(i| \p_i) = p_{i}^{1}$, $\text{Pr}(n+i| \p_i) = p_{i}^{2}$ and $\text{Pr}(\emptyset | \p_{i}) = 1-p_{i}^{1}-p_{i}^{2}$.

Moreover, from the part \textbf{1)} of \cref{thm1}, we know that,
\begin{equation*}
	\begin{aligned}
		&\frac{\partial F}{\partial (p_{i}^{1})}(\one_{[n]})=\epsilon\text{\ for $1\le i\le n-1$,}\ \ \ \ \ \ \ \ \ \ \ \ \ \ \ \ \ \ \ \frac{\partial F}{\partial (p_{n}^{1})}(\one_{[n]})=n-1,\\
		&\frac{\partial F}{\partial (p_{i}^{2})}(\one_{[n]})=0\text{\ for $1\le i\le n-1$,}\ \ \ \ \ \ \ \ \ \ \ \ \ \ \ \ \ \ \ \frac{\partial F}{\partial (p_{n}^{2})}(\one_{[n]})=n-k.\\
	\end{aligned}
\end{equation*}
As a result, for any $(\p_{1},\dots,\p_{n})\in\prod_{i=1}^{n}\Delta_{2}$, we have 
\begin{equation*}
	\begin{aligned}
	&\langle(\p_{1},\dots,\p_{n})-\one_{[n]},\nabla F(\one_{[n]})\rangle\\&=\sum_{i=1}^{n}\Bigg(p_{i}^{1}\frac{\partial F}{\partial (p_{i}^{1})}(\one_{[n]})+p_{i}^{2}\frac{\partial F}{\partial (p_{i}^{2})}(\one_{[n]})\Bigg)-(1+\epsilon)(n-1)\\
	&=\epsilon\sum_{i=1}^{n-1}p_{i}^{1}+(n-1)p_{n}^{1}+(n-k)p_{n}^{2}-(1+\epsilon)(n-1)\\
	&=\epsilon\sum_{i=1}^{n-1}(p_{i}^{1}-1)+(n-1)p_{n}^{1}+(n-k)p_{n}^{2}-(n-1)\\
	&=\epsilon\sum_{i=1}^{n-1}(p_{i}^{1}-1)+\big((n-k)-(n-1)\big)p_{n}^{2}-(n-1)(1-p_{n}^{1}-p_{n}^{2})\le 0.
	\end{aligned}
\end{equation*}

As a result, $\one_{[n]}$ is a stationary point of the \texttt{ME} $F$. Note that $\frac{F(\one_{[n]})}{F(\one_{\{n+1,\dots,2n\}})}=\frac{(1+\epsilon)(n-1)}{2n-1-k}\rightarrow\frac{1}{2}$ when $\epsilon\rightarrow 0 $ and $n\rightarrow\infty$, which implies that the approximation guarantee of \cref{thm3} is tight when $\alpha=\beta=\gamma=1$.
\end{proof}
\subsection{Proof of \texorpdfstring{\cref{thm4+}}{}}\label{append:proof4+}
In this subsection, we prove the \cref{thm4+}. Before that, we firstly prove the following lemma:
\begin{lemma}\label{lemma:C1}
	Given a monotone set function $f$, for its multinoulli extension $F$ introduced in \cref{def_ME}, if we consider an auxiliary function  $F^{a}:\prod_{k=1}^{K}\Delta_{n_{k}}\rightarrow\R_{+}$ whose gradient at each point $\x\in\prod_{k=1}^{K}\Delta_{n_{k}}$ is a weighted average of the gradient $\nabla F(z*\x)$, namely, $\nabla F^{a}(\x)\triangleq\int_{0}^{1} w(z)\nabla F_{t}(z*\boldsymbol{x})\mathrm{d}z$ where $w(z)$ is a positive weight function over $[0,1]$, we have that:
	
	\textbf{1):} When$f$ is $\alpha$-weakly DR-submodular and $w(z)=e^{\alpha(z-1)}$, for any vector $(\p_{1},\dots,\p_{K})\in\prod_{k=1}^{K}\Delta_{n_{k}}$ and any subset $S$ within the constraint of problem~\eqref{equ_problem}, then the following inequality holds:
	\begin{equation}\label{equ:boosting1}
		\begin{aligned}
		&\Big\langle\sum_{k=1}^{K}\frac{\one_{(S\cap\V_{k})}}{B_{k}}-(\p_{1},\dots,\p_{K}),\nabla F^{a}(\p_{1},\dots,\p_{K})\Big\rangle\\&=\left\langle\sum_{k=1}^{K}\frac{\one_{(S\cap\V_{k})}}{B_{k}}-(\p_{1},\dots,\p_{K}),\int_{0}^{1}e^{\alpha(z-1)}\nabla F\left(z*(\p_{1},\dots,\p_{K})\right)\mathrm{d}z\right\rangle\\&
		\ge\left(1-e^{-\alpha}\right)f(S)-F(\p_{1},\dots,\p_{K});
		\end{aligned}
	\end{equation}
	
\textbf{2):} When  $f$ is  $(\gamma,\beta)$-weakly submodular and $w(z)=e^{\phi(\gamma,\beta)(z-1)}$ where $\phi(\gamma,\beta)=\beta(1-\gamma)+\gamma^2$, for any policy vector $(\p_{1},\dots,\p_{K})\in\prod_{k=1}^{K}\Delta_{n_{k}}$ and any subset $S$ within the constraint of problem~\eqref{equ_problem},  then the following inequality holds:
\begin{equation}\label{equ:boosting2}
	\begin{aligned}
&\Big\langle\sum_{k=1}^{K}\frac{\one_{(S\cap\V_{k})}}{B_{k}}-(\p_{1},\dots,\p_{K}),\nabla F^{a}(\p_{1},\dots,\p_{K})\Big\rangle\\&=\left\langle\sum_{k=1}^{K}\frac{\one_{(S\cap\V_{k})}}{B_{k}}-(\p_{1},\dots,\p_{K}),\int_{0}^{1}e^{\phi(\gamma,\beta)(z-1)}\nabla F\left(z*(\p_{1},\dots,\p_{K})\right)\mathrm{d}z\right\rangle\\&
		\ge\left(\frac{\gamma^{2}(1-e^{-\phi(\gamma,\beta)})}{\phi(\gamma,\beta)}\right)f(S)-F(\p_{1},\dots,\p_{K}).
	\end{aligned}
\end{equation}
\end{lemma}
\begin{proof}
For \textbf{1)}: At first, we verify the following relationship:
\begin{equation*}
	\begin{aligned}
		&\left\langle (\p_{1},\dots,\p_{K}),\int_{0}^{1}e^{\alpha(z-1)}\nabla F\left(z*(\p_{1},\dots,\p_{K})\right)\mathrm{d}z\right\rangle=\int_{0}^{1}e^{\alpha(z-1)}\Big\langle (\p_{1},\dots,\p_{K}),\nabla F\left(z*(\p_{1},\dots,\p_{K})\right)\Big\rangle\mathrm{d}z\\
		&=\int_{0}^{1}e^{\alpha(z-1)}\frac{\mathrm{d}F\left(z*(\p_{1},\dots,\p_{K})\right)}{\mathrm{d}z}\mathrm{d}z=\int_{0}^{1}e^{\alpha(z-1)}\mathrm{d}F\left(z*(\p_{1},\dots,\p_{K})\right)\\
		&=e^{\alpha(z-1)}F\left(z*(\p_{1},\dots,\p_{K})\right)|_{z=0}^{z=1}-\int_{0}^{1}F\left(z*(\p_{1},\dots,\p_{K})\right)\mathrm{d}\left(e^{\alpha(z-1)}\right)\\
		&=F(\p_{1},\dots,\p_{K})-\alpha\int_{0}^{1}e^{\alpha(z-1)}F\left(z*(\p_{1},\dots,\p_{K})\right)\mathrm{d}z,
	\end{aligned}
\end{equation*} where the final equality follows from $F(\mathbf{0})=0$.

Then, we have the following inequality:
\begin{equation*}
\begin{aligned}
	&\left\langle \sum_{k=1}^{K}\frac{\one_{(S\cap\V_{k})}}{B_{k}},\int_{0}^{1}e^{\alpha(z-1)}\nabla F\left(z*(\p_{1},\dots,\p_{K})\right)\mathrm{d}z\right\rangle=\int_{0}^{1}e^{\alpha(z-1)}\Big\langle \sum_{k=1}^{K}\frac{\one_{(S\cap\V_{k})}}{B_{k}},\nabla F\left(z*(\p_{1},\dots,\p_{K})\right)\Big\rangle\mathrm{d}z\\
	&\ge\alpha\int_{0}^{1}e^{\alpha(z-1)}\Big(f(S)-F\big(z*(\p_{1},\dots,\p_{K})\big)\Big)\mathrm{d}z=\Big(\alpha\int_{0}^{1}e^{\alpha(z-1)}\mathrm{d}z\Big)f(S)-\alpha\int_{0}^{1}e^{\alpha(z-1)}F\left(z*(\p_{1},\dots,\p_{K})\right)\mathrm{d}z\\
	&=\Big(1-e^{-\alpha}\Big)f(S)-\alpha\int_{0}^{1}e^{\alpha(z-1)}F\left(z*(\p_{1},\dots,\p_{K})\right)\mathrm{d}z,
\end{aligned}
\end{equation*}where the first inequality follows from \cref{thm3}.

As a result, we have 
	\begin{equation*}
		\left\langle\sum_{k=1}^{K}\frac{\one_{(S\cap\V_{k})}}{B_{k}}-(\p_{1},\dots,\p_{K}),\int_{0}^{1}e^{\alpha(z-1)}\nabla F\left(z*(\p_{1},\dots,\p_{K})\right)\mathrm{d}z\right\rangle
		\ge\left(1-e^{-\alpha}\right)f(S)-F(\p_{1},\dots,\p_{K}).
\end{equation*}

As for \textbf{ii):} Similarly, when  $f$ is  monotone $(\gamma,\beta)$-weakly submodular and $w(z)=e^{\phi(\gamma,\beta)(z-1)}$ where $\phi(\gamma,\beta)=\beta(1-\gamma)+\gamma^2$, we have 
\begin{equation*}
	\begin{aligned}
		&\left\langle (\p_{1},\dots,\p_{K}),\int_{0}^{1}e^{\phi(\gamma,\beta)(z-1)}\nabla F\left(z*(\p_{1},\dots,\p_{K})\right)\mathrm{d}z\right\rangle=\int_{0}^{1}e^{\phi(\gamma,\beta)(z-1)}\Big\langle (\p_{1},\dots,\p_{K}),\nabla F\left(z*(\p_{1},\dots,\p_{K})\right)\Big\rangle\mathrm{d}z\\
		&=\int_{0}^{1}e^{\phi(\gamma,\beta)(z-1)}\frac{\mathrm{d}F\left(z*(\p_{1},\dots,\p_{K})\right)}{\mathrm{d}z}\mathrm{d}z=\int_{0}^{1}e^{\phi(\gamma,\beta)(z-1)}\mathrm{d}F\left(z*(\p_{1},\dots,\p_{K})\right)\\
		&=e^{\phi(\gamma,\beta)(z-1)}F\left(z*(\p_{1},\dots,\p_{K})\right)|_{z=0}^{z=1}-\phi(\gamma,\beta)\int_{0}^{1}F\left(z*(\p_{1},\dots,\p_{K})\right)\mathrm{d}\left(e^{\alpha(z-1)}\right)\\
		&=F(\p_{1},\dots,\p_{K})-\phi(\gamma,\beta)\int_{0}^{1}e^{\phi(\gamma,\beta)(z-1)}F\left(z*(\p_{1},\dots,\p_{K})\right)\mathrm{d}z,
	\end{aligned}
\end{equation*} where the final equality follows from $F(\mathbf{0})=0$.

Then, we have the following inequality:
\begin{equation*}
\begin{aligned}
	&\left\langle \sum_{k=1}^{K}\frac{\one_{(S\cap\V_{k})}}{B_{k}},\int_{0}^{1}e^{\phi(\gamma,\beta)(z-1)}\nabla F\left(z*(\p_{1},\dots,\p_{K})\right)\mathrm{d}z\right\rangle=\int_{0}^{1}e^{\phi(\gamma,\beta)(z-1)}\Big\langle \sum_{k=1}^{K}\frac{\one_{(S\cap\V_{k})}}{B_{k}},\nabla F\left(z*(\p_{1},\dots,\p_{K})\right)\Big\rangle\mathrm{d}z\\
	&\ge\int_{0}^{1}e^{\phi(\gamma,\beta)(z-1)}\Big(\gamma^{2}f(S)-\phi(\gamma,\beta)F\big(z*(\p_{1},\dots,\p_{K})\big)\Big)\mathrm{d}z\\&=\Big(\gamma^{2}\int_{0}^{1}e^{\phi(\gamma,\beta)(z-1)}\mathrm{d}z\Big)f(S)-\phi(\gamma,\beta)\int_{0}^{1}e^{\phi(\gamma,\beta)(z-1)}F\left(z*(\p_{1},\dots,\p_{K})\right)\mathrm{d}z\\
	&=\Big(\frac{\gamma^{2}\left(1-e^{-\phi(\gamma,\beta)}\right)}{\phi(\gamma,\beta)}\Big)f(S)-\phi(\gamma,\beta)\int_{0}^{1}e^{\phi(\gamma,\beta)(z-1)}F\left(z*(\p_{1},\dots,\p_{K})\right)\mathrm{d}z,
\end{aligned}
\end{equation*}where the first inequality follows from \cref{thm3}.

As a result, we have 
\begin{equation*}
\left\langle\sum_{k=1}^{K}\frac{\one_{(S\cap\V_{k})}}{B_{k}}-(\p_{1},\dots,\p_{K}),\int_{0}^{1}e^{\phi(\gamma,\beta)(z-1)}\nabla F\left(z*(\p_{1},\dots,\p_{K})\right)\mathrm{d}z\right\rangle
		\ge\left(\frac{\gamma^{2}(1-e^{-\phi(\gamma,\beta)})}{\phi(\gamma,\beta)}\right)f(S)-F(\p_{1},\dots,\p_{K}).
\end{equation*}
\end{proof}

From \cref{def:stationary}, we can show that, when the vector $(\p^{a}_{1},\dots,\p^{a}_{K})$ is the stationary point of the surrogate objective $F_{a}$ over domain $\prod_{k=1}^{K}\Delta_{n_{k}}$ and $S$ satisfies the constraint of problem~\eqref{equ_problem}, due to $\sum_{k=1}^{K}\frac{\one_{(S\cap\V_{k})}}{B_{k}}\in\prod_{k=1}^{K}\Delta_{n_{k}}$ , the following inequality holds:
\begin{equation*}
	\left\langle\sum_{k=1}^{K}\frac{\one_{(S\cap\V_{k})}}{B_{k}}-(\p^{a}_{1},\dots,\p^{a}_{K}),\int_{0}^{1}w(z)\nabla F\left(z*(\p^{a}_{1},\dots,\p^{a}_{K})\right)\mathrm{d}z\right\rangle\le0,
\end{equation*} where $w(z)$ is the corresponding weight function.

then, following the part \textbf{1)} \cref{lemma:C1}, we can show that when$f$ is monotone $\alpha$-weakly DR-submodular and $w(z)=e^{\alpha(z-1)}$, for any stationary point $(\p^{a}_{1},\dots,\p^{a}_{K})\in\prod_{k=1}^{K}\Delta_{n_{K}}$,
\begin{equation*}
F(\p^{a}_{1},\dots,\p^{a}_{K})\ge\big(1-e^{-\alpha}\big)f(S),
\end{equation*} where $S$ is any subset within the constraint of problem~\eqref{equ_problem}.

So we get the result of  part \textbf{1)} of \cref{thm4+}. Similarly, from \textbf{2)} of \cref{lemma:C1}, we also can achieve the part \textbf{2)} of \cref{thm4+}, i.e., 
    \begin{equation*}
F(\p^{a}_{1},\dots,\p^{a}_{K})\ge\big(\frac{\gamma^{2}(1-e^{-\phi(\gamma,\beta)})}{\phi(\gamma,\beta)}\big)f(S^{*}), \end{equation*} where $S^{*}$ is the optimal solution of problem~\eqref{equ_problem}.

%% file: Proof_for_Rounding.tex
\subsection{Proof of \texorpdfstring{\cref{thm5+}}{}}\label{append:rounding_without_replacement}
In this subsection, we aim to prove \cref{thm5+}.

At first, due to the normalization in Lines 2-6 of \cref{alg_without_replacement}, we assume that each $\|\p_{k}\|_{1}=1,\forall k\in[K]$ in this subsection. Subsequently, under this assumption, we show that the subset obtained from the sampling-without-replacement operation, i.e., Lines 7-8 of \cref{alg_without_replacement},
is \emph{lossless} in expectation. Before that, we provide a detailed description of the sampling-without-replacement process of Lines 7-8 of \cref{alg_without_replacement}, as shown in \cref{alg_sampling_without_replacement}.
\begin{algorithm}[ht]
	\caption{Sampling Without Replacement}\label{alg_sampling_without_replacement}
	{\bf Input:} Point $(\p_{1},\dots,\p_{K})\in\prod_{k=1}^{K}\Delta_{n_{k}}$ where $\p_{k}=(p_{k}^{1},\dots,p_{k}^{n_{k}})$ and $\sum_{m=1}^{n_{k}}p_{k}^{m}=1$ for any $k\in[K]$, Partition $(\V_{1},\dots,\V_{K})$ of set $\V$ where $\V_{k}\triangleq\{v_{k}^{1},\dots,v_{k}^{n_{k}}\}$ for any $k\in[K]$, Budget set $\{B_{1},\dots,B_{K}\}$
	\begin{algorithmic}[1]
		\STATE \textbf{Initialize} $S=\emptyset$;
		\FOR{$\hat{k}=1,\dots,K$}
        \STATE Compute $B^{\hat{k}}_{min}\triangleq\min(\|\p_{\hat{k}}\|_{0},B_{\hat{k}})$;
		\STATE Set $P=0$;
		\FOR{$\hat{b}=1,\dots,B^{\hat{k}}_{min}$}
		\IF{$\hat{b}=1$}
		\STATE Sampling a number $N_{1}$ from $[n_{k}]$ according to the probability Pr$(N_{1}=m)=p_{\hat{k}}^{m}$ for any $m\in[n_{\hat{k}}]$;
	   \ELSE
		\STATE Sampling a number $N_{\hat{b}}$ from $[n_{k}]-\{N_{1},\dots,N_{\hat{b}-1}\}$ according to the probability $\text{Pr}(N_{\hat{b}}=m)\triangleq\frac{p_{\hat{k}}^{m}}{1-P}$ for any $m\in[n_{\hat{k}}]\setminus\{N_{1},\dots,N_{\hat{b}-1}\}$;
		\ENDIF
		\ENDFOR
        \STATE Compute $\mathcal{A}_{\hat{k}}\triangleq\cup_{\hat{b}=1}^{B_{min}^{\hat{k}}}\{v_{\hat{k}}^{N_{\hat{b}}}\}$;
		\ENDFOR
		\STATE \textbf{Return} $S\triangleq\cup_{\hat{k}=1}^{K}\mathcal{A}_{\hat{k}}$;
	\end{algorithmic}
\end{algorithm}

The core of Algorithm~\ref{alg_sampling_without_replacement} lies in the Line 9, that is, instead of independently sampling according to each probability vector $\p_{\hat{k}}$, we take into account the elements previously selected within the same community, namely, $\{v_{\hat{k}}^{N_{1}},\dots,v_{\hat{k}}^{N_{\hat{b}-1}}\}$.  Next, we show the sampling-without-replacement process of \cref{alg_sampling_without_replacement} is \emph{lossless} with respect to the function value of our introduced multinoulli extension in expectation, that is to say,
\begin{lemma}\label{thm_appendix_sampling_without_replacement}
For any multinoulli priors $(\p_{1},\dots,\p_{K})\in\prod_{k=1}^{K}\Delta_{n_{k}}$ satisfying $\|\p_{k}\|_{1}=1,\forall k\in[K]$, if the set function $f$ is monotone, we can show that the subset $S$ returned by \cref{alg_sampling_without_replacement} satisfies:
\begin{itemize}
\item $|S\cap\V_{k}|\equiv B^{k}_{min}$ where $B^{k}_{min}\triangleq\min(\|\p_{k}\|_{0},B_{k})$ for any $k\in[K]$;
\item   $E\Big(f(S)\Big)\ge F(\p_{1},\dots,\p_{K})$ where $F$ is the Multinoulli Extension of $f$.
\end{itemize}
\end{lemma}
\begin{proof}
The first point of \cref{thm_appendix_sampling_without_replacement} is easy to verify. We mainly focus on the second point. 

At first, fixing a $\hat{k}\in[K]$ and any subset $\tilde{S}\subseteq\V$, we show the following inequality holds:
\begin{equation}\label{equ_rounding_1}
	\E_{\mathcal{A}_{\hat{k}}}\Bigg(f\Big(\tilde{S}\cup\{\mathcal{A}_{\hat{k}}\}\Big)\Bigg)\ge\E_{e_{\hat{k}}^{\hat{b}}\sim\text{Multi}(\p_{\hat{k}}),\forall\hat{b}\in[B_{\hat{k}}]}\Bigg(f\Big(\tilde{S}\cup\{\cup_{\hat{b}=1}^{B_{\hat{k}}}e_{\hat{k}}^{\hat{b}}\}\Big)\Bigg),
\end{equation} where the symbol $e_{\hat{k}}^{\hat{b}}\sim\text{Multi}(\p_{\hat{k}}),\forall\hat{b}\in[B_{\hat{k}}]$ implies that each element $e_{\hat{k}}^{\hat{b}}$ is independently drawn from $\V_{\hat{k}}$ according to the distribution $\text{Multi}(\p_{\hat{k}})$.

Note that, when $\|\p_{\hat{k}}\|_{0}<B_{\hat{k}}$, from Lines 6-10 of \cref{alg_sampling_without_replacement}, we can know that the set $\mathcal{A}_{\hat{k}}$ contains all the elements with nonzero probabilities in vector $\p_{\hat{k}}$. As a result, for any element $e_{\hat{k}}^{\hat{b}}\sim\text{Multi}(\p_{\hat{k}}),\forall \hat{b}\in[B_{\hat{k}}]$, we have $e_{\hat{k}}^{\hat{b}}\in\mathcal{A}_{\hat{k}}$ such that $\cup_{\hat{b}=1}^{B_{\hat{k}}}\{e_{\hat{k}}^{\hat{b}}\}\subseteq\mathcal{A}_{\hat{k}}$. Naturally, from the monotonicity of $f$, we get the result of Eq.\eqref{equ_rounding_1}. Next, we investigate the case where $B^{\hat{k}}_{min}\triangleq\min(\|\p_{\hat{k}}\|_{0},B_{\hat{k}})=B_{\hat{k}}$. In this scenario, in order to prove the result of Eq.\eqref{equ_rounding_1}. We first show that
\begin{equation}\label{equ_rounding_2}
\E\Bigg(f\Big(\tilde{S}\cup\{\mathcal{A}_{\hat{k}}\}\Big)\Big|v_{\hat{k}}^{1},\dots,v_{\hat{k}}^{B_{\hat{k}}-1}\Bigg)\ge\E_{e_{\hat{k}}^{B_{\hat{k}}}\sim\text{Multi}(\p_{\hat{k}})}\Bigg(f\Big(\tilde{S}\cup\{\cup_{\hat{b}=1}^{B_{\hat{k}}-1}v_{\hat{k}}^{N_{\hat{b}}}\}\cup\{e_{\hat{k}}^{B_{\hat{k}}}\}\Big)\Big|v_{\hat{k}}^{1},\dots,v_{\hat{k}}^{B_{\hat{k}}-1}\Bigg)    \end{equation}

Next, we show the verification of Eq.\eqref{equ_rounding_2}, that is,
\begin{equation*}
	\begin{aligned}
&\E\Bigg(f\Big(\tilde{S}\cup\{\mathcal{A}_{\hat{k}}\}\Big)\Big|v_{\hat{k}}^{1},\dots,v_{\hat{k}}^{B_{\hat{k}}-1}\Bigg)=\E\Bigg(f\Big(\tilde{S}\cup\{\cup_{\hat{b}=1}^{B_{\hat{k}}-1}v_{\hat{k}}^{N_{\hat{b}}}\}\cup\{v_{\hat{k}}^{B_{\hat{k}}}\}\Big)\Big|v_{\hat{k}}^{1},\dots,v_{\hat{k}}^{B_{\hat{k}}-1}\Bigg)\\
&=\sum_{m\in[n_{\hat{k}}]\setminus\{N_{1},\dots,N_{B_{\hat{k}}-1}\}}\frac{p_{\hat{k}}^{m}}{1-P}\E\Bigg(f\Big(\tilde{S}\cup\{\cup_{\hat{b}=1}^{B_{\hat{k}}-1}v_{\hat{k}}^{N_{\hat{b}}}\}\cup\{v_{\hat{k}}^{m}\}\Big)\Big|v_{\hat{k}}^{1},\dots,v_{\hat{k}}^{B_{\hat{k}}-1}\Bigg)\\
&=\frac{P}{1-P}\sum_{m\in[n_{\hat{k}}]\setminus\{N_{1},\dots,N_{B_{\hat{k}}-1}\}}p_{\hat{k}}^{m}\E\Bigg(f\Big(\tilde{S}\cup\{\cup_{\hat{b}=1}^{B_{\hat{k}}-1}v_{\hat{k}}^{N_{\hat{b}}}\}\cup\{v_{\hat{k}}^{m}\}\Big)\Big|v_{\hat{k}}^{1},\dots,v_{\hat{k}}^{B_{\hat{k}}-1}\Bigg)\\&+\sum_{m\in[n_{\hat{k}}]\setminus\{N_{1},\dots,N_{B_{\hat{k}}-1}\}}p_{\hat{k}}^{m}\E\Bigg(f\Big(\tilde{S}\cup\{\cup_{\hat{b}=1}^{B_{\hat{k}}-1}v_{\hat{k}}^{N_{\hat{b}}}\}\cup\{v_{\hat{k}}^{m}\}\Big)\Big|v_{\hat{k}}^{1},\dots,v_{\hat{k}}^{B_{\hat{k}}-1}\Bigg)\\
&\ge\frac{P}{1-P}\left(\sum_{m\in[n_{\hat{k}}]\setminus\{N_{1},\dots,N_{B_{\hat{k}}-1}\}}p_{\hat{k}}^{m}\right)\E\Bigg(f\Big(\tilde{S}\cup\{\cup_{\hat{b}=1}^{B_{\hat{k}}-1}v_{\hat{k}}^{N_{\hat{b}}}\}\}\Big)\Big|v_{\hat{k}}^{1},\dots,v_{\hat{k}}^{B_{\hat{k}}-1}\Bigg)\\&+\sum_{m\in[n_{\hat{k}}]\setminus\{N_{1},\dots,N_{B_{\hat{k}}-1}\}}\E\Bigg(p_{\hat{k}}^{m}f\Big(\tilde{S}\cup\{\cup_{\hat{b}=1}^{B_{\hat{k}}-1}v_{\hat{k}}^{N_{\hat{b}}}\}\cup\{v_{\hat{k}}^{m}\}\Big)\Big|v_{\hat{k}}^{1},\dots,v_{\hat{k}}^{B_{\hat{k}}-1}\Bigg)\\
&=\E\Bigg(Pf\Big(\tilde{S}\cup\{\cup_{\hat{b}=1}^{B_{\hat{k}}-1}v_{\hat{k}}^{N_{\hat{b}}}\}\}\Big)\Big|v_{\hat{k}}^{1},\dots,v_{\hat{k}}^{B_{\hat{k}}-1}\Bigg)+\sum_{m\in[n_{\hat{k}}]\setminus\{N_{1},\dots,N_{B_{\hat{k}}-1}\}}p_{\hat{k}}^{m}\E\Bigg(f\Big(\tilde{S}\cup\{\cup_{\hat{b}=1}^{B_{\hat{k}}-1}v_{\hat{k}}^{N_{\hat{b}}}\}\cup\{v_{\hat{k}}^{m}\}\Big)\Big|v_{\hat{k}}^{1},\dots,v_{\hat{k}}^{B_{\hat{k}}-1}\Bigg)\\
&=\sum_{m\in\{N_{1},\dots,N_{B_{\hat{k}}-1}\}}\E\Bigg(p_{\hat{k}}^{m}f\Big(\tilde{S}\cup\{\cup_{\hat{b}=1}^{B_{\hat{k}}-1}v_{\hat{k}}^{N_{\hat{b}}}\}\}\Big)\Big|v_{\hat{k}}^{1},\dots,v_{\hat{k}}^{B_{\hat{k}}-1}\Bigg)\\&+\sum_{m\in[n_{\hat{k}}]\setminus\{N_{1},\dots,N_{B_{\hat{k}}-1}\}}p_{\hat{k}}^{m}\E\Bigg(f\Big(\tilde{S}\cup\{\cup_{\hat{b}=1}^{B_{\hat{k}}-1}v_{\hat{k}}^{N_{\hat{b}}}\}\cup\{v_{\hat{k}}^{m}\}\Big)\Big|v_{\hat{k}}^{1},\dots,v_{\hat{k}}^{B_{\hat{k}}-1}\Bigg)\\
&=\sum_{m\in\{N_{1},\dots,N_{B_{\hat{k}}-1}\}}\E\Bigg(p_{\hat{k}}^{m}f\Big(\tilde{S}\cup\{\cup_{\hat{b}=1}^{B_{\hat{k}}-1}v_{\hat{k}}^{N_{\hat{b}}}\}\}\cup\{v_{\hat{k}}^{m}\}\Big)\Big|v_{\hat{k}}^{1},\dots,v_{\hat{k}}^{B_{\hat{k}}-1}\Bigg)\\&+\sum_{m\in[n_{\hat{k}}]\setminus\{N_{1},\dots,N_{B_{\hat{k}}-1}\}}p_{\hat{k}}^{m}\E\Bigg(f\Big(\tilde{S}\cup\{\cup_{\hat{b}=1}^{B_{\hat{k}}-1}v_{\hat{k}}^{N_{\hat{b}}}\}\cup\{v_{\hat{k}}^{m}\}\Big)\Big|v_{\hat{k}}^{1},\dots,v_{\hat{k}}^{B_{\hat{k}}-1}\Bigg)\\
&=\sum_{m\in[n_{\hat{k}}]}\E\Bigg(p_{\hat{k}}^{m}f\Big(\tilde{S}\cup\{\cup_{\hat{b}=1}^{B_{\hat{k}}-1}v_{\hat{k}}^{N_{\hat{b}}}\}\}\cup\{v_{\hat{k}}^{m}\}\Big)\Big|v_{\hat{k}}^{1},\dots,v_{\hat{k}}^{B_{\hat{k}}-1}\Bigg)\\&=\E_{e_{\hat{k}}^{B_{\hat{k}}}\sim\text{Multi}(\p_{\hat{k}})}\Bigg(f\Big(\tilde{S}\cup\{\cup_{\hat{b}=1}^{B_{\hat{k}}-1}v_{\hat{k}}^{N_{\hat{b}}}\}\cup\{e_{\hat{k}}^{B_{\hat{k}}}\}\Big)\Big|v_{\hat{k}}^{1},\dots,v_{\hat{k}}^{B_{\hat{k}}-1}\Bigg),
	\end{aligned}
\end{equation*}  where the first equality follows from $B^{\hat{k}}_{min}=B_{\hat{k}}$, the second equality comes from $P\triangleq\sum_{m\in\{N_{1},\dots,N_{N_{B_{\hat{k}}}-1}\}}p_{\hat{k}}^{m}$, the first inequality comes from the monotonicity of $f$, the fifth  equality follows from $P\triangleq\sum_{m\in\{N_{1},\dots,N_{N_{B_{\hat{k}}}-1}\}}p_{\hat{k}}^{m}$, the sixth equality comes from $\tilde{S}\cup\{\cup_{\hat{b}=1}^{B_{\hat{k}}-1}v_{\hat{k}}^{N_{\hat{b}}}\}\}\cup\{v_{\hat{k}}^{m}\}=\tilde{S}\cup\{\cup_{\hat{b}=1}^{B_{\hat{k}}-1}v_{\hat{k}}^{N_{\hat{b}}}\}\}$ when $m\in\{N_{1},\dots,N_{B_{\hat{k}}-1}\}$ and the final equality follows from that the element $e_{\hat{k}}^{\hat{b}}$ is randomly chosen according to the distribution $\text{Multi}(\p_{\hat{k}})$ and $\|\p_{\hat{k}}\|_{1}=1$. 

With Eq.\eqref{equ_rounding_2}, we then fix a random element $e_{\hat{k}}^{B_{\hat{k}}}\in\V_{\hat{k}}\cup\{\emptyset\}$ and re-use the result of Eq.\eqref{equ_rounding_2} to the new objective $\E\Bigg(f\Big(\big(\tilde{S}\cup\{e_{\hat{k}}^{B_{\hat{k}}}\}\big)\cup\{\cup_{\hat{b}=1}^{B_{\hat{k}}-2}v_{\hat{k}}^{N_{\hat{b}}}\}\cup\{v_{\hat{k}}^{N_{B_{\hat{k}}-1}}\}\Big)\Big|v_{\hat{k}}^{1},\dots,v_{\hat{k}}^{B_{\hat{k}}-2}\Bigg)$ such that we can have 
\begin{equation*}
\begin{aligned}
&\E\Bigg(f\Big(\big(\tilde{S}\cup\{e_{\hat{k}}^{B_{\hat{k}}}\}\big)\cup\{\cup_{\hat{b}=1}^{B_{\hat{k}}-2}v_{\hat{k}}^{N_{\hat{b}}}\}\cup\{v_{\hat{k}}^{N_{B_{\hat{k}}-1}}\}\Big)\Big|v_{\hat{k}}^{1},\dots,v_{\hat{k}}^{B_{\hat{k}}-2}\Bigg)\\&\ge\E_{e_{\hat{k}}^{B_{\hat{k}}-1}\sim\text{Multi}(\p_{\hat{k}})}\Bigg(f\Big(\big(\tilde{S}\cup\{e_{\hat{k}}^{B_{\hat{k}}}\}\big)\cup\{\cup_{\hat{b}=1}^{B_{\hat{k}}-2}v_{\hat{k}}^{N_{\hat{b}}}\}\cup\{e_{\hat{k}}^{B_{\hat{k}}-1}\}\Big)\Big|v_{\hat{k}}^{1},\dots,v_{\hat{k}}^{B_{\hat{k}}-2}\Bigg).     \end{aligned}
 \end{equation*}

By repeatedly iterating the process described above, we will eventually obtain the result of Eq.\eqref{equ_rounding_1} for the case where $B^{\hat{k}}_{min}\triangleq\min(\|\p_{\hat{k}}\|_{0},B_{\hat{k}})=B_{\hat{k}}$.

With the result of Eq.\eqref{equ_rounding_1}, we then can show that 
\begin{equation*}
    \begin{aligned}
&\E_{S}\left(f\left(S\right)\right)=\E_{\mathcal{A}_{\hat{k}},\hat{k}\in[K]}\left(f\left(\cup_{\hat{k}=1}^{K}\mathcal{A}_{\hat{k}}\right)\right)\\&=\E\left(\E_{\mathcal{A}_{K}}\left(f\big(\{\cup_{\hat{k}=1}^{K-1}\mathcal{A}_{\hat{k}}\}\cup\{\mathcal{A}_{K}\}\big)\Big|\mathcal{A}_{1},\dots,\mathcal{A}_{K-1}\right)\right)\\&\ge\E\left(\E_{e_{K}^{\hat{b}}\sim\text{Multi}(\p_{K}),\forall\hat{b}\in[B_{K}]}\left(f\big(\{\cup_{\hat{k}=1}^{K-1}\mathcal{A}_{\hat{k}}\}\cup\{\cup_{\hat{b}=1}^{B_{K}}e_{K}^{\hat{b}}\}\big)\Big|\mathcal{A}_{1},\dots,\mathcal{A}_{K-1}\right)\right)\\
&=\E\Bigg(\E\left(f\big(\big\{\{\cup_{\hat{k}=1}^{K-2}\mathcal{A}_{\hat{k}}\}\cup\{\cup_{\hat{b}=1}^{B_{K}}e_{K}^{\hat{b}}\}\big\}\cup\{\mathcal{A}_{K-1}\}\big)\Big|\mathcal{A}_{1},\dots,\mathcal{A}_{K-2},\{\cup_{\hat{b}=1}^{B_{K}}e_{K}^{\hat{b}}\}\right)\Bigg)\\
&\ge\E\Bigg(\E_{e_{K-1}^{\hat{b}}\sim\text{Multi}(\p_{K}),\forall\hat{b}\in[B_{K-1}]}\left(f\Big(\{\cup_{\hat{k}=1}^{K-2}\mathcal{A}_{\hat{k}}\}\cup\{\cup_{\hat{b}=1}^{B_{K}}e_{K}^{\hat{b}}\}\cup\{\cup_{\hat{b}=1}^{B_{K-1}}e_{K-1}^{\hat{b}}\}\Big)\Big|\mathcal{A}_{1},\dots,\mathcal{A}_{K-2},\{\cup_{\hat{b}=1}^{B_{K}}e_{K}^{\hat{b}}\}\right)\Bigg)\\
&\ge\dots\\
&\ge\E_{e_{\hat{k}}^{\hat{b}}\sim\text{Multi($\p_{\hat{k}}$)},\forall \hat{k}\in[K],\hat{b}\in[B_{\hat{k}}]}\Big(f\big(\cup_{\hat{k}=1}^{K}\cup_{\hat{b}=1}^{B_{\hat{k}}}\{e_{\hat{k}}^{\hat{b}}\}\big)\Big)=F(\p_{1},\dots,\p_{K}).
    \end{aligned}
\end{equation*} 
Thus, we get the \cref{thm_appendix_sampling_without_replacement}.
\end{proof}
With the result of \cref{thm_appendix_sampling_without_replacement}, we prove the \cref{thm5+} for our \cref{alg_without_replacement}. 
\begin{proof}
 At first, when $\|\p_{k}\|_{1}>0$, we define $\tilde{\p}_{k}\triangleq\frac{\p_{k}}{\|\p_{k}\|_{1}}$. Otherwise, we set $\tilde{\p}_{k}\triangleq(\frac{1}{n_{k}},\dots,\frac{1}{n_{k}})\in\R^{n_{k}}$. From the result of \cref{thm_appendix_sampling_without_replacement}, we can have $\E\Big(f\Big(\cup_{k=1}^{K}\cup_{b=1}^{B_{min}^{k}}\{e_{k}^{b}\}\Big)\Big)\ge F(\tilde{\p}_1,\dots,\tilde{\p}_K)$. Next, according to the monotonicity of $f$ and $B_{min}^{k}\le B_{k}$, we also can show that $\E\Big(f\Big(\cup_{k=1}^{K}\cup_{b=1}^{B_{k}}\{e_{k}^{b}\}\Big)\Big)\ge\E\Big(f\Big(\cup_{k=1}^{K}\cup_{b=1}^{B_{min}^{k}}\{e_{k}^{b}\}\Big)\Big)$ such that $\E\big(f(S)\big)=\E\Big(f\Big(\cup_{k=1}^{K}\cup_{b=1}^{B_{k}}\{e_{k}^{b}\}\Big)\Big)\ge F(\tilde{\p}_1,\dots,\tilde{\p}_K)$. Furthermore, from the second point of \cref{thm1}, we know that when the set function $f$ exhibits monotonicity, its Multinoulli Extension $F$ is also monotone. As a result, $F(\tilde{\p}_1,\dots,\tilde{\p}_K)\ge F(\p_1,\dots,\p_K)$ where the definition of $\tilde{\p}_{k}$ implies that $\tilde{\p}_{k}\ge\p_{k}$ for any $k\in[K]$. 

 From the previous discussions, we can  infer that $\E\big(f(S)\big)=\E\Big(f\Big(\cup_{k=1}^{K}\cup_{b=1}^{B_{k}}\{e_{k}^{b}\}\Big)\Big)\ge F(\tilde{\p}_1,\dots,\tilde{\p}_K)\ge F(\p_1,\dots,\p_K)$. Thus, we prove the second point in \cref{thm5+}. Similarly, from the first point of \cref{thm_appendix_sampling_without_replacement} and the Lines 9-11 in \cref{alg_without_replacement}, i.e., the complement sampling process, we can easily show the first point of \cref{thm5+}. 

\end{proof}

%% file: Proofs_for_Thm5.tex
In this section, we prove the Theorem~\ref{thm5}. 

Before going into the details, we firstly bound each second-order derivative of our proposed \texttt{ME} $F$, that is to say, 
\begin{lemma}~\label{lemma1_Thm5}
Given a monotone set function $f:2^{\V}\rightarrow\R_{+}$, if we denote the maximum marginal value of $f$ as $M_{f}=\max_{S\subseteq\V,e\in\V\setminus S}\Big(f(e|S)\Big)$, we have that, for any $k_{1}, k_{2}\in[K]$, $m_{1}\in[n_{k_1}]$ and $m_{2}\in[n_{k_2}]$,
\begin{equation*}
	\Big|\frac{\partial^{2}F}{\partial p_{k_{1}}^{m_{1}}\partial p_{k_{2}}^{m_{2}}}(\p_{1},\dots,\p_{K})\Big|\le \bar{B}^{2}M_{f},
\end{equation*} where $(\p_{1},\dots,\p_{K})\in\prod_{k=1}^{K}\Delta_{n_{k}}$ and $\bar{B}=\max_{k=1}^{K}B_{k}$ is the maximum budget over the K communities $\{\V_{1},\dots,\V_{K}\}$.
\end{lemma}
\begin{proof}
From the Theorem~\ref{thm4}, we know that,  If $k_{1}\neq k_{2}\in[K]$, the second-order derivative of the Multinoulli Extension $F$ at any point $(\p_{1},\dots,\p_{K})\in\prod_{k=1}^{K}\Delta_{n_{k}}$ can be  written as follows: 
\begin{equation*}
	\frac{\partial^{2}F}{\partial p_{k_{1}}^{m_{1}}\partial p_{k_{2}}^{m_{2}}}(\p_{1},\dots,\p_{K})=B_{k_{1}}B_{k_{2}}\E_{e_{\hat{k}}^{\hat{b}}}\Big(f\big(v_{k_{1}}^{m_{1}}|S\cup\{v_{k_{2}}^{m_{2}}\}\big)-f\big(v_{k_{1}}^{m_{1}}|S\big)\Big),
\end{equation*}
where $S=\cup_{(\hat{k},\hat{b})\neq\left\{(k_{1},1),(k_{2},1)\right\}}\{e_{\hat{k}}^{\hat{b}}\}$ and each  $e_{\hat{k}}^{\hat{b}}$ is drawn from the multinoulli distribution Multi($\p_{\hat{k}}$). Furthermore, for any monotone set function $f$ and any subset $S\subseteq\V$, we can easily know that:
\begin{equation*}
	-M_{f}\le-f\big(v_{k_{1}}^{m_{1}}|S\big)\le f\big(v_{k_{1}}^{m_{1}}|S\cup\{v_{k_{2}}^{m_{2}}\}\big)-f\big(v_{k_{1}}^{m_{1}}|S\big)\le f\big(v_{k_{1}}^{m_{1}}|S\cup\{v_{k_{2}}^{m_{2}}\}\big)\le M_{f}
\end{equation*} such that, when $k_{1}\neq k_{2}\in[K]$,
\begin{equation*}
	|\frac{\partial^{2}F}{\partial p_{k_{1}}^{m_{1}}\partial p_{k_{2}}^{m_{2}}}(\p_{1},\dots,\p_{K})|\le B_{k_{1}}B_{k_{2}}M_{f}\le \bar{B}^{2}M_{f}.
\end{equation*}
Similarly, we also can verify that when $k_{1}=k_{2}=k\in[K]$, for any $m_{1}\in[n_{k_1}]$ and $m_{2}\in[n_{k_2}]$, we have that $\Big|\frac{\partial^{2}F}{\partial p_{k}^{m_{1}}\partial p_{k}^{m_{2}}}(\p_{1},\dots,\p_{K})\Big|\le \bar{B}^{2}M_{f}$.
\end{proof}

Similarly, we also can verify that the estimations of the second-order derivative of Multinoulli Extension $F$ in Remark~\ref{remark:estimate_2_order} are also bounded by $\bar{B}^{2}M_{f}$, that is to say, 
\begin{lemma}~\label{lemma2_Thm5}
	Given a monotone set function $f:2^{\V}\rightarrow\R_{+}$, if we denote the maximum marginal value of $f$ as $M_{f}=\max_{S\subseteq\V,e\in\V\setminus S}\Big(f(e|S)\Big)$, we can infer that each  second-order estimators in Remark~\ref{remark:estimate_2_order} is also bounded by $\bar{B}^{2}M_{f}$, i.e., for any $k_{1}, k_{2}\in[K]$, $m_{1}\in[n_{k_1}]$ and $m_{2}\in[n_{k_2}]$, we have that
	\begin{equation*}
	|\widehat{\frac{\partial^{2}F}{\partial p_{k_{1}}^{m_{1}}\partial p_{k_{2}}^{m_{2}}}}((\p_{1},\dots,\p_{K}))|\le \bar{B}^{2}M_{f}
	\end{equation*} where $(\p_{1},\dots,\p_{K})\in\prod_{k=1}^{K}\Delta_{n_{k}}$ and $\bar{B}=\max_{k=1}^{K}B_{k}$ is the maximum budget over the K communities $\{\V_{1},\dots,\V_{K}\}$.
\end{lemma}
As a result, we also show that 
\begin{lemma}~\label{lemma3_Thm5}
	Given a monotone set function $f:2^{\V}\rightarrow\R_{+}$, if we denote the maximum marginal value of $f$ as $M_{f}=\max_{S\subseteq\V,e\in\V\setminus S}\Big(f(e|S)\Big)$, we can infer that  each Hessian approximation $\widehat{\nabla}^{2}_{t}:=\frac{1}{L}\sum_{l=1}^{L} \widehat{\nabla}^{2}F\big(\x_{l}(t)\big)$ in Line 9 of Algorithm~\ref{alg:scg_mutli} satisfies that, 
	\begin{equation*}
	\|\widehat{\nabla}^{2}_{t}\|^{2}_{2,\infty}\le n\bar{B}^{4}M^{2}_{f}
	\end{equation*} where $t\in[T]$, $n=|\V|$, $L$ is a  positive integer, $\bar{B}=\max_{k=1}^{K}B_{k}$ is the maximum budget over the K communities $\{\V_{1},\dots,\V_{K}\}$.
\end{lemma}
\begin{remark}
For any matrix $A\in\R^{n\times n}$, the $(2,\infty)$-norm of $A$ is defined as $\|A\|_{2,\infty}=\sup\{\|A\x\|_{\infty}: \x\in\R^{n},\|\x\|_{2}=1\}$ where $\|\cdot\|$ denotes the L2 norm.
\end{remark}
\begin{proof}
From the definition of the norm $\|\cdot\|_{2,\infty}$, we can show that 
\begin{equation*}
	\|\widehat{\nabla}^{2}_{t}\|^{2}_{2,\infty}=\max_{i\in[n]}\|\widehat{\nabla}^{2}_{t}(i,:)\|_{2}^{2}\le n\bar{B}^{4}M^{2}_{f},
\end{equation*} where $\widehat{\nabla}^{2}_{t}(i,:)$ is the $i$-th line of the Hessian approximation $\widehat{\nabla}^{2}_{t}$ and the final inequality follows from the \cref{lemma2_Thm5}.
\end{proof}

With the \cref{lemma3_Thm5}, we next verify the gap between our gradient estimator $\boldsymbol{\mathrm{g}}(t)$ and the exact gradient $\nabla F(\x(t))$, that is,
\begin{lemma}~\label{lemma4_Thm5}
	Given a monotone set function $f:2^{\V}\rightarrow\R_{+}$, if we denote the maximum marginal value of $f$ as $M_{f}=\max_{S\subseteq\V,e\in\V\setminus S}\Big(f(e|S)\Big)$, we can show that each gradient estimator  $\boldsymbol{\mathrm{g}}(t)$ in Line  11 of Algorithm~\ref{alg:scg_mutli} satisfies that, for any $t\in[T]$
	\begin{equation*}
	\E\Big(\|\boldsymbol{\mathrm{g}}(t)-\nabla F(\x(t))\|^{2}_{2}\Big)\le\frac{nr\bar{B}^{4}M^{2}_{f}}{LT},
	\end{equation*} where $L$ is the batch size, $n=|\V|$, $\bar{B}=\max_{k=1}^{K}B_{k}$ is the maximum budget over the K communities $\{\V_{1},\dots,\V_{K}\}$ and the rank $r=\sum_{k=1}^{K}B_{k}$.
\end{lemma}
\begin{proof}
Note that in Lines 4 of Algorithm~\ref{alg:scg_mutli}, we compute the exact gradient of our proposed Multinoulli Extension $F$ at the point $\textbf{0}$ and then assign this value to $\boldsymbol{\mathrm{g}}(1)$. Therefore, we know that when $t=1$, 
\begin{equation*}
\|\boldsymbol{\mathrm{g}}(1)-\nabla F(\x(1))\|=\|\boldsymbol{\mathrm{g}}(1)-\nabla F(\textbf{0})\|=0\le\frac{nr\bar{B}^{4}M^{2}_{f}}{LT}.
\end{equation*}

When $t>1$, we have that
\begin{equation*}
	\begin{aligned}
		&\E\Big(\|\boldsymbol{\mathrm{g}}(t)-\nabla F(\x(t))\|^{2}_{2}\Big)\\
		&=\E\Big(\|\boldsymbol{\mathrm{g}}(t-1)+\boldsymbol{\xi}_{t}-\nabla F(\x(t))\|^{2}_{2}\Big)\\
		&=\E\Big(\|\boldsymbol{\mathrm{g}}(t-1)-\nabla F(\x(t-1))\|^{2}_{2}\Big)+\E\Big(\|\boldsymbol{\xi}_{t}-\Big(\nabla F(\x(t))-\nabla F(\x(t-1))\Big)\|^{2}_{2}\Big)\\
		&+\E\Big(\left\langle\boldsymbol{g}(t-1)-\nabla F(\x(t-1)),\boldsymbol{\xi}_{t}-\Big(\nabla F(\x(t))-\nabla F(\x(t-1))\Big)\right\rangle\Big).
	\end{aligned}
\end{equation*}
Note that 
\begin{equation*}
	\begin{aligned}
		&\E\Big(\left\langle\boldsymbol{g}(t-1)-\nabla F(\x(t-1)),\boldsymbol{\xi}_{t}-\Big(\nabla F(\x(t))-\nabla F(\x(t-1))\Big)\right\rangle\Big)\\
		&=\E\Bigg(\E\Big(\left\langle\boldsymbol{g}(t-1)-\nabla F(\x(t-1)),\boldsymbol{\xi}_{t}-\Big(\nabla F(\x(t))-\nabla F(\x(t-1))\Big)\right\rangle\Big|\x(t)\Big)\Bigg)\\
			&=\E\Big(\left\langle\boldsymbol{g}(t-1)-\nabla F(\x(t-1)),\E\Big(\boldsymbol{\xi}_{t}|\x(t)\Big)-\Big(\nabla F(\x(t))-\nabla F(\x(t-1))\Big)\right\rangle\Big)\Bigg)\\
			&=0.
	\end{aligned}
\end{equation*}
Therefore, we have
\begin{equation*}
	\begin{aligned}
		&\E\Big(\|\boldsymbol{\mathrm{g}}(t)-\nabla F(\x(t))\|^{2}_{2}\Big)\\
		&=\E\Big(\|\boldsymbol{\mathrm{g}}(t-1)-\nabla F(\x(t-1))\|^{2}_{2}\Big)+\E\Big(\|\boldsymbol{\xi}_{t}-\Big(\nabla F(\x(t))-\nabla F(\x(t-1))\Big)\|^{2}_{2}\Big)\\
		&=\E\Big(\|\boldsymbol{\mathrm{g}}(t-1)-\nabla F(\x(t-1))\|^{2}_{2}\Big)+\E\Big(\left\|\widehat{\nabla}^{2}_{t}\big(\x(t)-\x(t-1)\big)-\Big(\nabla F(\x(t))-\nabla F(\x(t-1))\Big)\right\|^{2}_{2}\Big)\\
		&=\E\Big(\|\boldsymbol{\mathrm{g}}(t-1)-\nabla F(\x(t-1))\|^{2}_{2}\Big)+\E\Big(\left\|\Big(\frac{1}{L}\sum_{l=1}^{L} \Big(\widehat{\nabla}^{2}F\big(\x_{l}(t)\big)\big(\x(t)-\x(t-1)\big)\Big)-\Big(\nabla F(\x(t))-\nabla F(\x(t-1))\Big)\right\|^{2}_{2}\Big)\\
			&=\E\Big(\|\boldsymbol{\mathrm{g}}(t-1)-\nabla F(\x(t-1))\|^{2}_{2}\Big)+\frac{1}{L}\E\Bigg(\left\| \Big(\widehat{\nabla}^{2}F\big(\x_{1}(t)\big)\big(\x(t)-\x(t-1)\big)\Big)-\Big(\nabla F(\x(t))-\nabla F(\x(t-1))\Big)\right\|^{2}_{2}\Bigg)\\
			&\le\E\Big(\|\boldsymbol{\mathrm{g}}(t-1)-\nabla F(\x(t-1))\|^{2}_{2}\Big)+\frac{1}{L}\E\Bigg(\| \widehat{\nabla}^{2}F\big(\x_{1}(t)\big)\big(\x(t)-\x(t-1)\big)\|^{2}_{2}\Bigg)\\
		    	&\le\E\Big(\|\boldsymbol{\mathrm{g}}(t-1)-\nabla F(\x(t-1))\|^{2}_{2}\Big)+\frac{1}{L}\E\Bigg(\| \widehat{\nabla}^{2}F\big(\x_{1}(t)\big)\|^{2}_{2,\infty}\|\x(t)-\x(t-1)\|^{2}_{2}\Bigg)\\
		    		&=\E\Big(\|\boldsymbol{\mathrm{g}}(t-1)-\nabla F(\x(t-1))\|^{2}_{2}\Big)+\frac{1}{L}\E\Bigg(\| \widehat{\nabla}^{2}F\big(\x_{1}(t)\big)\|^{2}_{2,\infty}\|\frac{1}{T}\sum_{k=1}^{K}\frac{1}{B_{k}}\one_{S(t-1)\cap\V_{k}}\|^{2}_{2}\Bigg)\\
		    	&\le\E\Big(\|\boldsymbol{\mathrm{g}}(t-1)-\nabla F(\x(t-1))\|^{2}_{2}\Big)+\frac{nr\bar{B}^{4}M^{2}_{f}}{LT^{2}}\\
		    	&\dots\\
		    	&\le\E\Big(\|\boldsymbol{\mathrm{g}}(1)-\nabla F(\x(1))\|^{2}_{2}\Big)+\frac{nr\bar{B}^{4}M^{2}_{f}}{LT^{2}}(t-1)\le\frac{nr\bar{B}^{4}M^{2}_{f}}{LT},
	\end{aligned}
\end{equation*} where the first inequality follows from $\E\big(X-\E(X)\big)^{2}\le\E(X^{2})$ for any random variable $X$; the second inequality  comes from the definition of the norm $\|\cdot\|_{2,\infty}$; the third inequality follows from the \cref{lemma3_Thm5} and the ascent direction $\sum_{k=1}^{K}\frac{1}{B_{k}}\one_{S(t-1)\cap\V_{k}}$ has at most $r$ non-zero elements.
\end{proof}
Now, we verify the Theorem~\ref{thm5}.
\begin{proof}
From calculus, we know that, there exist a constant $a\in[0,1]$ such that
\begin{equation*}
F(\x(t+1))-F(\x(t))-\left\langle\nabla F(\x(t)),\x(t+1)-\x(t)\right\rangle=\frac{1}{2}\left\langle\nabla^{2} F(\x^{a}(t+1))\Big(\x(t+1)-\x(t)\Big),\x(t+1)-\x(t)\right\rangle,
\end{equation*} where $\x^{a}(t+1)=a\x(t+1)+(1-a)\x(t)$.
Therefore, we can show that
\begin{equation}\label{equ_appendix5_final1}
	\begin{aligned}
	&F(\x(t+1))\\
	&\ge F(\x(t))+\left\langle\nabla F(\x(t)),\x(t+1)-\x(t)\right\rangle- \frac{1}{2}\left|\left\langle\nabla^{2} F(\x^{a}(t+1))\Big(\x(t+1)-\x(t)\Big),\x(t+1)-\x(t)\right\rangle\right|\\
		&\ge F(\x(t))+\left\langle\nabla F(\x(t)),\x(t+1)-\x(t)\right\rangle- \frac{\|\nabla^{2} F(\x^{a}(t+1))\|_{2,\infty}}{2}\|\x(t+1)-\x(t)\|_{2}^{2}\\
			&\ge F(\x(t))+\left\langle\nabla F(\x(t)),\x(t+1)-\x(t)\right\rangle- \frac{\bar{B}^{2}M_{f}\sqrt{n}}{2}\|\x(t+1)-\x(t)\|_{2}^{2},
	\end{aligned}
\end{equation} where the final inequality follows from \cref{lemma3_Thm5}. 

With Eq.\eqref{equ_appendix5_final1} and $\x(t+1)=\x(t)+\frac{1}{T}\v(t)$, we also have that,for any subset $S$ within the partition constraint of problem~\eqref{equ_problem},
\begin{equation*}
	\begin{aligned}
		&F(\x(t+1))\\
		&\ge F(\x(t))+\left\langle\nabla F(\x(t)),\x(t+1)-\x(t)\right\rangle- \frac{\bar{B}^{2}M_{f}\sqrt{n}}{2}\|\x(t+1)-\x(t)\|_{2}^{2}\\
		&=F(\x(t))+\frac{1}{T}\left\langle\nabla F(\x(t)),\sum_{k=1}^{K}\frac{1}{B_{k}}\one_{S(t)\cap\V_{k}}\right\rangle- \frac{\bar{B}^{2}M_{f}\sqrt{n}}{2T^{2}}\|\sum_{k=1}^{K}\frac{1}{B_{k}}\one_{S(t)\cap\V_{k}}\|_{2}^{2}\\
		&=F(\x(t))+\frac{1}{T}\left\langle\boldsymbol{g}(t),\sum_{k=1}^{K}\frac{\one_{S(t)\cap\V_{k}}}{B_{k}}\right\rangle+\frac{1}{T}\left\langle\nabla F(\x(t))-\boldsymbol{g}(t),\sum_{k=1}^{K}\frac{1}{B_{k}}\one_{S(t)\cap\V_{k}}\right\rangle- \frac{\bar{B}^{2}M_{f}\sqrt{n}}{2T^{2}}\|\sum_{k=1}^{K}\frac{1}{B_{k}}\one_{S(t)\cap\V_{k}}\|_{2}^{2}\\
		&\ge F(\x(t))+\frac{1}{T}\left\langle\boldsymbol{g}(t),\sum_{k=1}^{K}\frac{\one_{S\cap\V_{k}}}{B_{k}}\right\rangle+\frac{1}{T}\left\langle\nabla F(\x(t))-\boldsymbol{g}(t),\sum_{k=1}^{K}\frac{1}{B_{k}}\one_{S(t)\cap\V_{k}}\right\rangle- \frac{\bar{B}^{2}M_{f}\sqrt{n}}{2T^{2}}\|\sum_{k=1}^{K}\frac{1}{B_{k}}\one_{S(t)\cap\V_{k}}\|_{2}^{2},
	\end{aligned}
\end{equation*} where the final inequality follows from  $\sum_{k=1}^{K}\frac{\one_{S\cap\V_{k}}}{B_{k}}\in\prod_{k=1}^{T}\Delta_{n_{k}}$ if the subset $S$ is included into the partition constraint of problem~\eqref{equ_problem} and Line 13 in \cref{alg:scg_mutli}.

As a result, we can show that, in expectation,
\begin{small}
  \begin{equation}\label{equ_appendix5_final2}
\begin{aligned}
	&\E\big(F(\x(t+1))\big)\\&\ge \E\big(F(\x(t))\big)+\frac{1}{T}\E\big(\left\langle\boldsymbol{g}(t),\sum_{k=1}^{K}\frac{\one_{S\cap\V_{k}}}{B_{k}}\right\rangle\big)+\frac{1}{T}\E\big(\left\langle\nabla F(\x(t))-\boldsymbol{g}(t),\sum_{k=1}^{K}\frac{1}{B_{k}}\one_{S(t)\cap\V_{k}}\right\rangle\big)\\&-\frac{\bar{B}^{2}M_{f}\sqrt{n}}{2T^{2}} \E\big(\|\sum_{k=1}^{K}\frac{\one_{S(t)\cap\V_{k}}}{B_{k}}\|_{2}^{2}\big)\\
	&= \E\big(F(\x(t))\big)+\frac{1}{T}\left\langle\E\Big(\boldsymbol{g}(t)\Big),\sum_{k=1}^{K}\frac{\one_{S\cap\V_{k}}}{B_{k}}\right\rangle+\frac{1}{T}\E\big(\left\langle\nabla F(\x(t))-\boldsymbol{g}(t),\sum_{k=1}^{K}\frac{1}{B_{k}}\one_{S(t)\cap\V_{k}}\right\rangle\big)\\&-\frac{\bar{B}^{2}M_{f}\sqrt{n}}{2T^{2}} \E\big(\|\sum_{k=1}^{K}\frac{\one_{S(t)\cap\V_{k}}}{B_{k}}\|_{2}^{2}\big)\\
	&= \E\big(F(\x(t))\big)+\frac{1}{T}\left\langle\nabla F(\x(t)),\sum_{k=1}^{K}\frac{\one_{S\cap\V_{k}}}{B_{k}}\right\rangle+\frac{1}{T}\E\big(\left\langle\nabla F(\x(t))-\boldsymbol{g}(t),\sum_{k=1}^{K}\frac{1}{B_{k}}\one_{S(t)\cap\V_{k}}\right\rangle\big)\\&-\frac{\bar{B}^{2}M_{f}\sqrt{n}}{2T^{2}} \E\big(\|\sum_{k=1}^{K}\frac{\one_{S(t)\cap\V_{k}}}{B_{k}}\|_{2}^{2}\big),
\end{aligned}
\end{equation}   
\end{small}where the final equality follows from $\boldsymbol{g}(\x(t))$ is the unbiased estimator of $\nabla F(\x(t))$ for any $t\in[T]$.

\textbf{1):} When the set function $f:2^{\V}\rightarrow\R_{+}$ is monotone and $\alpha$-weakly DR-submodular, from \cref{thm2}, we can show that 
\begin{small}
 \begin{equation}\label{equ_appendix5_final3}
	\begin{aligned}
		&\E\big(F(\x(t+1))\big)\\
		&= \E\big(F(\x(t))\big)+\frac{1}{T}\left\langle\nabla F(\x(t)),\sum_{k=1}^{K}\frac{\one_{S\cap\V_{k}}}{B_{k}}\right\rangle+\frac{1}{T}\E\big(\left\langle\nabla F(\x(t))-\boldsymbol{g}(t),\sum_{k=1}^{K}\frac{\one_{S(t)\cap\V_{k}}}{B_{k}}\right\rangle\big)-\frac{\bar{B}^{2}M_{f}\sqrt{n}}{2T^{2}} \E\big(\|\sum_{k=1}^{K}\frac{\one_{S(t)\cap\V_{k}}}{B_{k}}\|_{2}^{2}\big)\\
		&\ge \E\big(F(\x(t))\big)+\frac{\alpha}{T}\Big(f(S)-\E\big(F(\x(t))\Big)++\frac{1}{T}\E\big(\left\langle\nabla F(\x(t))-\boldsymbol{g}(t),\sum_{k=1}^{K}\frac{\one_{S(t)\cap\V_{k}}}{B_{k}}\right\rangle\big)-\frac{\bar{B}^{2}M_{f}\sqrt{n}}{2T^{2}} \E\big(\|\sum_{k=1}^{K}\frac{\one_{S(t)\cap\V_{k}}}{B_{k}}\|_{2}^{2}\\
	&\ge \E\big(F(\x(t))\big)+\frac{\alpha}{T}\Big(f(S)-\E\big(F(\x(t))\Big)-\frac{1}{2\bar{B}^{2}M_{f}\sqrt{n}}\E\Big(\|F(\x(t))-\boldsymbol{g}(t)\|_{2}^{2}\Big)-\frac{\bar{B}^{2}M_{f}\sqrt{n}}{T^{2}} \E\big(\|\sum_{k=1}^{K}\frac{1}{B_{k}}\one_{S(t)\cap\V_{k}}\|_{2}^{2}\big)\big)\\
		&\ge \E\big(F(\x(t))\big)+\frac{\alpha}{T}\Big(f(S)-\E\big(F(\x(t))\big)\Big)-\frac{r\sqrt{n}\bar{B}^{2}M_{f}}{2LT}-\frac{r\sqrt{n}\bar{B}^{2}M_{f}}{T^{2}}\\
	\end{aligned}
\end{equation}    
\end{small}where the second inequality follows from the Young’s inequality and the final inequality comes from \cref{lemma4_Thm5}.

By rearranging the Eq.\eqref{equ_appendix5_final3}, we can show that
\begin{equation*}
	\begin{aligned}
		&\Big(f(S)-\E\big(F(\x(t+1))\big)\Big)\\
		&\le(1-\frac{\alpha}{T})\Big(f(S)-\E\big(F(\x(t))\Big)+\frac{r\sqrt{n}\bar{B}^{2}M_{f}}{2LT}+\frac{r\sqrt{n}\bar{B}^{2}M_{f}}{T^{2}}\\
		&\le\dots\\
		&\le(1-\frac{\alpha}{T})^{t}\Big(f(S)-\E\big(F(\x(1))\Big)+\Big(\frac{r\sqrt{n}\bar{B}^{2}M_{f}}{2LT}+\frac{r\sqrt{n}\bar{B}^{2}M_{f}}{T^{2}}\Big)\sum_{i=0}^{t-1}(1-\frac{\alpha}{T})^{i}\\
	&\le(1-\frac{\alpha}{T})^{t}\Big(f(S)-\E\big(F(\x(1))\Big)+\Big(\frac{r\sqrt{n}\bar{B}^{2}M_{f}}{2LT}+\frac{r\sqrt{n}\bar{B}^{2}M_{f}}{T^{2}}\Big)\frac{T}{\alpha}\\
	&\le(1-\frac{\alpha}{T})^{t}f(S)+\frac{r\sqrt{n}\bar{B}^{2}M_{f}}{2\alpha L}+\frac{r\sqrt{n}\bar{B}^{2}M_{f}}{\alpha T}.
	\end{aligned}
\end{equation*} 
Finally, we have that 
\begin{equation*}
\begin{aligned}
	&\E\big(F(\x(T+1)\big)\\
	&\ge\Big(1-(1-\frac{\alpha}{T})^{T}\Big)f(S)-\frac{r\sqrt{n}\bar{B}^{2}M_{f}}{2\alpha L}-\frac{r\sqrt{n}\bar{B}^{2}M_{f}}{\alpha T}\\
	&\ge\Big(1-e^{-\alpha}\Big)f(S)-\frac{r\sqrt{n}\bar{B}^{2}M_{f}}{2\alpha L}-\frac{r\sqrt{n}\bar{B}^{2}M_{f}}{\alpha T},
\end{aligned}
\end{equation*} where the final inequality follows from $(1-\frac{\alpha}{T})^{T}\le e^{-\alpha}$ when $T\ge 3$.

Therefore, when $L=\frac{T}{2}$, we have that
\begin{equation*}
\E\big(F(\x(T+1)\big)\ge\Big(1-e^{-\alpha}\Big)f(S^{*})-\frac{2r\sqrt{n}\bar{B}^{2}M_{f}}{\alpha T}.
\end{equation*} 
\textbf{2:} When the set function $f:2^{\V}\rightarrow\R_{+}$ is monotone and $(\gamma,\beta)$-weakly submodular, from the Theorem~\ref{thm2}, we can show that
\begin{equation}\label{equ_appendix5_final2_1}
	\E\big(F(\x(t+1))\big)\ge \E\big(F(\x(t))\big)+\frac{\gamma^{2}}{T}f(S)-\frac{\beta(1-\gamma)+\gamma^{2}}{T}\E\big(F(\x(t))\big)-\frac{r\sqrt{n}\bar{B}^{2}M_{f}}{2LT}-\frac{r\sqrt{n}\bar{B}^{2}M_{f}}{T^{2}}.
\end{equation} 
By rearranging the Eq.\eqref{equ_appendix5_final2_1}, we can have that
\begin{equation*}
	\begin{aligned}
		&\Big(\gamma^{2}f(S)-(\beta(1-\gamma)+\gamma^{2})\E\big(F(\x(t+1))\big)\Big)\\
		&\le(1-\frac{\beta(1-\gamma)+\gamma^{2}}{T})\Big(\gamma^{2}f(S)-(\beta(1-\gamma)+\gamma^{2})\E\big(F(\x(t))\big)\Big)+\Big(\frac{r\sqrt{n}\bar{B}^{2}M_{f}}{2LT}+\frac{r\sqrt{n}\bar{B}^{2}M_{f}}{T^{2}}\Big)(\beta(1-\gamma)+\gamma^{2})\\
		&\le\dots\\
		&\le(1-\frac{\beta(1-\gamma)+\gamma^{2}}{T})^{t}\Big(f\gamma^{2}f(S)-(\beta(1-\gamma)+\gamma^{2})\E\big(F(\x(0))\Big)+\Big(\frac{r\sqrt{n}\bar{B}^{2}M_{f}}{2LT}+\frac{r\sqrt{n}\bar{B}^{2}M_{f}}{T^{2}}\Big)T\\
		&\le(1-\frac{\beta(1-\gamma)+\gamma^{2}}{T})^{t}\gamma^{2}f(S)+\frac{r\sqrt{n}\bar{B}^{2}M_{f}}{2L}+\frac{ r\sqrt{n}\bar{B}^{2}M_{f}}{T}.
	\end{aligned}
\end{equation*} 
Finally, we have that 
\begin{equation*}
	\begin{aligned}
		&(\beta(1-\gamma)+\gamma^{2})\E\big(F(\x(T+1)\big)\\
		&\ge\Big(1-(1-\frac{\beta(1-\gamma)+\gamma^{2}}{T})^{T}\Big)\gamma^{2}f(S)-\frac{r\sqrt{n}\bar{B}^{2}M_{f}}{2L}+\frac{ r\sqrt{n}\bar{B}^{2}M_{f}}{T}\\
		&\ge\gamma^{2}(1-e^{-(\beta(1-\gamma)+\gamma^{2})})f(S)-\frac{r\sqrt{n}\bar{B}^{2}M_{f}}{2L}+\frac{ r\sqrt{n}\bar{B}^{2}M_{f}}{T},
	\end{aligned}
\end{equation*} where the final inequality follows from $(1-\frac{\beta(1-\gamma)+\gamma^{2}}{T})^{T}\le e^{-(\beta(1-\gamma)+\gamma^{2})}$ when $T\ge 3$.

Therefore, when $L=\frac{T}{2}$, we have that
\begin{equation*}
	\E\big(F(\x(T+1)\big)\ge\Big(\frac{\gamma^{2}(1-e^{-(\gamma(1-\beta)+\gamma^{2})})}{\gamma(1-\beta)+\gamma^{2}}\Big)f(S)-\frac{2 r\sqrt{n}\bar{B}^{2}M_{f}}{T(\beta(1-\gamma)+\gamma^{2})}.
\end{equation*} 
\end{proof}
\begin{remark}
If $T=L=\mathcal{O}(\frac{r\sqrt{n}}{\epsilon})$ and $S$ is the optimal subset of problem~\eqref{equ_problem}, we can show that, when the objective function is monotone $\alpha$-weakly DR-submodular or $(\gamma,\beta)$-weakly submodular, our \textbf{Multinoulli-SCG} algorithm can attain a value of   $(1-e^{-\alpha})f(S)-\epsilon$ or $(\frac{\gamma^{2}(1-e^{-(\beta(1-\gamma)+\gamma^2)})}{\beta(1-\gamma)+\gamma^2})f(S)-\epsilon$.  Note that, during the process of  $\mathcal{O}(\frac{r\sqrt{n}}{\epsilon})$ iterations, if $L=\mathcal{O}(\frac{r\sqrt{n}}{\epsilon})$, due to \cref{remark:scg_limited_queries}, \textbf{Multinoulli-SCG} only requires evaluating the set function $\mathcal{O}(\frac{r^{3}n^{2}}{\epsilon^{2}})$ times.
\end{remark}
\subsection{Proof of \texorpdfstring{\cref{thm:online}}{}}\label{appendix:online_scg}
This subsection aims to provide a proof for \cref{thm:online}. At first, like \cref{lemma1_Thm5}, \cref{lemma3_Thm5} and \cref{lemma4_Thm5}, if we denote the maximum marginal value of the sequence  $\{f_{1},\dots,f_{T}\}$ as $M_{1:T}\triangleq\max_{S\subseteq\V,e\in\V\setminus S,t\in[T]}\Big(f_{t}(e|S)\Big)$, then we can show that
\begin{lemma}~\label{lemma1_thm_online}Given a sequence of monotone set functions $\{f_{1},\dots,f_{T}\}$ where $f_{t}:2^{\V}\rightarrow\R_{+}$, if we denote the maximum marginal value of $\{f_{1},\dots,f_{T}\}$ as $M_{1:T}\triangleq\max_{S\subseteq\V,e\in\V\setminus S,t\in[T]}\Big(f_{t}(e|S)\Big)$, we have that, for any $k_{1}, k_{2}\in[K]$, $m_{1}\in[n_{k_1}]$ and $m_{2}\in[n_{k_2}]$,
\begin{equation*}
	\Big|\frac{\partial^{2}F_{t}}{\partial p_{k_{1}}^{m_{1}}\partial p_{k_{2}}^{m_{2}}}(\p_{1},\dots,\p_{K})\Big|\le \bar{B}^{2}M_{1:T},
\end{equation*} where $F_{t}$ is the \texttt{ME} of $f_{t}$ for any $t\in[T]$,$(\p_{1},\dots,\p_{K})\in\prod_{k=1}^{K}\Delta_{n_{k}}$ and $\bar{B}=\max_{k=1}^{K}B_{k}$ is the maximum budget over the K communities $\{\V_{1},\dots,\V_{K}\}$.
\end{lemma}
\begin{lemma}~\label{lemma2_thm_online}
	Given a sequence of monotone set functions $\{f_{1},\dots,f_{T}\}$ where $f_{t}:2^{\V}\rightarrow\R_{+}$, if we denote the maximum marginal value of $\{f_{1},\dots,f_{T}\}$ as $M_{1:T}\triangleq\max_{S\subseteq\V,e\in\V\setminus S,t\in[T]}\Big(f_{t}(e|S)\Big)$, we can infer that  each Hessian approximation $\widehat{\nabla}^{2}_{t}(q)\triangleq\frac{1}{L}\sum_{l=1}^{L} \widehat{\nabla}^{2}F_{t}\big(\x_{t}^{(l)}(q)\big),\forall q\in[Q],\forall t\in[T]$ in Line 17 of Algorithm~\ref{alg:online_scg_mutli} satisfies that, 
	\begin{equation*}
	\|\widehat{\nabla}^{2}_{t}(q)\|^{2}_{2,\infty}\le n\bar{B}^{4}M_{1:T}^{2}
	\end{equation*} where $t\in[T]$, $n=|\V|$, $L$ is a  positive integer, $\bar{B}=\max_{k=1}^{K}B_{k}$ is the maximum budget over the $K$ communities $\{\V_{1},\dots,\V_{K}\}$.
\end{lemma}
\begin{lemma}~\label{lemma3_thm_online}
	Given a sequence of monotone set functions $\{f_{1},\dots,f_{T}\}$ where $f_{t}:2^{\V}\rightarrow\R_{+}$, if we denote the maximum marginal value of $\{f_{1},\dots,f_{T}\}$ as $M_{1:T}\triangleq\max_{S\subseteq\V,e\in\V\setminus S,t\in[T]}\Big(f_{t}(e|S)\Big)$, we can show that each gradient estimator  $\boldsymbol{\mathrm{g}}_{t}(q),\forall q\in[Q],\forall t\in[T]$ in Line  19 of Algorithm~\ref{alg:online_scg_mutli} satisfies that, for any $t\in[T]$
	\begin{equation*}
	\E\Big(\|\boldsymbol{\mathrm{g}}_{t}(q)-\nabla F_{t}(\x_{t}(q))\|^{2}_{2}\Big)\le\frac{nr\bar{B}^{4}M^{2}_{1:T}}{LQ},
	\end{equation*} where $L$ is the batch size, $n=|\V|$, $\bar{B}=\max_{k=1}^{K}B_{k}$ is the maximum budget over the $K$ communities $\{\V_{1},\dots,\V_{K}\}$ and the rank $r=\sum_{k=1}^{K}B_{k}$.
\end{lemma}
Now, we verify the Theorem~\ref{thm:online}.
\begin{proof}
From calculus, we know that, there exist a constant $a\in[0,1]$ such that
\begin{equation*}
F_{t}(\x_{t}(q+1))-F_{t}(\x_{t}(q))-\left\langle\nabla F_{t}(\x_{t}(q)),\x_{t}(q+1)-\x_{t}(q)\right\rangle=\frac{1}{2}\left\langle\nabla^{2} F_{t}(\x_{t}^{a}(q+1)))\Big(\x_{t}(q+1)-\x_{t}(q)\Big),\x_{t}(q+1)-\x_{t}(q)\right\rangle,
\end{equation*} where $\x_{t}^{a}(q+1)=a\x_{t}(q+1)+(1-a)\x_{t}(q)$.
Therefore, we can show that
\begin{equation}\label{equ_online_0}
	\begin{aligned}
	&F_{t}(\x_{t}(q+1))\\
	&\ge F_{t}(\x_{t}(q+1))+\left\langle\nabla F_{t}(\x_{t}(q)),\x_{t}(q+1)-\x_{t}(q)\right\rangle- \frac{1}{2}\left|\left\langle\nabla^{2} F_{t}(\x^{a}_{t}(q+1))\Big(\x_{t}(q+1)-\x_{t}(q)\Big),\x_{t}(q+1)-\x_{t}(q)\right\rangle\right|\\
		&\ge F_{t}(\x_{t}(q+1))+\left\langle\nabla F_{t}(\x_{t}(q)),\x_{t}(q+1)-\x_{t}(q)\right\rangle- \frac{\|\nabla^{2} F_{t}(\x^{a}_{t}(q+1))\|_{2,\infty}}{2}\|\x_{t}(q+1)-\x_{t}(q)\|_{2}^{2}\\
			&\ge F_{t}(\x_{t}(q+1))+\left\langle\nabla F_{t}(\x_{t}(q)),\x_{t}(q+1)-\x_{t}(q)\right\rangle- \frac{\bar{B}^{2}M_{1:T}\sqrt{n}}{2}\|\x_{t}(q+1)-\x_{t}(q)\|_{2}^{2},
	\end{aligned}
\end{equation} where the final inequality follows from \cref{lemma2_thm_online}. 

Next, according to the Line 6 of \cref{alg:online_scg_mutli}, i.e., $\x_{t}(q+1)\triangleq \x_{t}(q)+\frac{1}{Q}\v_{t}(q)$, we also can show that, for any subset $S$ within the partition constraint of problem~\eqref{equ_problem},
\begin{equation*}
    \begin{aligned}
      &F_{t}(\x_{t}(q+1))\\
	&\ge F_{t}(\x_{t}(q+1))+\frac{1}{Q}\left\langle\nabla F_{t}(\x_{t}(q)),\v_{t}(q)\right\rangle- \frac{\bar{B}^{2}M_{1:T}\sqrt{n}}{2Q^{2}}\|\v_{t}(q)\|_{2}^{2}\\
    &=F_{t}(\x_{t}(q+1))+\frac{1}{Q}\left\langle\nabla F_{t}(\x_{t}(q)),\sum_{k=1}^{K}\frac{\one_{S\cap\V_{k}}}{B_{k}}\right\rangle+\frac{1}{Q}\left\langle\boldsymbol{\mathrm{g}}_{t}(q),\v_{t}(q)-\sum_{k=1}^{K}\frac{\one_{S\cap\V_{k}}}{B_{k}}\right\rangle\\&+\frac{1}{Q}\left\langle\nabla F_{t}(\x_{t}(q))-\boldsymbol{\mathrm{g}}_{t}(q),\v_{t}(q)-\sum_{k=1}^{K}\frac{\one_{S\cap\V_{k}}}{B_{k}}\right\rangle- \frac{\bar{B}^{2}M_{1:T}\sqrt{n}}{2Q^{2}}\|\v_{t}(q)\|_{2}^{2}.
    \end{aligned}
\end{equation*}
It is worth noting that when $\x\in\prod_{k=1}^{K}\Delta_{n_{k}}$, we have $\|\x\|_{2}^{2}\le\|\x\|_{1}\le K\le r$ where the rank $r\triangleq\sum_{k=1}^{K}B_{k}$. As a result, we can have that
\begin{equation}\label{equ:online1:1}
    \begin{aligned}
     &\frac{1}{Q}\E\left(\Big|\left\langle\nabla F_{t}(\x_{t}(q))-\boldsymbol{\mathrm{g}}_{t}(q),\v_{t}(q)-\sum_{k=1}^{K}\frac{\one_{S\cap\V_{k}}}{B_{k}}\right\rangle\big|\right)+ \frac{\bar{B}^{2}M_{1:T}\sqrt{n}}{2Q^{2}}\E\left(\|\v_{t}(q)\|_{2}^{2}\right)\\
     &\le\frac{1}{2\bar{B}^{2}M_{1:T}\sqrt{n}}\E\left(\|\nabla F_{t}(\x_{t}(q))-\boldsymbol{\mathrm{g}}_{t}(q)\|_{2}^{2}\right)+\frac{\bar{B}^{2}M_{1:T}\sqrt{n}}{2Q^{2}}\E\left(\|\v_{t}(q)-\sum_{k=1}^{K}\frac{\one_{S\cap\V_{k}}}{B_{k}}\|_{2}^{2}\right)+ \frac{\bar{B}^{2}M_{1:T}\sqrt{n}}{2Q^{2}}\E\left(\|\v_{t}(q)\|_{2}^{2}\right)\\
     &\le\frac{r\sqrt{n}\bar{B}^{2}M_{1:T}}{2LQ}+\frac{3r\sqrt{n}\bar{B}^{2}M_{1:T}}{2Q^{2}},
    \end{aligned}
\end{equation} where the first inequality follows from the Young’s inequality and the second inequality comes from \cref{lemma3_thm_online} and the truth that $\v_{t}(q)\in\prod_{k=1}^{K}\Delta_{n_{k}}$ and $\sum_{k=1}^{K}\frac{\one_{S\cap\V_{k}}}{B_{k}}\in\prod_{k=1}^{K}\Delta_{n_{k}}$ such that $\E(\|\v_{t}(q)\|_{2}^{2})\le r$ and $\E(\|\v_{t}(q)-\sum_{k=1}^{K}\frac{\one_{S\cap\V_{k}}}{B_{k}}\|_{2}^{2})\le\E(\|\v_{t}(q)\|_{2}^{2})+\E(\|\sum_{k=1}^{K}\frac{\one_{S\cap\V_{k}}}{B_{k}}\|_{2}^{2})\le 2r$.

With the result of Eq.\eqref{equ:online1:1}, we can have that
\begin{equation*}
    \begin{aligned}
      &\E\left(F_{t}(\x_{t}(q+1))\right)\ge\E\left(F_{t}(\x_{t}(q+1))+\frac{1}{Q}\left\langle\nabla F_{t}(\x_{t}(q)),\sum_{k=1}^{K}\frac{\one_{S\cap\V_{k}}}{B_{k}}\right\rangle+\frac{1}{Q}\left\langle\boldsymbol{\mathrm{g}}_{t}(q),\v_{t}(q)-\sum_{k=1}^{K}\frac{\one_{S\cap\V_{k}}}{B_{k}}\right\rangle\right)\\&+\frac{1}{Q}\left\langle\nabla F_{t}(\x_{t}(q))-\boldsymbol{\mathrm{g}}_{t}(q),\v_{t}(q)-\sum_{k=1}^{K}\frac{\one_{S\cap\V_{k}}}{B_{k}}\right\rangle- \frac{\bar{B}^{2}M_{1:T}\sqrt{n}}{2Q^{2}}\|\v_{t}(q)\|_{2}^{2}\\&\ge\E\left(F_{t}(\x_{t}(q+1))+\frac{1}{Q}\left\langle\nabla F_{t}(\x_{t}(q)),\sum_{k=1}^{K}\frac{\one_{S\cap\V_{k}}}{B_{k}}\right\rangle+\frac{1}{Q}\left\langle\boldsymbol{\mathrm{g}}_{t}(q),\v_{t}(q)-\sum_{k=1}^{K}\frac{\one_{S\cap\V_{k}}}{B_{k}}\right\rangle\right)\\&-\frac{r\sqrt{n}\bar{B}^{2}M_{1:T}}{2LQ}-\frac{3r\sqrt{n}\bar{B}^{2}M_{1:T}}{2Q^{2}}
    .
    \end{aligned}
\end{equation*}
\textbf{1):} When each set function $f_{t}:2^{\V}\rightarrow\R_{+}$ is monotone and $\alpha$-weakly DR-submodular, from \cref{thm2}, we can show that 
\begin{equation}\label{equ:online1:2}
    \begin{aligned}
        &\E\left(F_{t}(\x_{t}(q+1))\right)\\&\ge\E\left(F_{t}(\x_{t}(q))\right)+\frac{1}{Q}\E\left(\left\langle\nabla F_{t}(\x_{t}(q)),\sum_{k=1}^{K}\frac{\one_{S\cap\V_{k}}}{B_{k}}\right\rangle\right)+\frac{1}{Q}\E\left(\left\langle\boldsymbol{\mathrm{g}}_{t}(q),\v_{t}(q)-\sum_{k=1}^{K}\frac{\one_{S\cap\V_{k}}}{B_{k}}\right\rangle\right)\\&-\frac{r\sqrt{n}\bar{B}^{2}M_{1:T}}{2LQ}-\frac{3r\sqrt{n}\bar{B}^{2}M_{1:T}}{2Q^{2}}\\
        &\ge\E\Big(F_{t}(\x_{t}(q))\Big)+\frac{\alpha}{Q}\Big(f_{t}(S)-\E\big(F_{t}(\x_{t}(q))\big)\Big)+\frac{1}{Q}\E\left(\left\langle\boldsymbol{\mathrm{g}}_{t}(q),\v_{t}(q)-\sum_{k=1}^{K}\frac{\one_{S\cap\V_{k}}}{B_{k}}\right\rangle\right)-\frac{r\sqrt{n}\bar{B}^{2}M_{1:T}}{2LQ}-\frac{3r\sqrt{n}\bar{B}^{2}M_{1:T}}{2Q^{2}}.
    \end{aligned}
\end{equation}
By rearranging the Eq.\eqref{equ:online1:2}, we can show that
\begin{equation*}
    \begin{aligned}
     &f_{t}(S)-\E\big(F_{t}(\x_{t}(q+1))\big)\\
     &\le(1-\frac{\alpha}{Q})\left(f_{t}(S)-\E\big(F_{t}(\x_{t}(q))\big)\right)+\frac{r\sqrt{n}\bar{B}^{2}M_{1:T}}{2LQ}+\frac{3r\sqrt{n}\bar{B}^{2}M_{1:T}}{2Q^{2}}+\frac{1}{Q}\E\left(\left\langle\boldsymbol{\mathrm{g}}_{t}(q),\sum_{k=1}^{K}\frac{\one_{S\cap\V_{k}}}{B_{k}}-\v_{t}(q)\right\rangle\right)\\
     &\le\dots\\
     &\le(1-\frac{\alpha}{Q})^{q}\left(f_{t}(S)-\E\big(F_{t}(\x_{t}(1))\big)\right)+\Big(\frac{r\sqrt{n}\bar{B}^{2}M_{1:T}}{2LQ}+\frac{3r\sqrt{n}\bar{B}^{2}M_{1:T}}{2Q^{2}}\Big)\sum_{i=0}^{q-1}(1-\frac{\alpha}{Q})^{i}\\&+\frac{1}{Q}\sum_{i=1}^{q}\E\left(\left\langle\boldsymbol{\mathrm{g}}_{t}(i),\sum_{k=1}^{K}\frac{\one_{S\cap\V_{k}}}{B_{k}}-\v_{t}(i)\right\rangle\right)\\
     &=(1-\frac{\alpha}{Q})^{q}f_{t}(S)+\Big(\frac{r\sqrt{n}\bar{B}^{2}M_{1:T}}{2LQ}+\frac{3r\sqrt{n}\bar{B}^{2}M_{1:T}}{2Q^{2}}\Big)\sum_{i=0}^{q-1}(1-\frac{\alpha}{Q})^{i}+\frac{1}{Q}\sum_{i=1}^{q}\E\left(\left\langle\boldsymbol{\mathrm{g}}_{t}(i),\sum_{k=1}^{K}\frac{\one_{S\cap\V_{k}}}{B_{k}}-\v_{t}(i)\right\rangle\right)\\
     &\le(1-\frac{\alpha}{Q})^{q}f_{t}(S)+\frac{r\sqrt{n}\bar{B}^{2}M_{1:T}}{2L\alpha}+\frac{3r\sqrt{n}\bar{B}^{2}M_{1:T}}{2\alpha Q}+\frac{1}{Q}\sum_{i=1}^{q}\E\left(\left\langle\boldsymbol{\mathrm{g}}_{t}(i),\sum_{k=1}^{K}\frac{\one_{S\cap\V_{k}}}{B_{k}}-\v_{t}(i)\right\rangle\right),
    \end{aligned}
\end{equation*} where the final equality follows from $\x_{t}(1)\triangleq\boldsymbol{0}_{n},F_{t}(\boldsymbol{0}_{n})=0$ and the final inequality comes from $\sum_{i=0}^{q-1}(1-\frac{\alpha}{Q})^{i}\le\frac{Q}{\alpha}$.

As a result, we have 
\begin{equation}\label{equ:online1:3}
    \begin{aligned}
     &E\big(F_{t}(\x_{t}(Q+1))\big)\\&\ge\left(1-(1-\frac{\alpha}{Q})^{Q}\right)f_{t}(S)-\frac{r\sqrt{n}\bar{B}^{2}M_{1:T}}{2L\alpha}-\frac{3r\sqrt{n}\bar{B}^{2}M_{1:T}}{2\alpha Q}-\frac{1}{Q}\sum_{q=1}^{Q}\E\left(\left\langle\boldsymbol{\mathrm{g}}_{t}(q),\sum_{k=1}^{K}\frac{\one_{S\cap\V_{k}}}{B_{k}}-\v_{t}(q)\right\rangle\right)\\
     &\ge\left(1-e^{-\alpha}\right)f_{t}(S)-\frac{r\sqrt{n}\bar{B}^{2}M_{1:T}}{2L\alpha}-\frac{3r\sqrt{n}\bar{B}^{2}M_{1:T}}{2\alpha Q}-\frac{1}{Q}\sum_{q=1}^{Q}\E\left(\left\langle\boldsymbol{\mathrm{g}}_{t}(q),\sum_{k=1}^{K}\frac{\one_{S\cap\V_{k}}}{B_{k}}-\v_{t}(q)\right\rangle\right),
    \end{aligned}
\end{equation} where the second inequality comes from $(1-\frac{\alpha}{Q})^{Q}\le e^{-\alpha}$.

Then, summing up both sides of Eq.\eqref{equ:online1:3} from $t=1,\dots,T$, we can get
\begin{equation*}
    \begin{aligned}
    &\sum_{t=1}^{T}E\big(F_{t}(\x_{t}(Q+1))\big)\\&\ge\left(1-e^{-\alpha}\right)\sum_{t=1}^{T}f_{t}(S)-\frac{r\sqrt{n}\bar{B}^{2}M_{1:T}}{2\alpha}\frac{T}{L}-\frac{3r\sqrt{n}\bar{B}^{2}M_{1:T}}{2\alpha}\frac{T}{Q}-\frac{1}{Q}\sum_{t=1}^{T}\sum_{q=1}^{Q}\E\left(\left\langle\boldsymbol{\mathrm{g}}_{t}(q),\sum_{k=1}^{K}\frac{\one_{S\cap\V_{k}}}{B_{k}}-\v_{t}(q)\right\rangle\right)\\
    &=\left(1-e^{-\alpha}\right)\sum_{t=1}^{T}f_{t}(S)-\frac{r\sqrt{n}\bar{B}^{2}M_{1:T}}{2\alpha}\frac{T}{L}-\frac{3r\sqrt{n}\bar{B}^{2}M_{1:T}}{2\alpha}\frac{T}{Q}-\frac{1}{Q}\sum_{q=1}^{Q}\E\left(\sum_{t=1}^{T}\left\langle\boldsymbol{\mathrm{g}}_{t}(q),\sum_{k=1}^{K}\frac{\one_{S\cap\V_{k}}}{B_{k}}-\v_{t}(q)\right\rangle\right)\\
     &\ge\left(1-e^{-\alpha}\right)\sum_{t=1}^{T}f_{t}(S)-\frac{r\sqrt{n}\bar{B}^{2}M_{1:T}}{2\alpha}\frac{T}{L}-\frac{3r\sqrt{n}\bar{B}^{2}M_{1:T}}{2\alpha}\frac{T}{Q}-\frac{1}{Q}\sum_{q=1}^{Q}\E\left(\text{Reg}_{1}^{\mathcal{E}^{(q)}}(T)\right),
    \end{aligned}
\end{equation*} where the final inequality comes from the definition of \emph{1-Regret} of each linear maximization oracle $\mathcal{E}^{(q)},\forall q\in[Q]$, that is, $\text{Reg}_{1}^{\mathcal{E}^{(q)}}(T)\triangleq\max_{\x\in\prod_{k=1}^{K}\Delta_{n_{k}}}\sum_{t=1}^{T}\langle\boldsymbol{\mathrm{g}}_{t}(q),\x(q)\rangle-\sum_{t=1}^{T}\langle\boldsymbol{\mathrm{g}}_{t}(q),\v_{t}(q)\rangle$. Note that, when $S$ within the partition constraint of problem~\eqref{equ_problem}, $\sum_{k=1}^{K}\frac{\one_{S\cap\V_{k}}}{B_{k}}\in\prod_{k=1}^{K}\Delta_{n_{k}}$ such that $\text{Reg}_{1}^{\mathcal{E}^{(q)}}(T)\ge\sum_{t=1}^{T}\left\langle\boldsymbol{\mathrm{g}}_{t}(q),\sum_{k=1}^{K}\frac{\one_{S\cap\V_{k}}}{B_{k}}-\v_{t}(q)\right\rangle$.

Finally, from the \cref{thm5+} about the rounding-without-replacement process and Lines 8-9 of \cref{alg:online_scg_mutli}, we can infer that the $(1-e^{-\alpha})$-Regret of our proposed \cref{alg:online_scg_mutli}, i.e., $\text{Reg}_{(1-e^{-\alpha})}\left(T\right)$ satisfies that
\begin{equation*}
    \begin{aligned}
    &\text{Reg}_{(1-e^{-\alpha})}\left(T\right)\triangleq\left(1-e^{-\alpha}\right)\sum_{t=1}^{T}f_{t}(S^{*})-\sum_{t=1}^{T}E\big(f_{t}(S_{t})\big)\\&\le\left(1-e^{-\alpha}\right)\sum_{t=1}^{T}f_{t}(S^{*})-\sum_{t=1}^{T}E\big(F_{t}(\x_{t}(Q+1))\big)\\&\le\frac{r\sqrt{n}\bar{B}^{2}M_{1:T}}{2\alpha}\frac{T}{L}+\frac{3r\sqrt{n}\bar{B}^{2}M_{1:T}}{2\alpha}\frac{T}{Q}+\frac{1}{Q}\sum_{q=1}^{Q}\E\left(\text{Reg}_{1}^{\mathcal{E}^{(q)}}(T)\right)=\mathcal{O}\left(\sqrt{nr}T^{1/2}+r\sqrt{n}\left(\frac{T}{L}+\frac{T}{Q}\right)\right),
    \end{aligned}
\end{equation*} where $S^{*}\triangleq\mathop{\arg\max}_{S\subseteq\V,|S\cup\V_{k}|\le B_{k},\forall k\in[K]}\sum_{t=1}^{T}f_{t}(S)$ and the final equality follows from \cref{ass:1}.

\textbf{2):} When each set function $f_{t}:2^{\V}\rightarrow\R_{+}$ is monotone and $(\gamma,\beta)$-weakly submodular, from the Theorem~\ref{thm2}, we can show
\begin{equation}\label{equ:online1:21}
    \begin{aligned}
        &\E\left(F_{t}(\x_{t}(q+1))\right)\ge\E\Big(F_{t}(\x_{t}(q))\Big)+\frac{\gamma^{2}}{Q}f_{t}(S)-\frac{\beta(1-\gamma)+\gamma^{2}}{Q}\E\big(F_{t}(\x_{t}(q))\big)\Big)\\&+\frac{1}{Q}\E\left(\left\langle\boldsymbol{\mathrm{g}}_{t}(q),\v_{t}(q)-\sum_{k=1}^{K}\frac{\one_{S\cap\V_{k}}}{B_{k}}\right\rangle\right)-\frac{r\sqrt{n}\bar{B}^{2}M_{1:T}}{2LQ}-\frac{3r\sqrt{n}\bar{B}^{2}M_{1:T}}{2Q^{2}}.
    \end{aligned}
\end{equation}

By rearranging the Eq.\eqref{equ:online1:21}, we also can have that
\begin{equation*}
    \begin{aligned}
     &\Big(\gamma^{2}f_{t}(S)-(\beta(1-\gamma)+\gamma^{2})\E\big(F_{t}(\x_{t}(q+1))\big)\Big)\le(1-\frac{\beta(1-\gamma)+\gamma^{2}}{Q})\left(\gamma^{2}f_{t}(S)-(\beta(1-\gamma)+\gamma^{2})\E\big(F_{t}(\x_{t}(q))\big)\right)\\&+\Big(\frac{r\sqrt{n}\bar{B}^{2}M_{1:T}}{2LQ}+\frac{3r\sqrt{n}\bar{B}^{2}M_{1:T}}{2Q^{2}}+\frac{1}{Q}\E\left(\left\langle\boldsymbol{\mathrm{g}}_{t}(q),\sum_{k=1}^{K}\frac{\one_{S\cap\V_{k}}}{B_{k}}-\v_{t}(q)\right\rangle\right)\Big)(\beta(1-\gamma)+\gamma^{2})\\
     &\le\dots\\
     &\le\gamma^{2}(1-\frac{\beta(1-\gamma)+\gamma^{2}}{Q})^{q}f_{t}(S)+\frac{r\sqrt{n}\bar{B}^{2}M_{1:T}}{2L}+\frac{3r\sqrt{n}\bar{B}^{2}M_{1:T}}{2Q}+\frac{\beta(1-\gamma)+\gamma^{2}}{Q}\sum_{i=1}^{q}\E\left(\left\langle\boldsymbol{\mathrm{g}}_{t}(i),\sum_{k=1}^{K}\frac{\one_{S\cap\V_{k}}}{B_{k}}-\v_{t}(i)\right\rangle\right).
    \end{aligned}
\end{equation*}
As a result, we have that
\begin{equation}\label{equ:online1:22}
    \begin{aligned}
     &(\beta(1-\gamma)+\gamma^{2})E\big(F_{t}(\x_{t}(Q+1))\big)\\&\ge\gamma^{2}\big(1-(1-\frac{\beta(1-\gamma)+\gamma^{2}}{Q})^{Q}\big)f_{t}(S)-\frac{r\sqrt{n}\bar{B}^{2}M_{1:T}}{2L}-\frac{3r\sqrt{n}\bar{B}^{2}M_{1:T}}{2Q}-\frac{\beta(1-\gamma)+\gamma^{2}}{Q}\sum_{q=1}^{Q}\E\left(\left\langle\boldsymbol{\mathrm{g}}_{t}(q),\sum_{k=1}^{K}\frac{\one_{S\cap\V_{k}}}{B_{k}}-\v_{t}(q)\right\rangle\right)\\
     &\ge\gamma^{2}(1-e^{-(\beta(1-\gamma)+\gamma^{2})})f_{t}(S)-\frac{r\sqrt{n}\bar{B}^{2}M_{1:T}}{2L}-\frac{3r\sqrt{n}\bar{B}^{2}M_{1:T}}{2Q}-\frac{\beta(1-\gamma)+\gamma^{2}}{Q}\sum_{q=1}^{Q}\E\left(\left\langle\boldsymbol{\mathrm{g}}_{t}(q),\sum_{k=1}^{K}\frac{\one_{S\cap\V_{k}}}{B_{k}}-\v_{t}(q)\right\rangle\right),
    \end{aligned}
\end{equation}  where the final inequality follows from $(1-\frac{\beta(1-\gamma)+\gamma^{2}}{Q})^{Q}\le e^{-(\beta(1-\gamma)+\gamma^{2})}$ when $Q\ge 3$.

Then, summing up both sides of Eq.\eqref{equ:online1:22} from $t=1,\dots,T$, we can get
\begin{equation*}
    \begin{aligned}
    &\sum_{t=1}^{T}E\big(F_{t}(\x_{t}(Q+1))\big)\\&\ge\Big(\frac{\gamma^{2}(1-e^{-\phi(\gamma,\beta)})}{\phi(\gamma,\beta)}\Big)\sum_{t=1}^{T}f_{t}(S)-\frac{r\sqrt{n}\bar{B}^{2}M_{1:T}}{2\phi(\gamma,\beta)}\frac{T}{L}-\frac{3r\sqrt{n}\bar{B}^{2}M_{1:T}}{2\phi(\gamma,\beta)}\frac{T}{Q}-\frac{1}{Q}\sum_{t=1}^{T}\sum_{q=1}^{Q}\E\left(\left\langle\boldsymbol{\mathrm{g}}_{t}(q),\sum_{k=1}^{K}\frac{\one_{S\cap\V_{k}}}{B_{k}}-\v_{t}(q)\right\rangle\right)\\
    \\&\ge\Big(\frac{\gamma^{2}(1-e^{-\phi(\gamma,\beta)})}{\phi(\gamma,\beta)}\Big)\sum_{t=1}^{T}f_{t}(S)-\frac{r\sqrt{n}\bar{B}^{2}M_{1:T}}{2\phi(\gamma,\beta)}\frac{T}{L}-\frac{3r\sqrt{n}\bar{B}^{2}M_{1:T}}{2\phi(\gamma,\beta)}\frac{T}{Q}-\frac{1}{Q}\E\left(\text{Reg}_{1}^{\mathcal{E}^{(q)}}(T)\right),
    \end{aligned}
\end{equation*}  where $\phi(\gamma,\beta)\triangleq\beta(1-\gamma)+\gamma^{2}$ and the final inequality comes from the definition of \emph{1-Regret} of each linear maximization oracle $\mathcal{E}^{(q)},\forall q\in[Q]$, that is, $\text{Reg}_{1}^{\mathcal{E}^{(q)}}(T)\triangleq\max_{\x\in\prod_{k=1}^{K}\Delta_{n_{k}}}\sum_{t=1}^{T}\langle\boldsymbol{\mathrm{g}}_{t}(q),\x(q)\rangle-\sum_{t=1}^{T}\langle\boldsymbol{\mathrm{g}}_{t}(q),\v_{t}(q)\rangle$. Note that, when $S$ within the partition constraint of problem~\eqref{equ_problem}, $\sum_{k=1}^{K}\frac{\one_{S\cap\V_{k}}}{B_{k}}\in\prod_{k=1}^{K}\Delta_{n_{k}}$ such that $\text{Reg}_{1}^{\mathcal{E}^{(q)}}(T)\ge\sum_{t=1}^{T}\left\langle\boldsymbol{\mathrm{g}}_{t}(q),\sum_{k=1}^{K}\frac{\one_{S\cap\V_{k}}}{B_{k}}-\v_{t}(q)\right\rangle$.

Finally, from the \cref{thm5+} about the rounding-without-replacement process and Lines 8-9 of \cref{alg:online_scg_mutli}, we can infer that the $\Big(\frac{\gamma^{2}(1-e^{-(\beta(1-\gamma)+\gamma^{2})})}{\beta(1-\gamma)+\gamma^{2}}\Big)$-Regret of our proposed \cref{alg:online_scg_mutli} satisfies that
\begin{equation*}
    \begin{aligned}
 &\text{Reg}_{\rho}\left(T\right)\triangleq\left(\frac{\gamma^{2}(1-e^{-(\beta(1-\gamma)+\gamma^{2})}}{\beta(1-\gamma)+\gamma^{2}}\right)\sum_{t=1}^{T}f_{t}(S^{*})-\sum_{t=1}^{T}E\big(f_{t}(S_{t})\big)\\&\le\left(\frac{\gamma^{2}(1-e^{-(\beta(1-\gamma)+\gamma^{2})}}{\beta(1-\gamma)+\gamma^{2}}\right)\sum_{t=1}^{T}f_{t}(S^{*})-\sum_{t=1}^{T}E\big(F_{t}(\x_{t}(Q+1))\big)\\&\le\frac{r\sqrt{n}\bar{B}^{2}M_{1:T}}{2(\beta(1-\gamma)+\gamma^{2})}\frac{T}{L}+\frac{3r\sqrt{n}\bar{B}^{2}M_{1:T}}{2(\beta(1-\gamma)+\gamma^{2})}\frac{T}{Q}+\frac{1}{Q}\sum_{q=1}^{Q}\E\left(\text{Reg}_{1}^{\mathcal{E}^{(q)}}(T)\right)\\&=\mathcal{O}\left(\sqrt{nr}T^{1/2}+r\sqrt{n}\left(\frac{T}{L}+\frac{T}{Q}\right)\right),
    \end{aligned}
\end{equation*} where  $\rho\triangleq\frac{\gamma^{2}(1-e^{-(\beta(1-\gamma)+\gamma^{2})})}{\beta(1-\gamma)+\gamma^{2}}$, $S^{*}\triangleq\mathop{\arg\max}_{S\subseteq\V,|S\cup\V_{k}|\le B_{k},\forall k\in[K]}\sum_{t=1}^{T}f_{t}(S)$ and the final equality follows from \cref{ass:1}.
\end{proof}

%% file: Multinoulli-SGA.tex
In this subsection, we prove the Theorem~\ref{thm:online1}. At first, we give a lemma about the upper bound of gradients of $F_{t}$ where $F_{t}$ is the multinoulli extension of $t$-th incoming objective set function $f_{t}$ for any $t\in[T]$.
\begin{lemma}\label{lemma_bound_gradient}
	Given a sequence of monotone set functions $\{f_{1},\dots,f_{T}\}$ where $f_{t}:2^{\V}\rightarrow\R_{+}$, if we denote the maximum marginal value of $\{f_{1},\dots,f_{T}\}$ as $M_{1:T}\triangleq\max_{S\subseteq\V,e\in\V\setminus S,t\in[T]}\Big(f_{t}(e|S)\Big)$, we have that, for any $k\in[K]$, $m\in[n_{k}]$ and $t\in[T]$,
	\begin{equation*}
		\Big|\frac{\partial F_{t}}{\partial p_{k}^{m}}(\p_{1},\dots,\p_{K})\Big|\le \bar{B}M_{1:T},
	\end{equation*} where $(\p_{1},\dots,\p_{K})\in\prod_{k=1}^{K}\Delta_{n_{k}}$ and $\bar{B}\triangleq\max_{k=1}^{K}B_{k}$ is the maximum budget over the $K$ communities $\{\V_{1},\dots,\V_{K}\}$. Similarly, we also can infer that the gradient estimation in \cref{remark:gradient_estimation} also can be bounded by $\bar{B}M_{f}$,
	\begin{equation*}
		\Big|\widehat{\frac{\partial F_{t}}{\partial p_{k}^{m}}}(\p_{1},\dots,\p_{K})\Big|=B_{k}\Big|f_{t}\big(v_{k}^{m}\Big|\cup_{(\hat{k},\hat{b})\neq(k,1)}\{e_{\hat{k}}^{\hat{b}}\}\big)\Big|\le \bar{B}M_{1:T},
	\end{equation*} where each $e_{\hat{k}}^{\hat{b}}$ is independently drawn from the multinoulli distribution `Multi($\p_{\hat{k}}$)' for any $\hat{k}\in[K]$ and $\hat{b}\in[B_{\hat{k}}]$.
\end{lemma} 
With this \cref{lemma_bound_gradient}, we then prove the \cref{thm:online1}.
\begin{proof}
Let $S$ denote any subset within the partition constraint of problem~\eqref{equ_problem}. Then, from the Line 14 in \cref{alg:OGA_framework}, we know that
\begin{equation}\label{appendix:online1:1}
\begin{aligned}
\left\|\x_{t+1}-\sum_{k=1}^{K}\frac{\one_{S\cap\V_{k}}}{B_{k}}\right\|_{2}^{2}&\le\left\|\y_{t+1}-\sum_{k=1}^{K}\frac{\one_{S\cap\V_{k}}}{B_{k}}\right\|_{2}^{2}=\left\|\x_{t}+\eta\boldsymbol{\mathrm{g}}_{t}-\sum_{k=1}^{K}\frac{\one_{S^{*}\cap\V_{k}}}{B_{k}}\right\|_{2}^{2}\\
&=\left\|\x_{t}-\sum_{k=1}^{K}\frac{\one_{S\cap\V_{k}}}{B_{k}}\right\|_{2}^{2}+2\eta\left\langle\boldsymbol{\mathrm{g}}_{t}, \x_t-\sum_{k=1}^{K}\frac{\one_{S\cap\V_{k}}}{B_{k}}\right\rangle+\eta^{2}\left\|\boldsymbol{\mathrm{g}}_{t}\right\|_{2}^{2},
\end{aligned}		
\end{equation}  where the first inequality follows from the the Line 14 of \cref{alg:OGA_framework} and $\sum_{k=1}^{K}\frac{\one_{S\cap\V_{k}}}{B_{k}}\in\prod_{k=1}^{K}\Delta_{n_{k}}$ due to that $S$ is a subset within the partition constraint of problem~\eqref{equ_problem}.

\textbf{1)} We consider \cref{alg:OGA_framework} with the standard gradient estimation(See Lines 6-7). In other words, we choose $\text{Auxiliary}\triangleq\textbf{False}$ and set $\boldsymbol{\mathrm{g}}_{t}\triangleq\widehat{\nabla} F_{t}\big(\x_{t}\big)$ . As a result, from \cref{appendix:online1:1}, we can have that
\begin{equation}\label{equ_SGA_1}
	\begin{aligned}
		\E\left\|\x_{t+1}-\sum_{k=1}^{K}\frac{\one_{S\cap\V_{k}}}{B_{k}}\right\|_{2}^{2}&\le \E\left\|\x_{t}-\sum_{k=1}^{K}\frac{\one_{S\cap\V_{k}}}{B_{k}}\right\|_{2}^{2}+2\eta\E\Big(\left\langle\boldsymbol{\mathrm{g}}_{t}, \x_t-\sum_{k=1}^{K}\frac{\one_{S\cap\V_{k}}}{B_{k}}\right\rangle\Big)+\eta^{2}\E(\left\|\boldsymbol{\mathrm{g}}_{t}\right\|_{2}^{2})\\
        &=\E\left\|\x_{t}-\sum_{k=1}^{K}\frac{\one_{S\cap\V_{k}}}{B_{k}}\right\|_{2}^{2}+2\eta\E\Big(\left\langle\widehat{\nabla} F_{t}\big(\x_{t}\big), \x_t-\sum_{k=1}^{K}\frac{\one_{S\cap\V_{k}}}{B_{k}}\right\rangle\Big)+\eta^{2}\E(\left\|\widehat{\nabla} F_{t}\big(\x_{t}\big)\right\|_{2}^{2})\\
        &\le\E\left\|\x_{t}-\sum_{k=1}^{K}\frac{\one_{S\cap\V_{k}}}{B_{k}}\right\|_{2}^{2}+2\eta\E\Big(\left\langle\widehat{\nabla} F_{t}\big(\x_{t}\big), \x_t-\sum_{k=1}^{K}\frac{\one_{S\cap\V_{k}}}{B_{k}}\right\rangle\Big)+\eta^{2}n\bar{B}^{2}M_{1:T}^{2}\\
		&=\E\left\|\x_t-\sum_{k=1}^{K}\frac{\one_{S\cap\V_{k}}}{B_{k}}\right\|_{2}^{2}+2\eta\E\Big(\left\langle\E\Big(\widehat{\nabla}F_{t}(\x_t)\Big|\x_t\Big), \x_t-\sum_{k=1}^{K}\frac{\one_{S\cap\V_{k}}}{B_{k}}\right\rangle\Big)+\eta^{2}n\bar{B}^{2}M_{1:T}^{2}\\
		&=\E\left\|\x_t-\sum_{k=1}^{K}\frac{\one_{S\cap\V_{k}}}{B_{k}}\right\|_{2}^{2}+2\eta\E\Big(\left\langle\nabla F_{t}(\x_t), \x_t-\sum_{k=1}^{K}\frac{\one_{S\cap\V_{k}}}{B_{k}}\right\rangle\Big)+\eta^{2}n\bar{B}^{2}M_{1:T}^{2},
	\end{aligned}		
\end{equation} where the second inequality follows from \cref{lemma_bound_gradient}.

Then, if \textbf{i):} every set function $f_{t}:2^{\V}\rightarrow\R_{+}$ is monotone and $\alpha$-weakly DR-submodular, from \cref{thm_appendix_2} and Eq.\eqref{equ_SGA_1}, we have that
\begin{equation*}
	\begin{aligned}
		\E\left\|\x_{t+1}-\sum_{k=1}^{K}\frac{\one_{S\cap\V_{k}}}{B_{k}}\right\|_{2}^{2}
		&\le\E\left\|\x_t-\sum_{k=1}^{K}\frac{\one_{S\cap\V_{k}}}{B_{k}}\right\|_{2}^{2}+2\eta\E\Big(\left\langle\nabla F_{t}(\x_t), \x_t-\sum_{k=1}^{K}\frac{\one_{S\cap\V_{k}}}{B_{k}}\right\rangle\Big)+\eta^{2}n\bar{B}^{2}M_{1:T}^{2}\\
		&\le\E\left\|\x_t-\sum_{k=1}^{K}\frac{\one_{S\cap\V_{k}}}{B_{k}}\right\|_{2}^{2}-2\eta\Big(\alpha f_{t}(S)-(\alpha+\frac{1}{\alpha})F_{t}(\x_t)\Big)+\eta^{2}n\bar{B}^{2}M_{1:T}^{2}.
	\end{aligned}		
\end{equation*} 
As a result, we have that
\begin{equation*}
\begin{aligned}
    2\sum_{t=1}^{T}\eta\Big(\alpha f_{t}(S)-(\alpha+\frac{1}{\alpha})F_{t}(\x_t)\Big)&\le\sum_{t=1}^{T}\left(\E\Big(\left\|\x_t-\sum_{k=1}^{K}\frac{\one_{S\cap\V_{k}}}{B_{k}}\right\|_{2}^{2}\Big)-E\Big(\left\|\x_{t+1}-\sum_{k=1}^{K}\frac{\one_{S\cap\V_{k}}}{B_{k}}\right\|_{2}^{2}\Big)\right)+n\bar{B}^{2}M_{1:T}^{2}T\eta^{2}\\
    &=\E\Big(\left\|\x_1-\sum_{k=1}^{K}\frac{\one_{S\cap\V_{k}}}{B_{k}}\right\|_{2}^{2}\Big)-E\Big(\left\|\x_{T+1}-\sum_{k=1}^{K}\frac{\one_{S\cap\V_{k}}}{B_{k}}\right\|_{2}^{2}\Big)+n\bar{B}^{2}M_{1:T}^{2}T\eta^{2}.
\end{aligned}
\end{equation*}
Due to that $\E(f_{t}(S_{t}))\ge\E(F_{t}(\x_t))$ from \cref{thm5+}, we can infer that 
\begin{equation*}
\begin{aligned}
&\alpha\sum_{t=1}^{T}f_{t}(S)-(\alpha+\frac{1}{\alpha})\sum_{t=1}^{T}f_{t}(S_{t})\le\sum_{t=1}^{T}\Big(\alpha f_{t}(S)-(\alpha+\frac{1}{\alpha})F_{t}(\x_t)\Big)\\
&\le\frac{\E\Big(\left\|\x_1-\sum_{k=1}^{K}\frac{\one_{S\cap\V_{k}}}{B_{k}}\right\|_{2}^{2}\Big)}{2\eta}+\frac{n\bar{B}^{2}M_{1:T}^{2}T}{2}\eta\\
&\le\frac{\E\Big(\|\x_1\|_{2}^{2}+\|\sum_{k=1}^{K}\frac{\one_{S\cap\V_{k}}}{B_{k}}\|_{2}^{2}\Big)}{2\eta}+\frac{n\bar{B}^{2}M_{1:T}^{2}T}{2}\eta\le\frac{r}{\eta}+\frac{n\bar{B}^{2}M_{1:T}^{2}T}{2}\eta,
\end{aligned}	
\end{equation*} where the second inequality follows from $\x_{1}\triangleq\sum_{k=1}^{K}\frac{1}{n_{k}} \one_{\V_{k}}\in\prod_{k=1}^{K}\Delta_{n_{k}}$ and $\sum_{k=1}^{K}\frac{\one_{S\cap\V_{k}}}{B_{k}}\in\prod_{k=1}^{K}\Delta_{n_{k}}$ such that $\left\|\x_1-\sum_{k=1}^{K}\frac{\one_{S\cap\V_{k}}}{B_{k}}\right\|_{2}^{2}\le\|\x_1\|_{2}^{2}+\|\sum_{k=1}^{K}\frac{\one_{S\cap\V_{k}}}{B_{k}}\|_{2}^{2}$ and the final inequality comes from the truth that for any $\x\in\prod_{k=1}^{K}\Delta_{n_{k}}$ $\|\x\|_{2}^{2}\le\|\x\|_{1}\le K\le r$. Note that $r$ is the rank of partition constraint, i.e., $r=\sum_{k=1}^{K}B_{k}$.

Therefore, we can infer that the $(\frac{\alpha^{2}}{1+\alpha^{2}})$-Regret of our proposed \cref{alg:OGA_framework}, i.e., $\text{Reg}_{(\frac{\alpha^{2}}{1+\alpha^{2}})}\left(T\right)$ satisfies that
\begin{equation*}
\begin{aligned}
   &\text{Reg}_{(\frac{\alpha^{2}}{1+\alpha^{2}})}\left(T\right)\triangleq\Big(\frac{\alpha^{2}}{1+\alpha^{2}}\Big)\sum_{t=1}^{T}f_{t}(S^{*})-\sum_{t=1}^{T}f_{t}(S_{t})\\&\le\Big(\frac{1}{\frac{1}{\alpha}+\alpha}\Big)\Big(\frac{r}{\eta}+\frac{n\bar{B}^{2}M_{1:T}^{2}T}{2}\eta\Big)\le\frac{r}{2\eta}+\frac{n\bar{B}^{2}M_{1:T}^{2}T}{4}\eta=\mathcal{O}\left(\frac{r}{\eta}+nT\eta\right)
\end{aligned}
\end{equation*} where $S^{*}\triangleq\mathop{\arg\max}_{S\subseteq\V,|S\cup\V_{k}|\le B_{k},\forall k\in[K]}\sum_{t=1}^{T}f_{t}(S)$ and the final inequality comes from $\frac{1}{\alpha}+\alpha\ge2$ when $\alpha\in(0,1]$.

Furthermore, \textbf{ii):} if each set function $f_{t}:2^{\V}\rightarrow\R_{+}$ is monotone and $(\gamma,\beta)$-weakly submodular, from \cref{thm_appendix_2} and Eq.\eqref{equ_SGA_1}, we have that
\begin{equation*}
	\begin{aligned}
		\E\left\|\x_{t+1}-\sum_{k=1}^{K}\frac{\one_{S\cap\V_{k}}}{B_{k}}\right\|_{2}^{2}
		&\le\E\left\|\x_t-\sum_{k=1}^{K}\frac{\one_{S\cap\V_{k}}}{B_{k}}\right\|_{2}^{2}+2\eta\E\Big(\left\langle\nabla F_{t}(\x_t), \x_t-\sum_{k=1}^{K}\frac{\one_{S\cap\V_{k}}}{B_{k}}\right\rangle\Big)+\eta^{2}n\bar{B}^{2}M_{1:T}^{2}\\
		&\le\E\left\|\x_t-\sum_{k=1}^{K}\frac{\one_{S\cap\V_{k}}}{B_{k}}\right\|_{2}^{2}-2\eta\Big(\gamma^{2}f_{t}(S^{*})-(\beta+\beta(1-\gamma)+\gamma^{2})F_{t}(\x_t)\Big)+\eta^{2}n\bar{B}^{2}M_{1:T}^{2},
	\end{aligned}		
\end{equation*} where $S^{*}\triangleq\mathop{\arg\max}_{S\subseteq\V,|S\cup\V_{k}|\le B_{k},\forall k\in[K]}\sum_{t=1}^{T}f_{t}(S)$.

As a result,
\begin{equation*}
\begin{aligned}
    2\eta\sum_{t=1}^{T}\Big( \gamma^{2}f_{t}(S^{*})-(\beta+\beta(1-\gamma)+\gamma^{2})F_{t}(\x_t)\Big)&\le\E\Big(\left\|\x_1-\sum_{k=1}^{K}\frac{\one_{S\cap\V_{k}}}{B_{k}}\right\|_{2}^{2}\Big)-E\Big(\left\|\x_{T+1}-\sum_{k=1}^{K}\frac{\one_{S\cap\V_{k}}}{B_{k}}\right\|_{2}^{2}\Big)+n\bar{B}^{2}M_{1:T}^{2}T\eta^{2}\\
    &\le 2r+n\bar{B}^{2}M_{1:T}^{2}T\eta^{2},
\end{aligned}
\end{equation*}where the final inequality follows from $\x_{1}\triangleq\sum_{k=1}^{K}\frac{1}{n_{k}} \one_{\V_{k}}\in\prod_{k=1}^{K}\Delta_{n_{k}}$ and $\sum_{k=1}^{K}\frac{\one_{S\cap\V_{k}}}{B_{k}}\in\prod_{k=1}^{K}\Delta_{n_{k}}$ such that $\left\|\x_1-\sum_{k=1}^{K}\frac{\one_{S\cap\V_{k}}}{B_{k}}\right\|_{2}^{2}\le\|\x_1\|_{2}^{2}+\|\sum_{k=1}^{K}\frac{\one_{S\cap\V_{k}}}{B_{k}}\|_{2}^{2}\le 2r$.

Finally, due to that $\E(f_{t}(S_{t}))\ge\E(F_{t}(\x_t))$ from \cref{thm5+}, we can infer that $(\frac{\gamma^{2}}{\beta+\beta(1-\gamma)+\gamma^{2}})$-Regret of our proposed \cref{alg:OGA_framework}, i.e., $\text{Reg}_{(\frac{\gamma^{2}}{\beta+\beta(1-\gamma)+\gamma^{2}})}\left(T\right)$ satisfies that
\begin{equation*}
\begin{aligned}
   &\text{Reg}_{(\frac{\gamma^{2}}{\beta+\beta(1-\gamma)+\gamma^{2}})}\left(T\right)\triangleq\Big(\frac{\gamma^{2}}{\beta+\beta(1-\gamma)+\gamma^{2}}\Big)\sum_{t=1}^{T}f_{t}(S^{*})-\sum_{t=1}^{T}f_{t}(S_{t})\\
   &=\Big(\frac{1}{\beta+\beta(1-\gamma)+\gamma^{2}}\Big)\sum_{t=1}^{T}\Big( \gamma^{2}f_{t}(S^{*})-(\beta+\beta(1-\gamma)+\gamma^{2})F_{t}(\x_t)\Big)
   \\&\le\Big(\frac{1}{\beta+\beta(1-\gamma)+\gamma^{2}}\Big)\Big(\frac{r}{\eta}+\frac{n\bar{B}^{2}M_{1:T}^{2}T}{2}\eta\Big)\le\frac{4r}{3\eta}+\frac{2n\bar{B}^{2}M_{1:T}^{2}T}{3}\eta=\mathcal{O}\left(\frac{r}{\eta}+nT\eta\right),
\end{aligned}
\end{equation*} where the final inequality comes from Lemma B.1 in \citep{thiery2022two}, namely, $\beta(1-\gamma)+\gamma^{2}\ge\frac{3}{4}$.

\textbf{2)} We consider \cref{alg:OGA_framework} with the auxiliary gradient estimation(See Lines 9-11). In other words, we choose $\text{Auxiliary}\triangleq\textbf{True}$ and set $\boldsymbol{\mathrm{g}}_{t}\triangleq(\int_{a=0}^{1}w(a)\mathrm{d}a)\widehat{\nabla} F_{t}\big(z_{t}*\x_{t}\big)$ where $z_t$ is drawn from the random variable $\mathcal{Z}$ with distribution $\text{Pr}(\mathcal{Z}\le z)\triangleq\frac{\int_{0}^{z}w(a)\mathrm{d}a}{\int_{0}^{1}w(a)\mathrm{d}a},\forall z\in[0,1]$. As a result, from \cref{appendix:online1:1}, we can have that
\begin{equation}\label{equ_SGA_22}
	\begin{aligned}
		&\E\left\|\x_{t+1}-\sum_{k=1}^{K}\frac{\one_{S\cap\V_{k}}}{B_{k}}\right\|_{2}^{2}\\
		&=\E\left\|\x_t-\sum_{k=1}^{K}\frac{\one_{S\cap\V_{k}}}{B_{k}}\right\|_{2}^{2}+2\eta\E\Big(\left\langle\int_{0}^{1} w(z)\nabla F_{t}(z*\x_{t})\mathrm{d}z, \x_t-\sum_{k=1}^{K}\frac{\one_{S\cap\V_{k}}}{B_{k}}\right\rangle\Big)+\eta^{2}(\int_{0}^{1}w(z)\mathrm{d}z)^{2}\E\big(\|\widehat{\nabla} F_{t}\big(z_{t}*\x_{t}\big)\|_{2}^{2}\big)\\
        &\le\E\left\|\x_t-\sum_{k=1}^{K}\frac{\one_{S\cap\V_{k}}}{B_{k}}\right\|_{2}^{2}+2\eta\E\Big(\left\langle\int_{0}^{1} w(z)\nabla F_{t}(z*\x_{t})\mathrm{d}z, \x_t-\sum_{k=1}^{K}\frac{\one_{S\cap\V_{k}}}{B_{k}}\right\rangle\Big)+\eta^{2}(\int_{0}^{1}w(z)\mathrm{d}z)^{2}n\bar{B}^{2}M_{1:T}^{2},
	\end{aligned}		
\end{equation} where the final inequality follows from \cref{lemma_bound_gradient}.

Then, if \textbf{i):} every set function $f_{t}:2^{\V}\rightarrow\R_{+}$ is monotone and $\alpha$-weakly DR-submodular and $w(z)=e^{\alpha(z-1)}$, from \cref{lemma:C1} and Eq.\eqref{equ_SGA_22}, we have that
\begin{equation*}
	\begin{aligned}
		&\E\left\|\x_{t+1}-\sum_{k=1}^{K}\frac{\one_{S\cap\V_{k}}}{B_{k}}\right\|_{2}^{2}\\
		&\le\E\left\|\x_t-\sum_{k=1}^{K}\frac{\one_{S\cap\V_{k}}}{B_{k}}\right\|_{2}^{2}+2\eta\E\Big(\left\langle\int_{0}^{1} w(z)\nabla F_{t}(z*\x_{t})\mathrm{d}z, \x_t-\sum_{k=1}^{K}\frac{\one_{S\cap\V_{k}}}{B_{k}}\right\rangle\Big)+\eta^{2}(\int_{0}^{1}w(z)\mathrm{d}z)^{2}n\bar{B}^{2}M_{1:T}^{2}\\
		&\le\E\left\|\x_t-\sum_{k=1}^{K}\frac{\one_{S\cap\V_{k}}}{B_{k}}\right\|_{2}^{2}-2\eta\Big((1-e^{-\alpha})f_{t}(S^{*})-F_{t}(\x_t)\Big)+\eta^{2}(\frac{1-e^{-\alpha}}{\alpha})^{2}n\bar{B}^{2}M_{1:T}^{2},
	\end{aligned}		
\end{equation*} where $S^{*}\triangleq\mathop{\arg\max}_{S\subseteq\V,|S\cup\V_{k}|\le B_{k},\forall k\in[K]}\sum_{t=1}^{T}f_{t}(S)$ and the final inequality comes from  \cref{lemma:C1} and $\int_{0}^{1}w(z)\mathrm{d}z=\int_{0}^{1}e^{\alpha(z-1)}\mathrm{d}z=\frac{1-e^{-\alpha}}{\alpha}$.

As a result,
\begin{equation*}
\begin{aligned}
    2\eta\sum_{t=1}^{T}\Big((1-e^{-\alpha})f_{t}(S^{*})-F_{t}(\x_t)\Big)&\le\E\Big(\left\|\x_1-\sum_{k=1}^{K}\frac{\one_{S\cap\V_{k}}}{B_{k}}\right\|_{2}^{2}\Big)-E\Big(\left\|\x_{T+1}-\sum_{k=1}^{K}\frac{\one_{S\cap\V_{k}}}{B_{k}}\right\|_{2}^{2}\Big)+T\eta^{2}(\frac{1-e^{-\alpha}}{\alpha})^{2}n\bar{B}^{2}M_{1:T}^{2}\\
    &\le 2r+T\eta^{2}(\frac{1-e^{-\alpha}}{\alpha})^{2}n\bar{B}^{2}M_{1:T}^{2}\le2r+T\eta^{2}n\bar{B}^{2}M_{1:T}^{2},
\end{aligned}
\end{equation*} where the second inequality follows from $\x_{1}\triangleq\sum_{k=1}^{K}\frac{1}{n_{k}} \one_{\V_{k}}\in\prod_{k=1}^{K}\Delta_{n_{k}}$ and $\sum_{k=1}^{K}\frac{\one_{S\cap\V_{k}}}{B_{k}}\in\prod_{k=1}^{K}\Delta_{n_{k}}$ such that $\left\|\x_1-\sum_{k=1}^{K}\frac{\one_{S\cap\V_{k}}}{B_{k}}\right\|_{2}^{2}\le\|\x_1\|_{2}^{2}+\|\sum_{k=1}^{K}\frac{\one_{S\cap\V_{k}}}{B_{k}}\|_{2}^{2}\le 2r$ and the final inequality comes from $\frac{1-e^{-\alpha}}{\alpha}\le1$ when $\alpha\in(0,1]$.

Finally, due to that $\E(f_{t}(S_{t}))\ge\E(F_{t}(\x_t))$ from \cref{thm5+}, we can infer that $(1-e^{-\alpha})$-Regret of our proposed \cref{alg:OGA_framework}, i.e., $\text{Reg}_{(1-e^{-\alpha})}\left(T\right)$ satisfies that
\begin{equation*}
\begin{aligned}
   \text{Reg}_{(1-e^{-\alpha})}\left(T\right)\triangleq(1-e^{-\alpha})\sum_{t=1}^{T}f_{t}(S^{*})-\sum_{t=1}^{T}f_{t}(S_{t})\le\frac{r}{\eta}+\frac{\bar{B}^{2}M_{1:T}^{2}nT}{2}\eta=\mathcal{O}\left(\frac{r}{\eta}+nT\eta\right).
\end{aligned}
\end{equation*} 

Then, if \textbf{ii):} every set function $f_{t}:2^{\V}\rightarrow\R_{+}$ is monotone $(\gamma,\beta)$-weakly submodular and $w(z)=e^{\phi(\gamma,\beta)(z-1)}$ with $\phi(\gamma,\beta)=\beta(1-\gamma)+\gamma^2$, from \cref{lemma:C1} and Eq.\eqref{equ_SGA_22}, we have that
\begin{equation*}
	\begin{aligned}
		&\E\left\|\x_{t+1}-\sum_{k=1}^{K}\frac{\one_{S\cap\V_{k}}}{B_{k}}\right\|_{2}^{2}\\
		&\le\E\left\|\x_t-\sum_{k=1}^{K}\frac{\one_{S\cap\V_{k}}}{B_{k}}\right\|_{2}^{2}+2\eta\E\Big(\left\langle\int_{0}^{1} w(z)\nabla F_{t}(z*\x_{t})\mathrm{d}z, \x_t-\sum_{k=1}^{K}\frac{\one_{S\cap\V_{k}}}{B_{k}}\right\rangle\Big)+\eta^{2}(\int_{0}^{1}w(z)\mathrm{d}z)^{2}n\bar{B}^{2}M_{1:T}^{2}\\
		&\le\E\left\|\x_t-\sum_{k=1}^{K}\frac{\one_{S\cap\V_{k}}}{B_{k}}\right\|_{2}^{2}-2\eta\Big(\big(\frac{\gamma^{2}(1-e^{-\phi(\gamma,\beta)})}{\phi(\gamma,\beta)}\big)f_{t}(S^{*})-F_{t}(\x_t)\Big)+\eta^{2}\Big(\frac{1-e^{-\phi(\gamma,\beta)}}{\phi(\gamma,\beta)}\Big)^{2}n\bar{B}^{2}M_{1:T}^{2},
	\end{aligned}		
\end{equation*} where $S^{*}\triangleq\mathop{\arg\max}_{S\subseteq\V,|S\cup\V_{k}|\le B_{k},\forall k\in[K]}\sum_{t=1}^{T}f_{t}(S)$ and the final inequality comes from  \cref{lemma:C1} and $\int_{0}^{1}w(z)\mathrm{d}z=\int_{0}^{1}e^{\phi(\gamma,\beta)(z-1)}\mathrm{d}z=\frac{1-e^{-\phi(\gamma,\beta)}}{\phi(\gamma,\beta)}$.

As a result,
\begin{equation*}
\begin{aligned}
    &2\eta\sum_{t=1}^{T}\Big(\big(\frac{\gamma^{2}(1-e^{-\phi(\gamma,\beta)})}{\phi(\gamma,\beta)}\big)f_{t}(S^{*})-F_{t}(\x_t)\Big)\\&\le\E\Big(\left\|\x_1-\sum_{k=1}^{K}\frac{\one_{S\cap\V_{k}}}{B_{k}}\right\|_{2}^{2}\Big)-E\Big(\left\|\x_{T+1}-\sum_{k=1}^{K}\frac{\one_{S\cap\V_{k}}}{B_{k}}\right\|_{2}^{2}\Big)+T\eta^{2}(\frac{1-e^{-\phi(\gamma,\beta)}}{\phi(\gamma,\beta)})^{2}n\bar{B}^{2}M_{1:T}^{2}\\
    &\le 2r+T\eta^{2}(\frac{1-e^{-\phi(\gamma,\beta)}}{\phi(\gamma,\beta)})^{2}n\bar{B}^{2}M_{1:T}^{2}\le2r+T\eta^{2}n\bar{B}^{2}M_{1:T}^{2},
\end{aligned}
\end{equation*} where the second inequality follows from $\x_{1}\triangleq\sum_{k=1}^{K}\frac{1}{n_{k}} \one_{\V_{k}}\in\prod_{k=1}^{K}\Delta_{n_{k}}$ and $\sum_{k=1}^{K}\frac{\one_{S\cap\V_{k}}}{B_{k}}\in\prod_{k=1}^{K}\Delta_{n_{k}}$ such that $\left\|\x_1-\sum_{k=1}^{K}\frac{\one_{S\cap\V_{k}}}{B_{k}}\right\|_{2}^{2}\le\|\x_1\|_{2}^{2}+\|\sum_{k=1}^{K}\frac{\one_{S\cap\V_{k}}}{B_{k}}\|_{2}^{2}\le 2r$ and the final inequality comes from Lemma B.1 in \citep{thiery2022two}, namely, $\phi(\gamma,\beta)=\beta(1-\gamma)+\gamma^{2}\ge\frac{3}{4}$ and the truth that $\frac{1-e^{-\phi(\gamma,\beta)}}{\phi(\gamma,\beta)}\le1$ when $\phi(\gamma,\beta)\ge\frac{3}{4}$.

Finally, due to that $\E(f_{t}(S_{t}))\ge\E(F_{t}(\x_t))$ from \cref{thm5+}, we can infer that $\big(\frac{\gamma^{2}(1-e^{-\phi(\gamma,\beta)})}{\phi(\gamma,\beta)}\big)$-Regret of our proposed \cref{alg:OGA_framework}, i.e., $\text{Reg}_{\big(\frac{\gamma^{2}(1-e^{-\phi(\gamma,\beta)})}{\phi(\gamma,\beta)}\big)}\left(T\right)$ satisfies that
\begin{equation*}
\begin{aligned}
   \text{Reg}_{\big(\frac{\gamma^{2}(1-e^{-\phi(\gamma,\beta)})}{\phi(\gamma,\beta)}\big)}\left(T\right)\triangleq\big(\frac{\gamma^{2}(1-e^{-\phi(\gamma,\beta)})}{\phi(\gamma,\beta)}\big)\sum_{t=1}^{T}f_{t}(S^{*})-\sum_{t=1}^{T}f_{t}(S_{t})\le\frac{r}{\eta}+\frac{\bar{B}^{2}M_{1:T}^{2}nT}{2}\eta=\mathcal{O}\left(\frac{r}{\eta}+nT\eta\right).
\end{aligned}
\end{equation*} 
\end{proof}
\subsection{Transferring the Online Results of \texorpdfstring{\cref{thm:online1}}{} into Offline Settings}\label{sec:offline-to-online}
In this subsection, we show how to transfer the online results of \cref{thm:online1} into offline settings throughout the operation in  Line 16 of \cref{alg:OGA_framework}. At first, from the details of \cref{thm:online1}, we know that, if we set the step size $\eta\triangleq\mathcal{O}(\sqrt{\frac{r}{nT}})$, our \cref{alg:OGA_framework} can achieve a $\rho$-regret bound as follows:
\begin{equation*}
\begin{aligned}
\E\left(\text{Reg}_{\rho}\left(T\right)\right)=\rho\max_{S\in\mathcal{C}}\sum_{t=1}^{T}f_{t}(S)-\sum_{t=1}^{T}\E\left(f_{t}(S_{t})\right)\le\mathcal{O}\left(\frac{r}{\eta}+nT\eta\right)=\mathcal{O}\left(\sqrt{nrT}\right),
\end{aligned}
\end{equation*} where the obtained approximation ratio $\rho$ depends on whether the auxiliary gradient estimation is employed.

Therefore, when every incoming set function $f_{t}$ is a fixed set objective $f$, we can show that the subset $S_l\triangleq\mathop{\arg\max}_{S\in\{S_{1},\dots,S_{T}\}}f(S)$ returned by the Line 16 of \cref{alg:OGA_framework} satisfies that
\begin{equation*}
\begin{aligned}
&\E(f(S_{l}))\ge\frac{\sum_{t=1}^{T}\E\big(f(S_{t})\big)}{T}=\frac{\sum_{t=1}^{T}\E\big(f_{t}(S_{t})\big)}{T}
\\&\ge\frac{\rho\max_{S\in\mathcal{C}}\sum_{t=1}^{T}f_{t}(S)-\mathcal{O}\left(\sqrt{nrT}\right)}{T}=\rho\max_{S\in\mathcal{C}}f(S)-\mathcal{O}(\sqrt{\frac{nr}{T}}).
\end{aligned}
\end{equation*}

Especially when $T=\mathcal{O}(\frac{nr}{\epsilon^{2}})$ and $\eta\triangleq\mathcal{O}(\sqrt{\frac{r}{nT}})=\mathcal{O}(\frac{\epsilon}{n})$, we can have $\E(f(S_{l}))\ge\rho\max_{S\in\mathcal{C}}f(S)-\epsilon$, which verifies the discussions in \cref{remark:online-to-offline}.
